\documentclass[final,12pt]{arxiv} % Include author names

\usepackage{amsfonts, stmaryrd,mathtools}
\usepackage{enumitem}
\usepackage{hyperref}
\usepackage{bbm}

\usepackage{tikz}
\usetikzlibrary{decorations.pathreplacing}
\usepackage[ruled]{algorithm2e}

\newtheorem{Theorem}{Theorem}

\newtheorem{Lemma}{Lemma}

\renewcommand{\P}{\mathbb{P}}
\newcommand{\E}{\mathbb{E}}
\newcommand{\R}{\mathbb{R}}
\newcommand{\N}{\mathbb{N}}

\newcommand{\1}{\mathbbm{1}}
\newcommand{\defn}{\coloneqq}

\newcommand{\argmin}{\arg\!\min}
\newcommand{\argmax}{\arg\!\max}

\newcommand{\lrset}[1]{\left\{{#1}\right\}}
\newcommand{\lrp}[1]{\left({#1}\right)}
\newcommand{\lrs}[1]{\left[{#1}\right]}
\newcommand{\abs}[1]{\left\lvert{#1}\right\rvert}%
\newcommand{\Exp}[2]{\mathbb{E}_{#1}\lrs{#2}}

\newcommand{\Supp}[1]{\operatorname{Sp}\lrp{#1}}

\newcommand{\shubhada}[1]{{\color{orange}[SA: #1]}}
\newcommand{\deb}[1]{{\color{red} [DB: #1]}}
\newcommand{\oam}[1]{{\color{purple} [OA: #1]}}

\numberwithin{equation}{section}
\newcommand{\KL}{\mathrm{KL}}

\renewcommand{\d}{\mathrm{d}}

\newcommand{\corby}{\odot_{\varepsilon}}

\newcommand{\KLinf}{\operatorname{KL_{inf}}}
\newcommand{\eKinf}{\operatorname{KL^{\varepsilon}_{inf}}}
\newcommand{\ekl}{\mathrm{kl}^{\varepsilon}}
\newcommand{\Med}{\mathrm{Med}}

\newcommand{\G}{\mathcal{N}}
\newcommand{\cG}{\mathcal{G}}
\newcommand{\Ber}{\mathrm{Ber}}

\usepackage[nohints]{minitoc}
\renewcommand \thepart{}
\renewcommand \partname{}

\newcommand{\rimedstar}{\texttt{CRIMED}*}
\newcommand{\rimed}{\texttt{CRIMED}}

\newcommand{\eq}{\text{Equation}}

\usepackage{mathrsfs}

\title[Corruption Robust IMED]{CRIMED: Lower and Upper Bounds on Regret\\ for Bandits with Unbounded Stochastic Corruption}
\usepackage{times}
\altauthor{%
 \Name{Shubhada Agrawal} \Email{sagrawal362@gatech.edu}\\
 \addr H. Milton Stewart School of Industrial and Systems Engineering, Georgia Institute of Technology, USA.
 \AND
 \Name{Timothée Mathieu} \Email{timothee.mathieu@inria.fr}\\
  \addr Univ. Lille, Inria, CNRS,  Centrale Lille, UMR 9189 – CRIStAL, F-59000 Lille, France.
 \AND
 \Name{Debabrota Basu} \Email{debabrota.basu@inria.fr}\\
  \addr Univ. Lille, Inria, CNRS, Centrale Lille, UMR 9189 – CRIStAL, F-59000 Lille, France.
 \AND
 \Name{Odalric-Ambrym Maillard} \Email{odalric.maillard@inria.fr}\\
 \addr Univ. Lille, Inria, CNRS, Centrale Lille, UMR 9189 – CRIStAL, F-59000 Lille, France.%
}
\begin{document}

\maketitle
% For TOC in appendix (https://tex.stackexchange.com/a/419290)
\doparttoc % Tell to minitoc to generate a toc for the parts
\faketableofcontents % Run a fake tableofcontents command for the partocs
\begin{abstract}%

We investigate the regret-minimisation problem in a multi-armed bandit setting with arbitrary corruptions. Similar to the classical setup, the agent receives rewards generated independently from the distribution of the arm chosen at each time. However, these rewards are not directly observed. Instead, with a fixed $\varepsilon\in (0,\frac{1}{2})$, the agent observes a sample from the chosen arm's distribution with probability $1-\varepsilon$, or from an arbitrary corruption distribution with probability $\varepsilon$. Importantly, we impose no assumptions on these corruption distributions, which can be unbounded. In this setting, accommodating potentially unbounded corruptions, we establish a problem-dependent lower bound on regret for a given family of arm distributions. We introduce \rimed, an asymptotically-optimal algorithm that achieves the exact lower bound on regret for bandits with Gaussian distributions with known variance.  Additionally, we provide a finite-sample analysis of \rimed's regret performance. Notably, \rimed{} can effectively handle corruptions with $\varepsilon$ values as high as $\frac{1}{2}$. Furthermore, we develop a tight concentration result for medians in the presence of arbitrary corruptions, even with $\varepsilon$ values up to $\frac{1}{2}$, which may be of independent interest. We also discuss an extension of the algorithm for handling misspecification in Gaussian model.

\end{abstract}


\section{Introduction}
\vspace{-0.3em}
Multi-armed bandits are a widely-used statistical model in which an agent (or algorithm) interacts with the environment by selecting actions based on past observations and receives a \emph{reward} for each chosen action. A classical objective is to minimise the \emph{regret}, which is defined as the difference between the rewards accumulated by the algorithm and the rewards that would have been obtained by choosing the best action in hindsight at each step. In this paper, we delve into the problem of sequential decision-making under partial information, where the observations resulting from actions are susceptible to arbitrary yet stochastic corruption. More specifically, we explore a variant of the stochastic multi-armed bandit problem with \emph{unbounded stochastic corruption}.

The algorithm is presented with a set of $K$ arms, each  representing an unknown probability distribution from a given family of distributions $\mathcal L$. We denote this set of $K$-arms by $\mu \defn (\mu_1, \dots, \mu_K)$, where $\forall a\in [K], \mu_a \in \mathcal L$. When the algorithm selects an arm $A_n$ at time $n$, an independent sample, $Y_{n}$, is drawn from the corresponding distribution $\mu_{A_n}$. This is the reward of the algorithm for pulling the chosen arm. However, unlike in the classical setup, $Y_n$ is not directly observed by the algorithm. Instead, the observations are subject to corruption with a known probability $\varepsilon\in(0,0.5)$: at time $n$, the algorithm observes $Y_n \sim \mu_{A_n}$ with probability $1-\varepsilon$, and with probability $\varepsilon$, it observes a sample from an arbitrary distribution $H_{A_n,n}$. Let ${\bf H}_n \defn (H_{1,n}, \dots, H_{K,n})$ be the set of corruption distributions at time $n$. We place absolutely no assumptions on $H_{\cdot,n}$. Following the classical regret-minimisation setting, the goal of the algorithm in this partial-information setting is to sequentially sample these arms in order to maximize the expected cumulative reward when the observations are corrupted.
%\oam{Is $H$ the same for all arms?}

The bandit problem forms the theoretical cornerstone of modern Reinforcement Learning (RL) and serves as the algorithmic foundation for recommender systems. As both bandit problems and RL are increasingly finding practical applications, the question of robustness against corrupted or externally perturbed observations has gained considerable significance. This is particularly relevant because, in real-life applications such as finance, medicine,  advertising, or recommender systems, algorithms often need to contend with corrupted data, as observations collected from multiple sources are susceptible to measurement or recording errors and inaccuracies. 

In finance, the observed payoff data frequently contains outliers due to data contamination \citep{adams2019identifying}. In clinical research trials for new drugs and medical devices, outlier data can lead to false positive interventions and conclusions \citep{thabane2013tutorial}. %\shubhada{Does it make sense to cite \href{https://www.ncbi.nlm.nih.gov/pmc/articles/PMC4340084/}{this} article?}\deb{check the email}. 
In online platforms and recommender systems, the ranking of products can get skewed by the appearance of  small number of fake users~\citep{golrezaei2021learning}.
The classical corruption-oblivious bandit algorithms, however, cannot effectively decide the arms to pull, when a possibly small fraction of the data may be subject to measurement errors or corruption. %\shubhada{Add more discussion on motivation here? Why regret problem specifically?}\deb{check if the line after this suffices.} 
Specifically, a recent line of works~\citep{jun2018adversarial,liu2019data,xu2021observation,azize2023interactive} show that one can make a classical bandit algorithms to incur linear regret by contaminating only a small amount of observations (logarithmic on horizon or even lower). These findings motivate us to study the bandit setup in the presence of \emph{arbitrary} corruption, and design algorithms robust to it. 

Researchers have broadly studied three types of settings: \textit{adversarial bandits} ~\citep{auer1995gambling, auer2002nonstochastic}, \textit{stochastic bandits with bounded adversarial corruptions}, in which an adversary shifts the rewards under constraint on the total shift budget ~\citep{lykouris2018stochastic, gupta2019better, zimmert2019optimal}, and more recently, \textit{unbounded stochastic corruption}~\citep{JMLR:v20:18-395,mukherjee2021mean,basu2022bandits}. To the best of our knowledge, there is \textit{no established generic lower bound on regret in the context of unbounded stochastic corruptions}, unlike in the first two settings. Furthermore, there is \textit{no known algorithm capable of yielding an appropriate upper bound} on regret while also maintaining robustness. This paper aims to fill these two gaps by investigating bandits with unbounded stochastic corruptions.

\vspace{0.25em}
\noindent{\bf Regret.}~ For $\mu\in \mathcal L^K$, let $m^*(\mu)$ denote the mean of the optimal arm in $\mu$ (arm with the maximum mean), and let $m(\mu_a)$ denote the mean of arm $a$. For an arm $a$, let $\Delta_a \defn m^*(\mu) - m(\mu_a)$ denote the instantaneous mean regret incurred by pulling it. Recall that $Y_n$ denotes the independent (uncorrupted) sample drawn from the distribution associated with arm $A_n$. Let $Y_{a,j}$ denote the $j^{th}$ independent sample drawn from arm $a$. 

Since in our setup, the observations are corrupted (while the rewards are uncorrupted), we define the \textit{expected regret under corruption} till time $T$ as
$\mathbb{E}\lrs{R_T} \defn \E[{\sum\nolimits_{n=1}^T \lrp{m^*(\mu) - Y_{n}} }]$. Here, the expectation is with respect to all the randomness present in the system, including the impact of corruption on action selection. See also~\citep{kapoor2019corruption,basu2022bandits} for a similar notion of regret. We further observe  that $\E\lrs{R_T} = \sum\nolimits_{a=1}^K \E\lrs{N_a(T)} \Delta_a,$ where $N_a(T)$ is the number of pulls of the arm $a$ till time $T$. Since $\Delta_a$'s are constant for a given $\mu$ and for the optimal arm(s) $\Delta_a = 0$, minimising the expected regret reduces to minimise the expected number of pulls of the suboptimal arms $\E\lrs{N_a(T)}$. 

\vspace{0.25em}
\noindent{\bf Notation.}~ Let $\R$ and $\R^+$ denote the set of real numbers and non-negative real numbers, respectively, and let $\mathcal{P}(\R)$ denote the set of all probability distributions on $\R$. For any set $S$, we denote by $2^{S}$ the set of subsets of $S$. For $\mu$, $n\in\N$, and ${\bf H}_n \in \mathcal P(\R)^K$, let $\mu \corby {\bf H}_n \defn (1-\varepsilon) \mu + \varepsilon {\bf H}_n$ denote the vector of distributions in $\mu$ corrupted by corruption distributions in ${\bf H}_n$ with corruption proportion $\varepsilon$. We use a similar notation for each component $\mu_a, H_{a,n} \in \mathcal P(\R)$, i.e. $\mu_a \corby H_{a,n} \defn (1-\varepsilon) \mu_a + \varepsilon H_{a,n}$ for $a \in [K]$. Additionally, by $\bf H_T$, we denote the $T\times K$ matrix of corruption distributions with $\lrset{{\bf H}_n : n\in[T]}$ as rows. Finally, we denote  by $\mathcal{G}$ the set of all Gaussian distributions with variance $1$, by $\varphi$ the Gaussian pdf, and by $\Phi$ the Gaussian CDF.
%$\Delta_{\min}\defn 2\Phi^{-1}\left( \frac{1}{2(1-\varepsilon)}\right),$
%where $\Phi$ denotes the Gaussian CDF.

\vspace{-0.3em}
\subsection{Contributions}
\vspace{-0.4em}
In this paper, we investigate two questions:\\
\noindent$\bullet$ \textit{Can we derive a problem-dependent lower bound on regret for a given set of reward distributions and the corresponding worst-case corruption distributions?} \\
\noindent$\bullet$ \textit{Can we leverage this lower bound to design an asymptotically-optimal algorithm that is robust to unbounded stochastic corruptions? }

In this section, we briefly describe the main contributions of this work.

\begin{enumerate}[leftmargin=*, topsep=0.1pt, partopsep =0pt]%, noitemsep]
\setlength\itemsep{0.25em}
%\vspace{0.25em}
\item \textit{A Generic Lower Bound on Regret:} To the best of our knowledge, we establish the first instance-dependent lower bound on regret that is applicable to any given family of reward distributions and arbitrary corruption distributions (Section~\ref{sec:lower_bound}). Specifically, in Theorem~\ref{lem:lower_bound}, we demonstrate that any algorithm performing well across all bandit instances within a given class must, in expectation, pull each suboptimal arm at least $\Omega(\log T)$ times over the course of $T$ trials. This result aligns with the known $\Omega(\log T)$ problem-dependent lower bound in the classical setting~\citep{lai1985asymptotically}. Moreover, when $\epsilon = 0$, our proposed lower bound reduces to that of the uncorrupted setting~\citep{lai1985asymptotically, burnetas1996optimal}.

\item \textit{An Impossibility Result:} We demonstrate in Appendix~\ref{app:knowledgeofeps} that constructing confidence intervals for the mean of the true distribution in the presence of corruption, a classical problem in statistics, is not feasible without prior knowledge of a bound on corruption probability $\varepsilon$. This resolves the open problem discussed in \citet[Remark 3]{wang2023huber} and also justifies the assumption about the knowledge of $\varepsilon$ in the current work (Remark~\ref{rem:know_eps}).

\item \textit{An Analytical Quantifier of Hardness:} The lower bound in Theorem~\ref{lem:lower_bound} is in terms of an optimisation problem that takes the given bandit instance $\mu$ as an input. In order to explicitly bring out the structure of this problem and the hard corruption distributions for the given bandit instance, we undertake an in-depth study for the specific setting of Gaussian reward distributions with known variance, while still allowing for unbounded and arbitrary corruptions (see also ~\cite{chen2018robust}). For this setting, we characterise the hardest corruption distributions associated with pairs of arms in $\mu$ that lead to the maximum regret in any algorithm. In addition, we show that for each suboptimal arm $a$, $\Delta_a$ should be at least $2\Phi^{-1}\left(\frac{1}{2}(1-\varepsilon)\right)$ for any algorithm to achieve a sub-logarithmic regret in presence of corruption, and also observe a non-convexity in the lower bound (Section~\ref{sec:GaussianLB} and in Appendix~\ref{app:prop_ekl}). These observations stand in stark contrast to the classical bandit setup, necessitating a careful treatment in our analysis. 

\item \textit{Algorithm Design:} In Section~\ref{sec:algo}, we leverage the formulation and properties of the lower bound to propose an index-based algorithm, namely \rimed{} (Corruption Robust IMED, Algorithm~\ref{alg:robustimed}), for unbounded corruptions and Gaussian reward distributions with known variance (we discuss extension to \emph{misspecified} Gaussian distributions in Appendix~\ref{app:misspecified}). This is an extension of the IMED Algorithm proposed by ~\cite{JMLR:v16:honda15a}, with two main changes in the index design.  First, it replaces the classical information-theoretic quantities that appear in the IMED index with their pessimistic versions in order to account for the presence of corruptions. Second, it uses median as a robust estimate for mean in the presence of corruption. In Section~\ref{sub:regret_ub}, we give a finite-sample analysis of the regret of \rimed{} (Theorem~\ref{th:upper_bound}). Notably, \rimed{} is asymptotically (as $T\rightarrow \infty$) optimal for any corruption level $\varepsilon < \frac{1}{2}$, which is a significant improvement over the previous works of~\cite{kapoor2019corruption} and~\cite{basu2022bandits} allowing only much smaller $\epsilon$.

\item \textit{Median as the Robust Estimator and its Impact:} Bandits involving arbitrary corruptions present significantly greater challenges compared to their classical counterparts. In the presence of arbitrary corruptions, it is well-known that no consistent estimators for the mean of distributions can exist \citep{chen2018robust}. To address this challenge, we draw from the robust estimation literature and opt for the median as a robust estimate of the mean. This choice is motivated by the fact that for \emph{symmetric} distributions, median is optimal because it incurs the smallest bias among all robust estimators of the location parameter (see Section \ref{sec:non_inter}). 
In Theorem~\ref{th:concentration_median}, we establish a novel concentration bound for the empirical median of corrupted Gaussian rewards that applies to any value of $\varepsilon$ less than $\frac{1}{2}$. %This concentration result holds considerable interest for the robust statistics community and can be extended to symmetric unimodal distributions.
%\oam{Why "can be". either we extend it now, or we don't.}
%\tm{Say that we need symmetry.}\deb{seems done.}
\end{enumerate}
\vspace{-\topsep}

%For brevity, we postpone technical derivations to the appendix.

\vspace{-0.2em}
\subsection{Related work} \label{sec:related_lit}
\vspace{-0.4em}

Our work connects and relates to several research areas, which we now briefly summarize.

\vspace{0.25em}
\noindent\textbf{Multi-armed bandits.} The problem of bandits was first introduced in the context of designing adaptive clinical trials by  \citet{thompson1933likelihood}, and later popularised under this name by \citet{robbins1952some}. Since then, the variants of this problem have been widely studied and are used in practice. For the classical regret-minimisation framework introduced earlier, asymptotic instance-dependent lower bounds on the regret suffered by an algorithm are well known~\citep{lai1985asymptotically, burnetas1996optimal}. 

Index-based (UCB) algorithms for this setting were popularised by the work of \cite{auer2002finite}. \citet{cappe2013kullback,agrawal2021regret} proposed asymptotically-optimal UCB algorithms for parametric and heavy-tailed distributions, respectively. While these algorithms are statistically optimal, they can be computationally demanding. \citet{honda2009asymptotically,honda2010asymptotically,JMLR:v16:honda15a} developed a different style of (IMED) algorithms that have a lower computational cost and are also statistically optimal. Alternative optimal algorithms relying on Bayesian posteriors to sample arms (Thompson sampling) have also been developed~\citep{agrawal2012analysis, agrawal2017near, kaufmann2012thompson}. In this paper, we follow a frequentist approach and design an IMED-type algorithm due to its optimality and computational simplicity.

\vspace{0.25em}

\noindent{\bf Bandits with bounded corruption.}
In the adversarial bandits setting, the rewards are assumed to be generated by an adaptive adversary from a bounded interval, e.g. $[0,1]$. See, for example, ~\cite{auer1995gambling, auer2002nonstochastic, abernethy2009beating, audibert2009minimax, neu2015explore}. Researchers have aimed to design the best of the both worlds algorithms that perform almost optimally for this setting as well as the stochastic setting discussed in the previous paragraph, and are of parallel interest~\citep{bubeck2012best,seldin2014one,seldin2017improved,abbasi2018best,pogodin2020first}. 

In the stochastic setting with bounded adversarial corruptions, whenever an arm is pulled at time $n$, a reward $r_n$ is stochastically generated from the corresponding distribution. But an adversary switches the reward to $r'_n$ such that over the horizon $T$, $\sum_{n=1}^T |r'_n - r_n| \le C$, for a non-negative constant $C$. This setting and its variants have also been extensively studied in literature ~\citep{lykouris2018stochastic,gajane2018corrupt,gupta2019better,zimmert2019optimal,kapoor2019corruption}. Here, the bound $C$ plays a critical role, and the existing regret bounds are linearly dependent on it. These existing regret bounds and algorithms are unfit to handle large amounts of corruptions. This propels the study of bandits that are robust to unbounded corruptions.

\vspace{0.25em}
\noindent{\bf Robust estimation under unbounded stochastic corruption.} A robust estimator is an estimator that perform well even in the presence of anomalous data. The corruption model considered in this work has a long history in robust statistics. Given a data generating distribution $P$ and a corruption budget $\varepsilon$, a corruption neighbourhood of $P$ is the collection of all distributions of the form $(1-\varepsilon)P+\varepsilon H$, for $H\in\mathcal P(\R)$. \cite{huber1964} developed an asymptotic theory minimax optimality of estimators for distributions in corruption neighbourhood of $P$. Since then, several methods have been devised to assess the asymptotic robustness of estimators (see, ~\cite{robuststat,HampelEtal86}), in particular, in terms of the stability of the limit of an estimator when the samples come from a corrupted distribution. Lately, a non-asymptotic notion of robustness has gained interest. Here, the goal is to obtain estimators that concentrate fast, either when the data-generating distribution $P$ is heavy-tailed~\citep{catoni2012,subgaussian,DBLP:journals/focm/LugosiM19,thesis}, or corrupted ~\citep{wang2023huber,chen2018robust}.

These two concepts (asymptotic and non-asymptotic robustness) are closely linked, and the estimators that perform well in the asymptotic sense have also been shown to perform well in the non-asymptotic setting. Huber's contamination model has also been widely-studied in computer science~\citep{diakonikolas2018robustly,charikar2017learning}. In this work, we use concentration of median to control the regret of \rimed{}, that receives samples from a corruption neighbourhood of the arm distributions (or from a misspecified model). The median has also been used for the best-arm identification algorithms in which the goal is to find the arm with the largest median~\citep{JMLR:v20:18-395,even2006action,nikolakakis2021quantile}, which is significantly different from the regret-minimisation setting considered in this paper.  
\vspace{0.25em}

%\vspace{0.25em}
%\noindent{\bf Median.} It is customary to estimate the mean regret when dealing with Multi-armed bandits, and in most cases (when the distributions are not symmetric) the median is not a good estimator for the mean. In the case of symmetric distributions, however, estimating the median is the same as estimating the mean, this corresponds to our setting. The median has also been used for the best-arm identification algorithms in which the goal is to find the arm with the largest median~\citep{JMLR:v20:18-395,even2006action,nikolakakis2021quantile}, which is significantly different from the regret-minimisation setting considered in this paper. \shubhada{Why do we need this paragraph?}\deb{It justified the use of median than other and also it has been used in bandits before} \shubhada{But haven't we justified it in the point 5 in contributions? Moreover, the goal in the cited papers is to identify the arm with maximum quantile/median. It's not the case that they want to identify the arm with maximum mean, but are using median as a proxy (which is sort-of the case for us).}
%\vspace{0.25em}

\noindent{\bf Bandits with unbounded stochastic corruption.} 
To the best of our knowledge, unbounded stochastic corruption in bandits have only been studied in \cite{JMLR:v20:18-395}, \cite{mukherjee2021mean}, and~\cite{basu2022bandits}. \cite{JMLR:v20:18-395} and \cite{mukherjee2021mean} study the best-arm identification problem with a goal to find the arm with the largest median and mean, respectively, in presence of corruptions. While adhering to the same corruption model,  \cite{basu2022bandits} consider the regret minimisation problem, and devise a UCB-type algorithm that incurs $O(\log(T))$ instance-dependent regret that is within a constant of the lower bound. Significantly improving on their work, we devise an algorithm whose regret exactly matches the lower bound asymptotically, as $T\rightarrow\infty$. We also demonstrate the superiority of the proposed algorithm experimentally.

\vspace{-0.7em}
\section{Lower bound and KL-divergence in corrupted neighbourhoods}\label{sec:lower_bound_upsec}
\vspace{-0.3em}
Given a class $\cal L$ of probability distributions, we want algorithms that perform uniformly well on all the $K$-armed bandit instances with arms from $\mathcal L$, when the observations are corrupted with probability $\varepsilon\in [0,1/2)$. To meet this requirement, the algorithm needs to generate sufficient samples from each arm. In this section, we present a lower bound on the number of samples that the algorithm needs to generate from each arm.

%\oam{This paragraph is not really an introduction to the section (in this section, we do bla), rather looks like a comment.}
\vspace{-0.3em}
\subsection{Problem-dependent lower bound}\label{sec:lower_bound}

\vspace{-0.1em}

\begin{definition}[Uniformly-good algorithm]\label{def:unif_good} An algorithm acting on a distribution in $\cal L$ is said to be uniformly-good for a corruption level $\varepsilon$, if for all $\mu \in \mathcal L^K$ and for all suboptimal arms $a$, it satisfies 
\vspace{-0.1em}
\[\sup\limits_{{\bf H_T}\in\mathcal{P}(\R)^{T\times K} } \Exp{\mu \corby {\bf H_T} }{N_a(T)}= o\lrp{T^{\alpha}}, \quad \text{ for all } \alpha > 0.\]
\end{definition}
%\vspace{-0.75em}

Here, $\Exp{\mu\corby \bf{H_T}}{\cdot}$ denotes the expectation with respect to both the corrupted bandit process $\mu\corby {\bf H}_n$, for each $n\in [T]$, and the possible randomness of the algorithm (omitted from notation). Definition~\ref{def:unif_good} is similar to the notion of consistent algorithms considered in the classical setup \citep[Definition 16.1]{lattimore2018bandit}. Observe that unlike in that setting, for every instance, the algorithm should perform well with respect to \emph{every sequence} of $K$ corruption distributions. 

The lower bound on the expected number of times a uniformly-good algorithm pulls a suboptimal arm involves an optimisation problem, which we present first. 

\vspace{0.25em}
\noindent\textbf{Corrupted KL-inf.} The corrupted KL-inf is a function $\eKinf: \mathcal P(\R) \times \R  \times 2^{\mathcal P(\R)} \rightarrow \R_+$, that for $\eta\in\mathcal P(\R)$, $x\in\R$, and $\cal L \subset \mathcal P(\R)$, equals
\vspace{-0.25em}
\begin{equation} \label{eq:def_klinf}
\eKinf(\eta, x; \mathcal{L}) \!\defn\! \min\limits_{H,H',\kappa}\!\lrset{\KL\!\left( \eta \corby \!H, \kappa \corby \!H' \right)\!: ~\kappa\!\in\!\mathcal{L},  H, H' \!\in\! \mathcal{P}(\R),  m(\kappa)\!\ge\! x}. \vspace{-0.5em}\end{equation}

For $\varepsilon = 0$, this is equivalent to the optimisation problem that appears in the lower bound of the uncorrupted setting, leading to the traditional $\KL_{\inf} \defn \min\nolimits_{\kappa}\lrset{\KL\left( \eta, \kappa \right): ~ \kappa\!\in\!\mathcal{L}, ~  m(\kappa) \ge x}$ (c.f. \citet{burnetas1996optimal}, \citet[Chapter 16]{lattimore2018bandit}). The additional optimisation over the corruption distributions $H$ and $H'$ makes $\eKinf$ smaller than $\KL_{\inf}$. 
%Moreover, for $\eta\in\mathcal L$ and $x\le m(\eta)$, $\eKinf(\eta,x;\mathcal L)$ is $0$. This follows from the non-negativity of $\eKinf$ and for any $H$, $\kappa = \eta$ and $H'=H$ are feasible for $\eKinf$. 
Moreover, we observe (Figure~\ref{fig:klinf}) that for $\varepsilon > 0$, $\eKinf$ can be non-convex in the second argument, unlike for $\varepsilon = 0$ \citep[Lemma 10]{agrawal2021regret}. As we will see later, these imply that the problem in presence of corruption is inherently harder than the classical setting.

In the reminder of this paper, $\varepsilon$ denotes a known and fixed constant in  $(0,0.5)$. See Appendix~\ref{app:knowledgeofeps} for a negative result, and a justification for the need to know $\varepsilon$ (Remark~\ref{rem:know_eps}).

\begin{Theorem}[Lower bound]\label{lem:lower_bound} For $\varepsilon > 0$, $\mathcal L \subset \mathcal P(\R)$, and a bandit instance $\mu\in\mathcal L^K$, for any suboptimal arm $a$ in $\mu$, a uniformly-good algorithm satisfies  
\[ \liminf\limits_{T\rightarrow \infty} \frac{1}{\log T}\lrp{\sup\limits_{{\bf H} \in\mathcal{P}(\R)^{K}}\Exp{\mu \corby {\bf H}}{N_a(T)}} \ge~\frac{1}{\eKinf(\mu_a,m^*(\mu); \mathcal{L})}. \]
\end{Theorem}

A few remarks are in order. First, since $\eKinf \leq \KL_{\inf}$ for $\varepsilon \geq 0$, the lower bound above is higher than that in the classical setting. 
%This quantifies that the setting of stochastically-corrupted bandits is inherently harder than the classical setting. 
Second, we show in discussion around Remark~\ref{rem:know_eps} that for $\varepsilon = \frac{1}{2}$, $\eKinf = 0$, implying that logarithmic regret cannot be achieved if a bound on error probability is unknown. Next, we show in Lemma~\ref{lem:non-intersection} that a separation between $m(\mu_a)$ and $m^*(\mu)$ is required, without which $\eKinf = 0$. Finally, since the setting with corruption $\bf H$ fixed across time is simpler than one allowing for different ${\bf H}_n$ at each time $n$, the lower bound in Theorem \ref{lem:lower_bound} holds for the general setting considered in this paper (Remark \ref{rem:general_lb}). 

The central idea of our proof is to extend the classical change of measure lemma~\citep{garivier2019explore} over the $\varepsilon$ corruption neighbourhood of reward distributions, which is of independent interest.
We refer the reader to Section \ref{app:lb} for a complete proof of Theorem \ref{lem:lower_bound}. 

\vspace{-0.3em}
\subsection{Huber's pair and corrupted KL-inf} \label{sec:huberpair}
\vspace{-0.4em}
In Section~\ref{sec:algo}, the proposed algorithm computes $\eKinf$ using samples. To facilitate computation, in this section, we characterise the optimisers for $\eKinf(\eta,x;\cal L)$, specifically, the optimal pair of corruption distributions $H_1$ and $H_2$. Let $\Supp{\eta}$ denote the support of $\eta$. First, we fix $\kappa \in \cal L$, and consider the optimisation problem over the two corruption distributions in $\eKinf$. Define
 \vspace{-0.5em}%Let us express $H_1\in\mathcal P(\R)$ and $H_2 \in \mathcal P(\R)$ as:
\begin{equation}\label{eq:contamination1}
    \d(\eta \corby H_1)(x) \defn
    \begin{cases}
    (1-\varepsilon)\d \eta(x), \quad & \text{ for } { \frac{\d \eta}{\d \kappa}(x) \ge c_1}\\
    c_1(1-\varepsilon)\d \kappa(x), \quad & \text{otherwise},
    \end{cases}
\end{equation}
and\vspace{-1em}
\begin{equation}\label{eq:contamination2}
 \d(\kappa \corby H_2)(x) \defn
    \begin{cases}
    (1-\varepsilon) \d\kappa(x), \quad & \text{ for } \frac{\d \eta}{\d \kappa}(x) \le \frac{1}{c_2}\\
    c_2(1-\varepsilon)\d \eta(x), \quad & \text{otherwise}.
    \end{cases}
\end{equation}

Here, $\d f$ denotes the differential of a distribution function $f$, and $\frac{\d \eta}{\d \kappa}(x)$ denotes the Radon-Nikodym derivative of $\eta$ with respect to $\kappa$. For $x \in \Supp{\eta}\cap \Supp{\kappa}^c$, $ \frac{\d \eta}{\d \kappa}(x) \defn \infty$, and for $x\in\Supp{\eta}^c\cap \Supp{\kappa}$, $\frac{\d \eta}{\d \kappa}(x) \defn 0$. $c_1$ and $c_2$ are the normalisation constants ensuring that $ \d (\eta \corby H_1)$ and $ \d (\kappa \corby H_2)$ are probability measures, and also satisfying $0 \le c_2 \le \frac{1}{c_1} \le \infty$. Observe that Equations~\eqref{eq:contamination1} and~\eqref{eq:contamination2} implicitly define corruption distributions $H_1$ and $H_2$ (Remark~\ref{rem:implicit_corrpair}).

\setlength{\textfloatsep}{4pt}
\begin{figure}[t!]
\centering
    \subfigure[Plot of corrupted distributions for  $\varepsilon=0.2$.][c]{
    \includegraphics[width=0.45\textwidth]{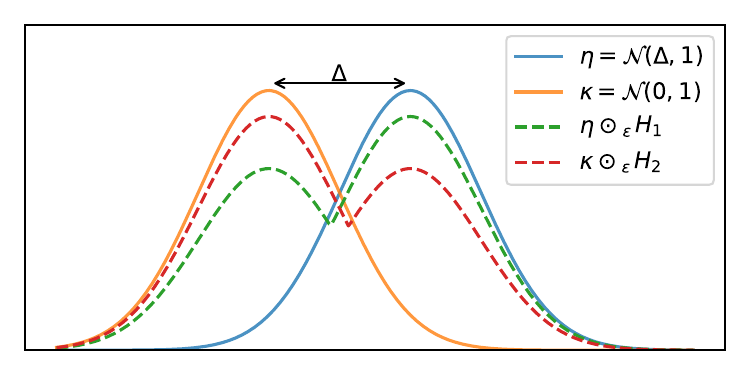}\label{fig:pair_corrupted}}\hspace*{2em}
     \subfigure[Plot of $\ekl_{\cG}$ for $\sigma=1$ and $\varepsilon=0.1$.][c]{
    \includegraphics[width=0.36\textwidth,trim=11 15 11 11,clip]{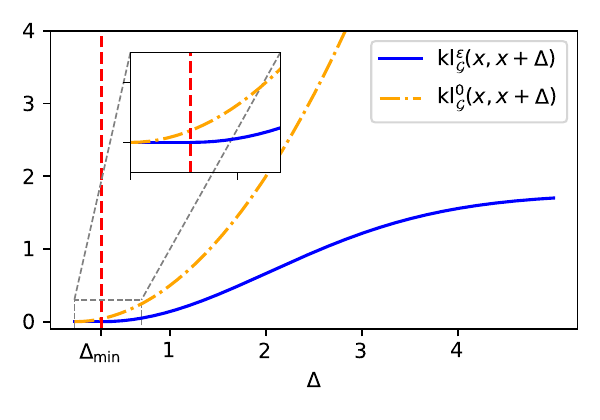}\vspace*{-1.2em}\label{fig:klinf}}
    \vspace{-0.5em}
    \caption{Illustration of the corrupted distributions from Lemma~\ref{lem:pair_distrib} and $\ekl_{\cG}$. $\mathrm{Supp}(H)$ denotes the support of distribution $H$. $\Delta_{\min}$ is defined in Definition~\ref{def:deltamin}. }
\end{figure}

\begin{Lemma}[Optimal corruption pair]\label{lem:pair_distrib}
For $\eta \in \mathcal P(\R)$, $\kappa \in \mathcal{P}(\R)$, and $\varepsilon\in (0,\frac{1}{2})$, $H_1$ and $H_2$ defined by Equations ~\eqref{eq:contamination1} and \eqref{eq:contamination2}, respectively,  are the optimal corruption pair in Equation~\eqref{eq:def_klinf}, i.e., $(H_1,H_2) \in  \argmin\lrset{\KL( \eta \corby H, \kappa \corby H' ): H \in \mathcal{P}(\R), ~ H'\in\mathcal{P}(\R)  }. $
\end{Lemma}

To prove the above result, we show that the directional derivative (for an appropriate notion in the space of probability measures) of $\KL$ in every direction is  non-negative at $(H_1, H_2)$. We refer the reader to Section~\ref{sec:proof_pair_distrib} for a proof of the above result. A similar pair of corruption distributions were considered by~\cite{10.1214/aoms/1177699803} in a hypothesis testing setup.

It follows from Lemma~\ref{lem:pair_distrib} that the optimal corruption pair depends on the two input distributions. In particular, the corruption always stays within the support of the input pair of distributions. For illustration, we present these for the  Gaussian-inlier setting, i.e., when both $\eta$ and $\kappa$ are Gaussian distributions with unit variance, in Figure~\ref{fig:pair_corrupted}. We observe that the sets on which there is corruption are located in the right-tail (respectively left-tail) of the distribution on the left (respectively on the right). These and other interesting properties  for Gaussian model with corruption are formally proven in Lemma~\ref{lem:basics_ekl} and Appendix~\ref{app:additional_prop}, and will be used later in our analysis.  

\vspace{-0.4em}
\subsection{The case of Gaussian rewards with known variance}\label{sec:GaussianLB}
\vspace{-0.3em}
With the above minimisers for the corruption pair for fixed $\eta$ and $\kappa$, we are now left with characterising the optimal $\kappa$ in $\eKinf(\eta,x;\mathcal L)$. When $\cal L = \cal G$, the collection of all Gaussian distributions with a unit variance, and $\eta \in \cal G$, it follows from Lemma \ref{lem:basics_ekl} (later in the section) that the optimiser $\kappa \in \cal G$ is the one with mean equal to $x$. Using these results in Theorem~\ref{lem:lower_bound}, we get the simplified lower bound for the Gaussian bandit models, which also holds under time-varying corruption (Remark \ref{rem:general_lb}). 

Recall that $m(\mu_a)$ denotes the mean reward of arm $a$ with distribution $\mu_a$. 
%\vspace{-0.25em}
\begin{proposition}[Lower bound for Gaussian bandits]\label{th:lower_bound_gauss} 
For $\mu\in\mathcal \cG^K$, any uniformly-good algorithm satisfies for any suboptimal arm $a$
  \[ \liminf\limits_{T\rightarrow \infty} \frac{1}{\log (T)} \left(\sup\limits_{{\bf H} \in\mathcal{P}(\R)^K}\Exp{\mu \corby {\bf H}}{N_a(T)}\right)~\ge~\frac{1}{\ekl_\cG(m(\mu_a),m^*(\mu))}~, \quad \text{where } \]   %where for any $x,y \in \R$, $x \le y$,
    \begin{equation}\label{eq:klg_cor} 
    \forall x,\!y\!\in\! \R,  x \!\le\! y\quad   \ekl_\cG(x,y)\!\defn\! \min\limits_{H,H'}\!\lrset{\KL( \G (x, 1) \!\corby \!H , \G (y, 1) \!\corby \!H')\!: ~ H, H'\!\in\! \mathcal{P}(\R)}.\!\!
    \end{equation}
    Here, the optimal pair of corruption distributions are given by Lemma~\ref{lem:pair_distrib}. 
\end{proposition}
%\vspace{-0.25em}

We now state the necessary and sufficient conditions to have a finite lower bound on regret in Gaussian bandits (Proposition~\ref{th:lower_bound_gauss}) for a known and fixed $\varepsilon\in\left(0,\frac{1}{2}\right)$. Thus, Lemma~\ref{lem:non-intersection} states the (necessary and sufficient) conditions to achieve logarithmic regret for Gaussian bandits.

%\vspace{-0.25em}
\begin{Lemma}[Disjoint corruption neighbourhoods]\label{lem:non-intersection}\!
For $\eta\in\cG$, $\kappa\in \cG$, the following are equivalent: 
\[{\bf (1)}~ \forall (H_1, H_2) \in \mathcal{P}(\R)^2, ~\kappa \corby H_1 \neq \eta \corby H_2, ~~ {\bf (2)}~ |m(\kappa)-m(\eta)|> 2 \Phi^{-1}\left( \frac{1}{2(1-\varepsilon)}\right).\]
\end{Lemma}
\vspace{-0.25em}
The condition $(1)$ above states that for any corruption distribution $H_1$, there doesn't exist a distribution $H_2$ such that the corrupted distributions $\kappa\corby H_1$ and $\eta\corby H_2$ are the same, rendering $\ekl_{\cal G}(m(\kappa), m(\eta)) = 0$, i.e., the corruption neighbourhoods (Definition~\ref{def:corr_nbd}) of $\kappa$ and $\eta$ are disjoint. The lemma above shows that this condition is equivalent to separation in the means of the two Gaussian distributions $\eta$ and $\kappa$. This is also related to the fact that in the presence of corruption, the mean of the true distributions can only be estimated up to an unavoidable error (Appendix~\ref{sec:non_inter}). We postpone the proof of Lemma~\ref{lem:non-intersection} to Section \ref{app:pf_non-intersection}. 

Lemma~\ref{lem:non-intersection}, when combined with Proposition~\ref{th:lower_bound_gauss}, implies that a suboptimal condition to ensure logarithmic regret is a separation between the means of the optimal arm and suboptimal arms. This justifies formally introducing the required minimum gap.
\vspace{-0.4em}
\begin{definition}[Minimum distinction gap under corruption] \label{def:deltamin} $\Delta_{\min} \defn 2\Phi^{-1}\left( \frac{1}{2(1-\varepsilon)}\right)$. 
\end{definition}\vspace{-0.4em}

For Gaussian distributions, KL can be expressed using CDF $\Phi$ and PDF $\varphi$ of a standard Gaussian. Moreover, the corrupted KL, viz. $\ekl_\cG$, enjoys nice properties like an almost closed-form expression, shift invariance, differentiability, etc., described in the following lemma.

%\vspace{-0.5em}
\begin{Lemma}[Properties of $\ekl_\cG(x,y)$]\label{lem:basics_ekl}\label{lem:deriv_K}
Let $H_1,H_2$ be minimisers in \eq~\eqref{eq:klg_cor}.  
\begin{enumerate}[label=(\alph{enumi}),topsep=0.75pt]%, noitemsep]
    \setlength\itemsep{0.5em}
    
    \vspace{0.25em}
    \item\label{norm_const} The normalisation constants $c_1$ and $c_2$ are equal, i.e., $c\defn c_1 = c_2$, and uniquely solve \vspace{-0.4em}
    \begin{equation}\label{eq:defn_c}
    {1}/\lrp{1-\varepsilon}=c\Phi(\Delta_-/2)+\Phi(\Delta_+/2),\vspace{-0.25em}
    \end{equation}
    with $\Delta_+ \defn \Delta + \frac{2}{\Delta}\log\frac{1}{c}$, and $\Delta_- \defn \Delta - \frac{2}{\Delta}\log\frac{1}{c}$.
    
    \item\label{shift_inv} $\ekl_\cG$ has an almost closed-form expression (\eq~\eqref{eq:Kg}) (up to $c$ defined as in Part~\ref{norm_const}). For $0 \le \Delta \le \Delta_{\min}$, $x\in\R$,  $\ekl_\cG(x,x+\Delta) = 0$. Moreover, it is invariant, i.e., \vspace{-0.4em}
    \[
        \ekl_\cG(x+\Delta,y)=\ekl_\cG(x,y-\Delta), \quad \text{ for } y \ge x+\Delta.\vspace{-.25em}
    \] 
    
    \item\label{differentiable} For $x\in\R$ and $\Delta \ge 0$, the function $\Delta \mapsto \ekl_\mathcal{G}(x,x+\Delta)$ is continuously differentiable. For $\varepsilon > 0$ and $\Delta \le \Delta_{\min}$, 
    %$c$ from Part~\ref{norm_const} equals $1$, and 
    ${\partial ~ \ekl_\mathcal{G} (x,x+\Delta)}/{ \partial \Delta} =0$.  For $\Delta > \Delta_{\min}$, %and the corresponding normalising $c$, 
    $$\frac{\partial ~ \ekl_\mathcal{G} (x,x+\Delta)}{ \partial \Delta} = \lrp{1-\varepsilon}\Delta\lrp{\Phi({\Delta_+}/{2}) - \Phi({\Delta_-}/{2}) } > 0.$$
\end{enumerate}
\vspace{-\topsep}
\end{Lemma}
They constitute the key properties of the corrupted divergence used in our regret analysis.
We refer the reader to Appendix~\ref{app:prop_ekl} for a complete proof of Lemma~\ref{lem:basics_ekl}, plus other interesting properties of $\ekl_\cG$. 

\vspace{0.25em}
\noindent\textbf{Consequences of Lemma~\ref{lem:basics_ekl}.} %The following facts will be important in proving results in Section~\ref{sub:regret_ub}.
Part~\ref{shift_inv} shows that there is a flat region below $\Delta = \Delta_{\min}$, where $\ekl_{\cal G}(x,x+\Delta)$ equals $0$ (Fig.~\ref{fig:klinf}). We use this property in  regret analysis of the algorithm for proving fast convergence of the empirical $\ekl_{\cG}$ to $0$, as well as to avoid certain computations at each step. Part~\ref{differentiable} shows that $\ekl_\cG$ is strictly increasing in the second argument for values larger than the first argument. This was used in Proposition~\ref{th:lower_bound_gauss} to conclude that $\eKinf(\eta,x;\cG) = \ekl_\cG(m(\eta), x)$.

\vspace{0.25em}

\noindent\textbf{Computational remarks.}
The normalising constant $c$ from Part~\ref{norm_const} implicitly depends on $\varepsilon$ and $\Delta$, and so do $\Delta_-$ and $\Delta_+$. Here, $\Delta_-$ and $\Delta_+$ are related to support sets of the optimal corruption pair $H_1$ and $H_2$ (Lemma~\ref{lem:GaussianCorruptionSupp} in Appendix~\ref{app:prop_ekl}). For $\Delta$ converging to $\Delta_{\min}$, $c$ can be shown to converge to $1$ with $\Delta_-$ and $\Delta_+$ converging to $\Delta_{\min}$. This can be seen from \eq~\eqref{eq:defn_c}. In this limit, from Lemma~\ref{lem:basics_ekl}\ref{differentiable}, it follows that the derivative of $\ekl_\cG$ converges to $0$.

We also note that unlike the classical Gaussian bandit, from Fig.~\ref{fig:klinf} we see that $\ekl_{\cal G}(x,x+\Delta)$ is non-convex in $\Delta$, implying that after a point, increasing $\Delta$ does not substantially decrease the number of pulls of suboptimal arms, and hence the regret, in presence of corruption.

\vspace{-0.25em}
\section{{\rimed}: Algorithm and analysis}\label{sec:algo}
%\vspace{-0.3em}
In this section, we leverage the lower bound in Proposition~\ref{th:lower_bound_gauss} to propose an algorithm robust to  corruption, namely \rimed. We then give a finite-sample upper bound on the regret of \rimed{}, showing its asymptotic optimality for Gaussian bandits with unbounded stochastic corruption. Finally, we explicate the technical novelty of our regret analysis. We discuss an extension of \rimed{} for handling model-misspecifications in Appendix~\ref{app:misspecified}.
\vspace{-0.4em}
\subsection{Algorithm design: An IMED-based algorithm with estimated medians}
\vspace{-0.3em}
First, we present our algorithm design. For $n\in\mathbbm{N}$ and $a\in [K]$, let $\hat{\mu}_a(n)$ denote the empirical distribution constructed using $N_a(n)$ samples from arm $a$. We use median of the corrupted observations as an estimator for the mean of underlying reward distributions. This choice is natural in the case of Gaussian distributions, since it is known that the median has the smallest \textit{bias due to corruption} among all location estimators in a corruption neighbourhood of the Gaussian (ref. Lemma \ref{lem:bias} and corresponding discussion in Section~\ref{sec:non_inter}). The fact that we use the median is also closely linked to the symmetry of the Gaussian distribution for which median is the same as mean. 

Let $\Med(\cdot)$ denote the median of the input distribution, and define the maximum estimated median at time $n$ as $\Med_*(n):=\max_{a}~  \Med(\hat{\mu}_a(n))$. We present \rimed{} in Algorithm~\ref{alg:robustimed}. Note that in Algorithm~\ref{alg:robustimed} we introduce a forced exploration for $N_{\min}$ steps, where for $T > 0$,%\vspace{-0.5em}
\begin{equation}\label{eq:nmin}
    N_{\min}\defn\left\lceil \frac{2\log(T) \log(1\!+\!\log(1\!+\!\log(T)))^2 s_{\varepsilon}^2}{\log(1\!+\!\log(T)^{0.99})}\right\rceil,\! \,
    s_\varepsilon \defn \frac{\lrp{\frac{{\varepsilon}/{2}}{\log\frac{1}{1-2\varepsilon}}}^\frac{1}{2}  + \lrp{\frac{1-2\varepsilon}{4\log\lrp{\frac{1-\varepsilon}{\varepsilon}}}}^\frac{1}{2}}{\lrp{1-\varepsilon}\varphi\left(\frac{\Delta_{\min}}{2}+1\right)}  .%\vspace{-0.5em}
\end{equation}

%\begin{equation}\label{eq:s_eps}
%    s_\varepsilon \defn \frac{{1}/\lrp{1-\varepsilon}}{\varphi\left({\Delta_{\min}}/{2}+1\right)}\left(\frac{\varepsilon^\frac{1}{2}}{\sqrt{2}\log^\frac{1}{2}\left({1}/\lrp{1-2\varepsilon} \right)}  + \frac{(1-2\varepsilon)^\frac{1}{2}}{2\log^\frac{1}{2}\lrp{\rp{1-\varepsilon}/{\varepsilon}}}\right).
%\end{equation}

Here, $\!s_\varepsilon$ is a proxy of the variance of the empirical median from Theorem~\ref{th:concentration_median}. It converges to a constant, $\!\frac{1}{2\varphi(1)}$, as $\varepsilon\!\rightarrow\!0$. The amount of forced-exploration $N_{\min}$ is $o(s_\varepsilon^2 \log T)$ as $T\!\to\!\infty$. Since we use empirical median as an estimate for the true mean using the corrupted observations, the empirically-optimal arm, or the arm with the maximum estimated mean is defined as $a^*(n):= \argmax_{b} ~ \Med(\hat{\mu}_b(n))$. Moreover, since for this arm $\ekl_\cG( \Med(\hat{\mu}_{a^*(n)}(n)) - \Delta_{\min} ,  \Med_*(n)) = 0$, its index is trivial to compute (Lemma~\ref{lem:basics_ekl}\ref{shift_inv}). For other arms, we use the explicit formulation for  $\ekl_\cG$ (Lemma~\ref{lem:basics_ekl}\ref{shift_inv}), while we compute the constant $c$ using a root-finding algorithm on \eq~\eqref{eq:defn_c}. 

%\vspace{-0.5em}
\begin{algorithm2e}[t!]
\RestyleAlgo{ruled}
\caption{\rimed{} for unit variance Gaussian bandits}\label{alg:robustimed}
    \SetKwInOut{KwInput}{Input}
    \KwInput{Horizon $T$, Corruption level $\varepsilon$, $K$}
    \textbf{Initialisation phase:} Compute $N_{\min}$ using \eq~\eqref{eq:nmin} and pull every arm $N_{\min}$ times. \\
    \For{$n\in\lrset{KN_{\min}+1, \dots,T-1, T}$}{
        Set $\Med_*(n)\gets \max_a \Med(\hat{\mu}_a(n))$, $A_n^* \in \argmax_a \Med(\hat{\mu}_a(n))$, $I_{A^*_n}(n) \gets \log N_{A^*_n}(n)$.\\
        Compute, for each arm $a$ different from $A_n^*$, \vspace{-0.5em}
        $$I_a(n) \gets N_a(n) \ekl_\cG\lrp{\Med(\hat{\mu}_a(n))- \Delta_{\min}, \Med_*(n)} + \log N_a(n).$$ \vspace{-1em}
        Pull the arm $A_n \in \argmin_a I_a(n) $.\\
    }
\end{algorithm2e}
%\vspace{-0.5em}

\vspace{-.5em}
\subsection{Theoretical results: Regret upper bound and concentration results}\label{sub:regret_ub}
\vspace{-0.3em}

We now present the theoretical guarantees of the proposed algorithm, as well as the refined concentration inequality for median that play a key role in our analysis, and are of independent interest.

%\vspace{-0.25em}
\begin{Theorem}[Finite-sample regret upper bound]\label{th:upper_bound}
For $\varepsilon \in (0,\frac{1}{2})$ and $\mu\in\cG^K$ such that for each suboptimal arm $a$, $\Delta_a >\Delta_{\min}$, \rimed{} satisfies
\[ \E[N_a(T)] \le N_{\min} +  \frac{\log(T)}{ \ekl_\cG\lrp{m(\mu_{a}), m^*(\mu)}-2\delta\left(\Delta_a + \delta + \Delta_{\min}\right)} + O((\log T)^{0.99}),\]
where $m^*(\mu)=\max_a m(\mu_a)$, and $\delta^2:=\lrp{\log(1+\log(1+\log T))}^{-1}$.
\end{Theorem}

The exact $O((\log T)^{0.99})$ term can be found in \eq~\eqref{eq:finite_ub} in the Appendix~\ref{app:regret_ubound}. We believe that the forced exploration for $N_{\min}$ steps is an artefact of the proof, and is needed in our analysis to handle the difficulties due to corruption. Indeed, in Section~\ref{sec:xp}, we numerically compare {\rimed} and an aggressive version with $N_{\min} = 1$, called \rimedstar{}, and observe a smaller regret for \rimedstar{}. %\shubhada{Discuss these.}

\begin{corollary}[Asymptotic optimality]
For $\varepsilon\in (0,\frac{1}{2})$ and $\mu \in \cG^K$ such that for each suboptimal arm $a$, $\Delta_a > \Delta_{\min}$, \rimed{} is asymptotically optimal, i.e., $\limsup\limits_{T\rightarrow\infty} \frac{\E\lrs{N_a(T)}}{\log T} \le \frac{1}{\ekl_{\cG}(m(\mu_a), m^*(\mu))}. $
\end{corollary}

The proof of Theorem~\ref{th:upper_bound}, which we detail in Appendix~\ref{app:regret_ubound}, proceeds by controlling the probability of selecting a suboptimal arm $a$ at each step $n$. \rimed{} pulls arm $a$ at time $n$ if its index $I_a(n)$ is the smallest. Thus, the probability of pulling an arm can be bounded by controlling the deviations of $\ekl_\cG$ evaluated on the empirical estimates. This in turn is related to the probability of deviation of the empirical estimates themselves. 
In Theorem~\ref{th:concentration_median}, we prove a new concentration result for the empirical median, which we leverage to prove that the regret of \rimed{} is well controlled. Later, in Lemma~\ref{lem:concentration_ekl}, we present the concentration results for the $\ekl_\cG$ evaluated at the empirical estimates. Given $n$ samples $X_1, \dots, X_n$, $\Med(X^n_1)$ denotes the empirical median. This can be alternatively seen as $\Med(\hat P_n)$, where $\hat P_n$ denotes their empirical distribution of $X_1,\dots, X_n$. 

%\vspace{-0.25em}
\begin{Theorem}[Concentration of median for corrupted Gaussians]\label{th:concentration_median}
%For $H\in\mathcal P(\R)$, let $X_1,\dots,X_{n}$ be i.i.d. samples from $ \G(m,1) \corby H$ and let $\varepsilon < \frac{1}{2}$. For $y \in [0,1]$, 
Let $X_1,\dots,X_{n}$ be independent samples such that $ X_i 
\sim \G(m,1) \corby H_i$, for $H_i \in\mathcal P(\R)$, and let $\varepsilon < \frac{1}{2}$. For $y \in [0,1]$, 
$$ \P \left(  \Med(X_1^{n})-m \ge \frac{\Delta_{\min}}{2} +y\right) \vee \P \left(  \Med(X_1^{n})-m \le - \frac{\Delta_{\min}}{2} -y\right) \le  2\exp \left(\frac{-ny^2}{s_\varepsilon^2} \right).$$
\end{Theorem}

We remark that Theorem \ref{th:concentration_median} does not require i.i.d. samples. In particular, the corruption distribution is allowed to be time-varying. Further, observe from \eq~\eqref{eq:nmin} that when $\varepsilon$ goes to $0$, $s_{\varepsilon}$ goes to $\frac{1}{2\varphi(1)}$. On the other hand, when $\varepsilon$ goes to $\frac{1}{2}$,  $\varphi(\frac{\Delta_{\min}}{2}+1)$ goes to $0$, hence $s_{\varepsilon}$ goes to $\infty$,  and we get a trivial bound in the theorem. In particular, our bound adapts to $\varepsilon$. We refer the reader to Appendix~\ref{sec:proof_med_concentration} for a proof of the theorem.  

\vspace{-0.35em}
\begin{remark}[A refined concentration result]
\emph{Theorem~\ref{th:concentration_median} is an improvement over the concentration in \citet[Lemma 7]{JMLR:v20:18-395} in which the variance term does not depend on $\varepsilon$. Furthermore, Theorem~\ref{th:concentration_median} allows for an $\varepsilon$ that is arbitrarily close to $\frac{1}{2}$, which is an improvement over the existing bounds for robust mean estimators featuring an upper limit on $\varepsilon$ away from $\frac{1}{2}$. For example, $\varepsilon \le \frac{1}{7}$ in \cite[Theorem 2]{wang2023huber}, and $\varepsilon\le \frac{1}{15}$ in  \cite[Theorem 18]{JMLR:v20:18-395}.}
\end{remark}
\vspace{-0.35em}

Using Theorem~\ref{th:concentration_median} and properties of $\ekl_{\cG}$, we prove the following concentration result. 

%\vspace{-0.5em}
\begin{Lemma}[Concentration of $\ekl_{\cG}$.]\label{lem:concentration_ekl}
Let $\delta > 0$, $x\in \R$. Let $X_1,\dots,X_n$ be $n$ independent samples such that $X_i\sim\G(m_a, 1)\corby H_i$, for $H_i \in \mathcal P(\R)$. 

\setlength{\textfloatsep}{0pt}
\begin{enumerate}[label=(\alph{enumi}), partopsep = 0pt]

\item \label{conc_0} For $y \in [0,1]$, with probability at least $1 - 2\exp\left(-n y^2/s_\varepsilon^2 \right)$,\vspace*{-0.6em}
$$\ekl_\cG\left(\Med(X^n_1)-\frac{\Delta_{\min}}{2}, m_a - \delta \right) \le \left(y-\delta \right)_+\left(|y-\delta|+\frac{\Delta_{\min}}{2}\right).$$ 

\item \label{conc_other} For $m_b > m_a + \Delta_{\min}$ and $y\in [0, 1]$,  with probability at least $1 - 2\exp\left(-n y^2/s_\varepsilon^2 \right)$,\vspace{-0.25em}
$$
\ekl_\cG\left(m_a, m_b \right)-\ekl_\cG\left(\Med(X^n_1)-\frac{\Delta_{\min}}{2}, m_b \right)\le    y (m_b-m_a+y+\Delta_{\min}).\vspace{-0.25em}$$
\end{enumerate}
\end{Lemma}%\vspace*{-0.6em}

\noindent{\bf A useful thresholding.} For $y\in [0,\delta)$ the probability of  $\ekl_\cG(\Med(X^n_1)-\frac{\Delta_{\min}}{2},  m_a-y )$ being $0$ is strictly positive, from  Lemma~\ref{lem:concentration_ekl}\ref{conc_0}. This contrasts with the uncorrupted-Gaussian setting, for which this is a zero probability event, except when $y = 0$. We extensively use this key property in the proof of Theorem~\ref{th:upper_bound}, which follows from the thresholding property that $\ekl_\cG\left(x, x+\Delta\right)\!=\!0$ holds for any $\Delta\!\le\! \Delta_{\min}$. Specifically, the probability of being $0$ coincides with that for $m_a-y-\Med(X^n_1)\le \frac{\Delta_{\min}}{2}$. We refer the reader to Appendix~\ref{app:lem:concentration_ekl} for a proof of the lemma. 
%\end{remark}

\vspace{.45em}

\noindent{\bf Challenges in the regret analysis.}\label{sec:sketch}
%\oam{Provide a high-level sketch of proof, explaining that are the steps, and the main difficulty compared to the non-corrupted setting}
To prove Theorem~\ref{th:upper_bound}, we modify the proof for the regret bound of~\cite{JMLR:v16:honda15a}. The major difference arises from the fact that Theorem~\ref{th:concentration_median} \emph{does not allow us to reach arbitrarily large level of confidence}, i.e. with $y\le 1$, the probability in Theorem~\ref{th:concentration_median} cannot be smaller than $\exp(-\Omega(n))$. This implies that very large deviation of the median do not imply very small probabilities. This is a known limitation of robust estimators~\citep{subgaussian}. As a consequence, we change the decomposition of the bad event $A_n = a$ to also include an event on which the deviation of the $\ekl_{\cG}$ is large.
%For simplicity of notation, we assume that arm $1$ is the optimal arm in $\mu$. 

We decompose the event $A_n = a$ as a union of three disjoint events. (i) $E_n(a)$: when the suboptimal arm is not well estimated. (ii) $F_n(a)$: when the optimal arm is not well estimated and $G_n(a)$ when $\ekl_{\cG}$ has large deviations (ref. Lemma~\ref{lem:events} for formal definitions). We highlight that Lemma~\ref{lem:concentration_ekl}\ref{conc_0} with $y=\delta$, i.e. the fast concentration to $0$, is specifically used to control the probabilities of events $F_n(a)$ and $G_n(a)$. 
(iii) Event $G_n(a)$ is further controlled thanks to the forced-exploration mechanism. Indeed, we observe that even refined concentration such as anytime concentration might only improve the lower order terms in the regret upper bound, but are not sufficient to control $G_n(a)$. 

%\oam{We should mention Lemma 4 more, as we say before it is used extensively. Should we say it is used to control $F_n(a)$ and $G_n(a)$?  Also, since we have more space now, we may elaborate more on the proof.}
%\shubhada{Yes, it is used in controlling $F_n$ and $G_n$. Forced exploration is specifically to control the double summation in $G_n$. Even anytime concentration won't work for $G_n$ }
%\shubhada{But maybe, a clever decomposition }

\vspace{-1em}
\section{Experimental illustration}\label{sec:xp}%\vspace{-0.3em}
\vspace{-0.3em}
In this section, we numerically illustrate the efficiency of our algorithm. The computation of \rimed{} indices depends on the threshold $c$, that we evaluate using \eq~\eqref{eq:defn_c} and the default \textit{scipy}~\citep{2020SciPy-NMeth} root-finding algorithm.
We consider arm distributions as $\G(m_a, \sigma_a^2)$, i.e., Gaussian inliers. The corruption distributions (outliers) are of form $\G(m_o,\sigma_o^2)$ in Setting $1$ and $2$, and standard Cauchy in Setting $3$. Table below details the parameters used. Arm $3$ is optimal.

\vspace{-0.25em}
\begin{center}
\begin{tabular}{ |c|c|c|c|c|c|c|c| } 
 \hline
 Parameters & Horizon & $\sigma_a$ & means  arms $m_a$ & $\varepsilon$ & medians outliers $m_b$& $\sigma_o$ outliers\\ 
 \hline
 Setting 1 & 10,000 & 0.5 & [0.8, 0.9, 1] & 0.01 &[1, 1, 0.8] & 1\\ 
 Setting 2 & 10,000 & 0.5 & [0.8, 0.9, 1] & 0.01 & [10, 10, -20] & 1\\
  Setting 3 & 10,000 & 0.5 & [0.8, 0.9, 1] & 0.01 & [10, 10, -20] & $\infty$\\
 \hline
\end{tabular}
\end{center}
\vspace{-0.25em}

Setting $1$ corresponds to a mild corruption, in which the corrupted distribution of arm $3$ still has the largest mean of $0.98$. In Setting $2$, the corruption causes a change in the order of the arms (arm $2$ is optimal according to the corrupted distributions). Hence, robustness is needed to identify correctly the optimal arm.  In Setting $3$, the outliers are heavy-tailed.
We compare $4$ algorithms: {\rimed} (Algorithm~\ref{alg:robustimed}) with $N_{\min}$ set to \eq~\eqref{eq:nmin}, \rimedstar{}, an aggressive version of \rimed{} with $N_{\min}$ set to $1$, \texttt{IMED} is the same as \rimedstar{} but in which there is no corruption ($\varepsilon=0$) and the means are estimated using the empirical mean, and finally \texttt{RobustUCB} ~\citep{basu2022bandits}. 
\begin{figure}[t!]
\centering
\includegraphics[width=0.32\textwidth, trim=11 11 11 11,clip]{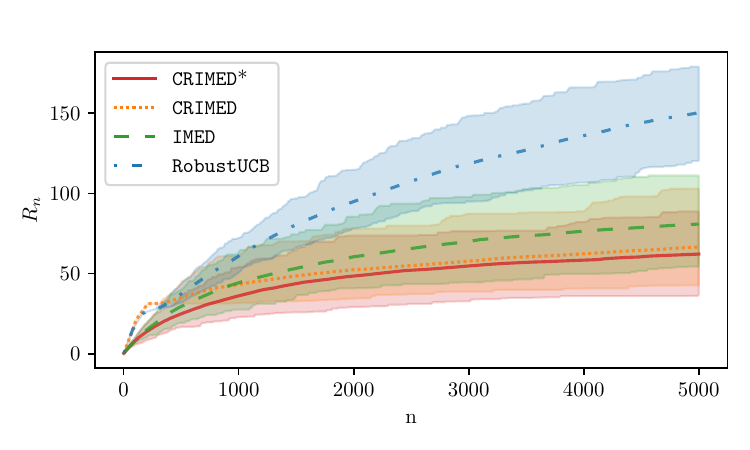}
\hfill
\includegraphics[width=0.32\textwidth, trim=11 11 11 11,clip]{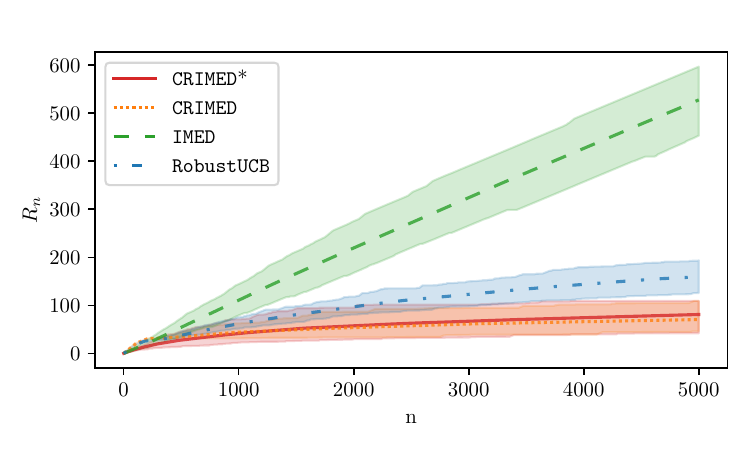}
\hfill
\includegraphics[width=0.32\textwidth, trim=11 11 11 11,clip]{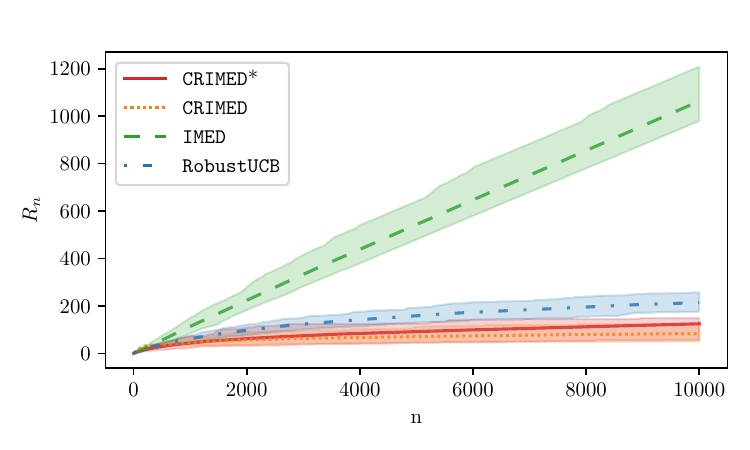}
\vspace{-0.5em}
 %  \subfigure[Setting 1]{
 %        \centering
 %     \includegraphics[width=0.5\textwidth]{FIG/median_gaussian.pdf}
 %     }\hfill
 %   \subfigure[Setting 2]{
 %        \centering
 %     \includegraphics[width=0.5\textwidth]{FIG/median_gaussian_extra.pdf}
 %   }
 \vspace{-0.5em}
    \caption{Cumulative regret for $100$ repetitions on Settings $1$ (left), $2$ (middle), and $3$ (right). Solid lines represent the means and shaded area are $90\%$ percentile intervals.}\vspace{-1em}
    \label{fig:xp}
\end{figure}

Figure~\ref{fig:xp} illustrates that all the algorithms, except \texttt{IMED}, feature a logarithmic regret. \texttt{IMED} incurs a linear regret in large-corruption settings (Settings $2$, $3$). On the other hand, \rimed{} and \rimedstar{} perform comparably, and they are both performing  better than \texttt{RobustUCB}. We also observe no significant difference in performance when the corruptions are heavy-tailed vs Gaussian.

\vspace*{-1em}
\section{Discussion and open questions}\vspace*{-.3em}
We studied a variant of the stochastic multi-armed bandits with unbounded stochastic corruption and analysed the behaviour of the KL divergence within a corruption neighbourhood in details. This served as the foundation for the proposed algorithm \rimed, that asymptotically achieves this lower bound for Gaussian bandits. We developed a new concentration result for median allowing \rimed{} to tackle corruption proportion up to $\frac{1}{2}$, which was not possible with the existing algorithms. We discussed a modification of \rimed{} to handle misspecifications in Gaussian model. 

We believe that the Gaussian reward assumption is not essential, and we believe it is possible to generalise these results to at least symmetric, unimodal distributions. Generalisation  to non-symmetric and possibly non-parametric inliers is likely to be much more challenging. This is because in the non-symmetric case, robust estimators approximate location parameters that may be far from the mean and we have to trade off between the error due to corruption and the error due to asymmetry. This requires the study of a non-trivial trade-off between robustness and asymptotic distance to the mean, which we leave as an open question.

% Acknowledgments---Will not appear in anonymized version
\acks{This work commenced when S. Agrawal was a PhD student at TIFR, Mumbai, India. S. Agrawal acknowledges support from the Department of Atomic Energy, Government of India, under project no. RTI4001. This work has also been supported by the French Ministry of Higher Education and Research, the Hauts-de-France region, Inria, the MEL, the	I-Site ULNE regarding project R-PILOTE-19-004-APPRENF, and the Inria A.Ex. SR4SG project. D. Basu, O.-A. Maillard, and T. Mathieu acknowledge the Inria-Kyoto University Associate Team ``RELIANT'' for supporting the project. D. Basu also acknowledges the ANR JCJC grant for the REPUBLIC project (ANR-22-CE23-0003-01). Additionally, S. Agrawal acknowledges support from the Sarojini Damodaran Fellowship and the Google PhD Fellowship in Machine Learning.}

\bibliography{bibliography.bib}

\newpage
\appendix
%\part{Appendix}%\vspace*{-10em}
%\parttoc
%\clearpage
%\appendix
\section{Proofs for results in Section~\ref{sec:lower_bound} and Section~\ref{sec:huberpair}: Lower bound and hardest corruption pair}

\begin{remark}\label{rem:general_lb}
\emph{Observe that }
\[\sup\limits_{{\bf H_T} \in\mathcal{P}(\R)^{T\times K}}\Exp{\mu \corby {\bf H_T}}{N_a(T)} \ge \sup\limits_{{\bf H} \in\mathcal{P}(\R)^{K}}\Exp{\mu \corby {\bf H}}{N_a(T)},\]
\emph{where $\bf H_T$ is a $T\times K$ matrix of corruption distributions with row $n$ being ${\bf H}_n$, the vector of corruption distributions at time $n$. Furthermore, on the rhs in the above inequality, we assume that the corruption distributions associated with each arm are same across time. Below, we show that the lower bound for this simpler setting where the corruption doesn't change with time. Since the lower bound for fixed $H$ across time will only be smaller, this then proves that the same lower bound holds even for the more general setting considered in this work that allows for time-varying corruption.}
\end{remark}

\subsection{Proof of Theorem~\ref{lem:lower_bound}}\label{app:lb} Let the given bandit instance be denoted by $\mu\in\mathcal L^K$ and for $a\in[K]$ and let $\lrset{X_{a,i}}_{a,i}$, for $a\in[K], i\in \lrset{1, \dots, N_a(T)}$ denote the $N_a(T)$ corrupted observations from arm $a$ till time $T$. Under the instance $\mu$ with the corruption distributions ${\bf H} = \lrset{H_1, \dots, H_K}$, the likelihood of observing the samples, denoted by $L_{\mu,\bf H}$, is 
\[ L_{\mu,\bf H} = \prod\limits_{a=1}^K \prod\limits_{i=1}^{N_a(T)} \lrp{( 1-\varepsilon ) \mu_a(X_{a,i}) + \varepsilon H_a(X_{a,i})} = \prod\limits_{a=1}^K \prod\limits_{i=1}^{N_a(T)} \mu_a\corby  H_a(X_{a,i}). \]

Without loss of generality, we assume that arm $1$ is the unique optimal arm in $\mu$ and establish the lower bound for the sub-optimal arm $2$. To this end, consider an alternative bandit instance, $\nu = (\nu_1, \dots, \nu_K)$, where, $\nu_b\in\cal L$ for each $b\in [K]$, for $b\ne 2$, $\nu_b = \mu_b$, and $m(\nu_2) \ge m^*(\mu)$. Clearly, $m^*(\nu) \ge m^*(\mu)$. The likelihood of observing samples under $\nu$ with corruption distributions ${\bf H}' = \lrset{H'_1, \dots, H'_K}$, denoted by $L_{\nu,\bf H'}$ is given by 
\[ L_{\nu,\bf H'}  =\prod\limits_{a=1}^K \prod\limits_{i=1}^{N_a(T)}\nu_a  \corby H'_a(X_{a,i}) . \]
Writing the log-likelihood ratio, 
\[
    LL_T =  \sum\limits_{a=1}^K \sum\limits_{i=1}^{N_a(T)}\log\lrp{\frac{\mu_a\corby  H_a}{\nu_a\corby  H'_a}(X_{a,i})}. % = \sum\limits_{i=1}^{N_2(n)} \log \frac{\mu^{\varepsilon,H}_2}{\nu^{\varepsilon,H'}_2}(X_{2,i}).
\]
Taking average with respect to $\mu\corby  \bf H$, we get
\[ 
    \Exp{\mu\corby  \bf H}{LL_T} = \sum\limits_{a=1}^K \Exp{\mu\corby  \bf H}{N_a(T)} \KL(\mu_a\corby  H_a, \nu_a\corby  H'_a). 
\]
    
Informally, let $\mathcal F_t$ be the $\sigma-$algebra generated by the randomness of the algorithm and the observations up to time $t$ (see \cite[Chapter 4]{lattimore2018bandit} for formal introduction to stochastic multi-armed bandits). An application of the data-processing inequality (see, \citet{kaufmann2016complexity,garivier2019explore}) gives that for any $\mathcal F_T$ measurable event $\mathcal E_T$,
\begin{equation*} \label{eq:transportation} 
\sum\limits_{a=1}^K \Exp{\mu\corby  \bf H}{N_a(T)} \KL(\mu_a\corby  H_a, \nu_a\corby  H'_a) \ge d\lrp{ \mu_a\corby  H_a(\mathcal E_T), \nu_a\corby  H'_a(\mathcal E_T) }, \end{equation*}
where  for $x\in (0,1)$ and $y\in (0,1)$, $d(x,y) \defn \KL(\Ber(x),\Ber(y))$ denotes the $\KL$ divergence between Bernoulli distributions with the given means. Since the r.h.s. above is true for all events  $\mathcal E_T$ that are $\mathcal F_T$ measurable, optimizing over them we get  
\begin{equation*} \sum\limits_{a=1}^K \Exp{\mu\corby  \bf H}{N_a(T)} \KL(\mu_a\corby  H_a, \nu_a\corby  H'_a) \ge \sup\limits_{\mathcal E_T \in \mathcal F_T} ~ d\lrp{ \mu_a\corby  H_a(\mathcal E_T), \nu_a\corby  H'_a(\mathcal E_T) }. \end{equation*}
Taking infimum over the corruptions ${\bf H}\in {\mathcal P(\R)}^K$ and ${\bf H'}\in {\mathcal  P(\R)}^K$ on both sides, the above inequality implies 
    \begin{align} \label{eq:transportation} 
    \inf\limits_{\bf H,\bf H'}~ \sum\limits_{a=1}^K~ &\Exp{\mu\corby  \bf H}{N_a(T)}\KL(\mu_a\corby  H_a, \nu_a\corby  H'_a) \nonumber \\
    &\ge \inf\limits_{\bf H,\bf H'}~ \sup\limits_{\mathcal E_T \in \mathcal F_T} ~ d\lrp{ \mu_a\corby  H_a(\mathcal E_T), \nu_a \corby H'_a(\mathcal{E}_T)}. \end{align}
    Since for every $\bf H$, $\bf H'$ with $H'_b = H_b$ for all $b\in[K]$ and $b\ne 2$ is a feasible candidate for $\bf H'$, the infimum in the l.h.s. above is at most 
    \begin{equation} \label{eq:transportation_lhs}
        \lrp{\sup\limits_{\bf H}~ \Exp{\mu\corby \bf  H}{N_2(T)}}\lrp{ \inf\limits_{H_2,H'_2} \KL(\mu_2\corby H_2, \nu_2 \corby H'_2)}. 
    \end{equation} 

    Now, following along the arguments in \cite{kaufmann2016complexity} for the classical regret-minimization setting without corruption, we first choose 
    \[ \mathcal E_T =  \lrset{ N_1(T) \le T - \sqrt{T} }. \]
    Then, we obtain by a simple application of Markov's inequality
    \[ \mathbb{P}_{\mu\corby \bf  H}\lrp{ \mathcal E_T } = \mathbb{P}_{\mu\corby \bf  H}\lrp{ T - N_1(T)  \ge \sqrt{T} } \le \frac{\sum\limits_{a\ne 1} \Exp{\mu \corby \bf  H}{N_a(T)} }{\sqrt{T}} =: P^{\bf H}_T, \]
    and
    \[ \mathbb{P}_{\nu\corby\bf  H'}\lrp{ \mathcal E^c_T } = \mathbb{P}_{\nu\corby\bf  H'}\lrp{ N_1(T)  \ge T- \sqrt{T} } \le \frac{\sum\limits_{a\ne 2} \Exp{\nu\corby\bf  H'}{N_a(T)} }{T-\sqrt{T}} =: (Q^{\bf H'}_T)^c.\]
    Next, recall that the algorithm under consideration is uniformly-good (see Definition \ref{def:unif_good}). This implies
    \[ \sup\limits_{\bf H} ~ \mathbb{P}_{\mu\corby\bf  H}\lrp{ \mathcal E_T }  \le P_T:=\sup\limits_{H}~P^{H}_T \xrightarrow{T\rightarrow \infty}{0}, 
\]
and
\[ \sup\limits_{\bf H'}~ \mathbb{P}_{\nu\corby \bf H'}\lrp{ \mathcal E^c_T }  \le Q^c_T:= \sup\limits_{\bf H'}~(Q^{\bf H'}_T)^c \xrightarrow{T\rightarrow \infty}{0}.\]
    Clearly, for a fixed $\bf H$ and $\bf H'$, from the monotonicity of $d(\cdot,\cdot)$ in its arguments, it holds
    \begin{align} \label{eq:transportation_rhs}
        d\lrp{ \mathbb{P}_{\mu\corby \bf H}\lrp{\mathcal E_T}, \mathbb{P}_{\nu\corby\bf  H'}\lrp{\mathcal E_T} } &\ge d\left(\sup\limits_{\bf H}  \mathbb{P}_{\mu\corby \bf H}\lrp{ \mathcal E_T } , 1-\sup\limits_{\bf H'} \mathbb{P}_{\nu\corby \bf H'}\lrp{ \mathcal E^c_T }\right) \nonumber \\ 
        &\ge d(P_T,Q_T).
    \end{align}
    Using \eqref{eq:transportation_lhs} and \eqref{eq:transportation_rhs} in \eqref{eq:transportation}, we get 
    \begin{equation} \label{eq:lb.inequality}
        \lrp{\sup\limits_{\bf H}~ \Exp{\mu\corby \bf H}{N_2(n)}}\lrp{ \inf\limits_{H_2,H'_2}~ \KL(\mu_2 \corby  H_2, \nu_2 \corby H'_2)} \ge d(P_T,Q_T). 
    \end{equation}
    Next, we consider  the following relation
    \[  \lim\limits_{T\rightarrow \infty}\frac{d\lrp{ P_T, Q_T }}{\log T} = \lim\limits_{T\rightarrow \infty} \frac{1}{\log T} \log\frac{1}{Q^c_T } \ge \lim\limits_{T\rightarrow \infty} \frac{1}{\log T} \log\frac{T-\sqrt{T}}{\sup\limits_{\bf H'}\sum\limits_{a\ne 2} \Exp{\nu\corby \bf H'}{ N_a(T)}} .   \]
    We observe that the r.h.s. above equals 
    \[ \lim\limits_{T\rightarrow \infty} \lrp{ 1 + \frac{\log\lrp{1 - \frac{1}{\sqrt{T}}}}{\log T}  - \frac{ \log\lrp{ \sup\limits_{\bf H'} \sum\limits_{a \ne 2} \Exp{\nu \corby \bf H'}{N_a(T)} } }{\log T} }, \]
    which in turns equals $1$. Thus, we have obtained
    \[\lim\limits_{T\rightarrow \infty}\frac{d\lrp{ P_T, Q_T }}{\log T} \ge 1 .\]
    Using this in \eq~\eqref{eq:lb.inequality}, 
    \[ \liminf\limits_{T\rightarrow \infty} \frac{\lrp{\sup\limits_{\bf H}~ \Exp{\mu\corby \bf  H}{N_2(T)}}}{\log T} \ge \frac{1}{{ \inf\limits_{H_2,H'_2}~ \eKinf(\mu_2 \corby H_2, \nu_2\corby H'_2)}}. \]
    Since the above inequality is true for all the alternative bandit instances $\nu \in \mathcal L^K$ with $m(\nu_2) \ge m^*(\mu)$, we optimize over these to get 
    \[ \liminf\limits_{T\rightarrow \infty} \frac{\lrp{\sup\limits_{\bf H}~ \Exp{\mu\corby\bf  H}{N_2(T)}}}{\log T} \ge \frac{1}{\KLinf(\mu_2, m^*(\mu); \cal L)}, \]
    where $\eKinf(\mu_2, m^*(\mu);\cal L)$ equals 
    $$ \inf\lrset{\KL( \mu_2 \corby  H_2, \nu_2 \corby H'_2 ): ~ \nu_2\in\mathcal L, ~  m(\nu_2) \ge m^*(\mu), ~ H_2 \in \mathcal P(\R), ~ H'_2 \in \mathcal P(\R)}.  $$

%\section{Proofs for results from Section~\ref{sec:huberpair}}

\subsection{Proof of Lemma~\ref{lem:pair_distrib}}\label{sec:proof_pair_distrib}
Given $\eta\in\cal L$ and $\kappa \in \cal L$, we will show that $H_1$ and $H_2$ satisfying~\eqref{eq:contamination1} and~\eqref{eq:contamination2} are optimal for $\eKinf$, for a fixed $\kappa$. To this end, consider any alternative corruption distributions, $H'_1\in\mathcal P(\R)$ and $H'_2\in\mathcal P(\R)$. For $t \in (0,1)$, define  
\[H_{i,t} = (1-t) H_i + t H'_i, \quad  \text{ for } i\in \lrset{1,2}, \] 
and 
\[J_{H'_1,H'_2}(t) =\frac{1}{\varepsilon}\KL( \eta \corby  H_{1,t} , \kappa \corby H_{2,t}). \]

To prove the lemma, we show that $J_{H'_1,H'_2}$ is a convex function that is minimized at $t=0$. To see this, 

\begin{align*}
\frac{\d J_{H'_1,H'_2}}{\d t}(t) =& \int \log\frac{\d \eta \corby  H_{1,t} }{\d \kappa \corby  H_{2,t} }(x) \lrp{\d H'_1 - \d H_1}(x) \\
    &\qquad - \int\d \eta \corby  H_{1,t}(x) \left( \frac{\lrp{\d H'_2 - \d H_2}}{\d \kappa \corby  H_{2,t} }(x) - \frac{\lrp{\d H'_1 - \d H_1}}{\d \eta \corby  H_{1,t} }(x) \right)\\
=& \int \log\frac{\d \eta \corby  H_{1,t} }{\d \kappa \corby  H_{2,t} }(x) \lrp{\d H'_1 - \d H_1}(x) - \int\frac{\d \eta \corby  H_{1,t}}{\d \kappa \corby  H_{2,t} }(x)\left(\d H'_2 - \d H_2\right)(x),
\end{align*}
where the last equality follows from the fact that $H_1$ and $H_1'$ both integrate to $1$. Differentiating again with respect  to $t$, 

\begin{align*}
\frac{{\d}^2 J_{H'_1,H'_2}}{\d t^2}(t) 
%=& \int \left( \frac{\d H'_1-\d H_1}{\d \eta \corby H_{1,t}}-\frac{\d H'_2-\d H_2}{\d \kappa \corby H_{2,t}}\right)\lrp{\d H'_1 - \d H_1}(x) \\
%& \qquad - \int\left(\frac{(\d H'_1-\d H_1)}{\d \kappa \corby  H_{2,t} }(x)- \frac{\d \eta\corby H_{1,t} (\d H'_2-\d H_2)}{(\d \kappa\corby H_{2,t})^2}(x)\right)\left(\d H'_2 - \d H_2\right)(x)\\
=& \int \left( \frac{\d H_1'-\d H_1}{\sqrt{\d \eta \corby H_{1,t}}}-\sqrt{\d \eta \corby H_{1,t}}\frac{\d H_2'-\d H_2}{\d \kappa \corby H_{2,t}}\right)^2\\ 
&\ge 0,
\end{align*}
proving the convexity of $J_{H'_1,H'_2}$ for any $H_1',H_2'$. Thus, it suffices to prove that its derivative is non-negative at $t=0$. 

We now define the sets
\begin{equation}   \label{eq:supp}
    A := \lrset{x : ~ \frac{d \eta}{\d \kappa}(x) < c_1}\quad \text{ and } \quad D := \lrset{x: ~ \frac{\d \eta}{\d \kappa}(x) > \frac{1}{c_2} }.
\end{equation}

Then, 
\begin{align*}
    \frac{\d J_{H'_1,H'_2}}{\d t}(0) 
    %&= \int \log\frac{\d \eta \corby  H_{1} }{\d \kappa \corby  H_{2} }(x) \lrp{\d H'_1 - \d H_1}(x) - \int \frac{\d \eta \corby  H_{1} }{\d \kappa \corby  H_{2} }(x) \lrp{\d H'_2 - \d H_2}(x)  \\
    &= \int\limits_{A} \log \lrp{c_2} \lrp{\d H'_1 - \d H_1}(x) + \int\limits_{A^c\cap D^c} \log\frac{\d \eta}{\d \kappa}(x) \d H'_1(x) + \int\limits_{D} \log\frac{1}{c_1} \d H'_1(x) \\
    &\qquad \qquad   - \int\limits_{A} c_2 \d H'_2(x) - \int\limits_{A^c\cap D^c} \frac{\d \eta}{\d \kappa}(x) \d H'_2(x) - \int\limits_{D} \frac{1}{c_1} \lrp{\d  H'_2 - \d H_2 }(x)\\
    &\ge \log c_2\lrp{1 - H_1(A)} - \frac{1}{c_1} \lrp{1- H_2(D) }\\
    & = 0,
\end{align*}
where the last equality follows from the facts that $H'_1$ and $H'_2$ have supports equal to $A$ and $D$, respectively, and integrate to $1$. 

\begin{remark}\label{rem:implicit_corrpair}\emph{ Observe that \eq~\eqref{eq:contamination1} and~\eqref{eq:contamination2} implicitly define the probability measures $H_1$ and $H_2$. To see this, we first argue that $H_1$ is non-negative. Recall that  }
\[\d(\eta\corby H_1)  = (1-\varepsilon)\d \eta + \varepsilon \d H_1. \]
\emph{From~\eqref{eq:contamination1} it is clear that $\d H_1 (x) = 0$ for $\frac{\d \eta}{\d \kappa} \ge c_1$. Thus, $H_1$ is supported only on the complement set, where it is non-negative by choice of $c_1$. Next, since $\d (\eta\corby H_1)$ integrates to $1$ (by choice of $c_1$) and so does $\eta$, $H_1$ too integrates to $1$. One can similarly argue that $H_2$ defined by \eq~\eqref{eq:contamination2} is a probability measure.}
\end{remark}

\section{Discussion on the knowledge of  $\varepsilon$ and an impossibility result}\label{app:knowledgeofeps}
We note that in the current work, we do not require the knowledge of the precise value of corruption probability $\varepsilon$. Instead, the knowledge of an upper bound on $\varepsilon$ is sufficient. In this section, we will show that without the knowledge of $\varepsilon$, no uniformly-good algorithm can achieve logarithmic regret (Remark~\ref{rem:know_eps}).

We remind that this is also a standard assumption in the robust statistics literature. On a related note, \citet[Remark 3]{wang2023huber} mention as an open problem whether it is possible to construct confidence intervals using corrupted samples without the knowledge of $\varepsilon$ (or a bound on it). In this section, we provide a negative answer to this question as well.

Our approach is to relate the problem of constructing anytime-valid confidence intervals in presence of corruption to a specific sequential hypothesis testing problem using corrupted data, which can be formulated as the best-arm identification (BAI) problem with corruptions in the multi-armed bandit framework. We then use the machinery from the proof of Theorem~\ref{lem:lower_bound} to arrive at a lower bound for this problem in terms of $\eKinf$, defined in \eq~\eqref{eq:def_klinf}. We show that if a bound on $\varepsilon$ is not known, then $\eKinf =0$, rendering the BAI problem un-learnable, hence, the impossibility for existence of a sequential test, and hence, impossibility for existence of non-trivial anytime confidence intervals. The approach for proving impossibility result for BAI is reminiscent of a similar negative result for classical BAI in bandits with heavy-tailed distributions (uncorrupted setting), proven in ~\cite[Theorem 3]{agrawal2020optimal}. 

%However, we first look at the anytime-valid confidence intervals from the sequential-hypothesis testing lens and use the lower bounds for that setting to arrive at the negative result. We detail this next. 

For $\delta > 0$ and a collection $\cal L$ of probability measures, we only prove the negative result for constructing an any-time valid upper bound that holds with probability at least $1-\delta$, using samples generated from a corruption neighbourhood (Definition~\ref{def:corr_nbd}) of a distribution $\mu_1\in\cal L$. Symmetric arguments (including a symmetric sequential test) give a corresponding negative result for the lower bound that holds with probability at least $1-\delta$. We now introduce the specific sequential test for this. 

\begin{paragraph}{\bf Sequential setting (hypothesis testing) of interest.} Consider the problem of testing whether the mean of a distribution $\mu_1\in\cal L$ is below a given threshold $\zeta$ in a $\delta$-correct framework. To be more specific, let $m(\mu_1)$ denote its mean, and let $m(\mu_1) < \zeta$ (unknown to the algorithm). The algorithm can generate samples from $\mu_1$. However, on doing so, it  observes the true sample with probability $1-\varepsilon$, and receives a sample from an arbitrary corruption distribution with the remaining $\varepsilon$ probability, i.e., it observes samples from an $\varepsilon$ corruption neighbourhood of $\mu_1$. In presence of $\varepsilon$ corruption, the goal of the algorithm is to generate finite samples (possibly random number of samples, depending on observations made), and declare that $m(\mu_1) < \zeta$ with probability at least $1-\delta$. While ensuring this $\delta$-correctness property, the algorithm's goal is also to minimise its expected stopping time. Let us denote the $\delta$-correct algorithm for this problem by $\mathcal A(\delta,\zeta)$.
\end{paragraph}

\begin{paragraph}{\bf Equivalence of $\delta$-correct anytime-valid upper bound and sequential test described above.} Observe that given a $\delta$-correct algorithm for the above described problem, $\mathcal A(\delta,\zeta)$, one can construct an anytime-valid upper bound on the true mean for $\mu_1$ that holds with probability at least $1-\delta$ using $\varepsilon$-corrupted samples as follows: at any time $n$, define the set \[U_n :=  \lrset{\zeta :  \mathcal A(\delta, \zeta) \text{ has not stopped in $n$ samples} }.\]
Then, $U_n$ is an anytime-valid  upper bound on $m(\mu_1)$ that holds with probability at least $1-\delta$. The reverse implication also holds, i.e., given any sequence of $\delta$-correct anytime-valid upper bound on $m(\mu_1)$ constructed using $\varepsilon$-corrupted samples, say $\bar{U}_n$, one can design a $\delta$-correct algorithm for the above described problem, as described next. Consider the algorithm that stops and declares $m(\mu_1) < \zeta$ at time $n$ if $\bar{U}_n \le  \zeta$. Else, it generates a sample, computes $U_{n+1}$, and proceeds. \\
\end{paragraph}

With the above equivalence at hand, it suffices to show that any $\delta$-correct algorithm for identifying $m(\mu_1) \le \zeta$ would require an unbounded number of samples if $\varepsilon$ is not known. For this, following arguments similar to those in the proof of Theorem~\ref{lem:lower_bound}, one can show the following lower bound on the expected number of samples $\E\lrs{N}$ any $\delta$-correct algorithm would require to generate, when the corruption proportion $\varepsilon$ is known: 
\begin{equation} \label{eq:lb_bai}
    \E\lrs{N} \ge \frac{\log\frac{1}{\delta}}{\eKinf(\mu_1,\zeta; \mathcal{L})}, \end{equation}
where, recall that
\begin{equation} \label{eq:def_klinf_app}
\eKinf(\mu_1, \zeta; \mathcal{L}) \!\defn\! \min\limits_{H,H',\nu_1}\!\lrset{\KL\!\left( \mu_1 \corby \!H, \nu_1 \corby \!H' \right)\!: ~\nu_1\!\in\!\mathcal{L},  H, H' \!\in\! \mathcal{P}(\R),  m(\nu_1)\!\ge\! \zeta}. 
\end{equation}

When $\varepsilon$ is not known to the algorithm, a further optimisation over $\varepsilon < \frac{1}{2}$ would feature in the lower bound. Otherwise, for every fixed $\varepsilon < \frac{1}{2}$ for which the algorithm is $\delta$-correct, there exists an $\varepsilon' > 0$ such that $ \tilde{\varepsilon} := \varepsilon + \varepsilon' < \frac{1}{2}$, and the algorithm wouldn't be $\delta$-correct for some distribution in $\cal L$ with corruption proportions being $\tilde{\varepsilon}$. This follows  from the corresponding lower bound in \eq~\eqref{eq:lb_bai} for $\tilde{\varepsilon}$ instead of $\varepsilon$, and monotonicity of $\eKinf$ in $\varepsilon$ (larger $\varepsilon$ implies a smaller $\eKinf$, and a higher lower bound). Thus, when $\varepsilon$ is unknown, $\eKinf$ with $\varepsilon = \frac{1}{2}$ would feature in the lower bound in \eq~\eqref{eq:lb_bai}.

\begin{paragraph}{\bf Unknown $\varepsilon$ implies $\eKinf = 0$.} As discussed above, if $\varepsilon$ is unknown, $\eKinf$ features with $\varepsilon = \frac{1}{2}$ in the lower bound. Now, consider any distribution $\kappa\in\cal L$ such that $m(\kappa) \ge \zeta$. In the definition of $\eKinf(\mu_1,\zeta;\cal L)$ in \eq~\eqref{eq:def_klinf_app}, $\nu_1 = \kappa$, $H=\kappa$, $H' = \eta$ are feasible solutions for $\varepsilon = \frac{1}{2}$, and satisfy $\mu_1\corby\kappa = \kappa\corby\mu_1$, implying that $\eKinf = 0$. Hence, the lower bound in \eq~\eqref{eq:lb_bai} is unbounded, implying non-existence of a $\delta$-correct algorithm, hence $\delta$-correct anytime-valid upper bound. 
\end{paragraph}

\begin{remark}\label{rem:know_eps}
\emph{
Observe that from the above discussion, we have that for $\varepsilon = \frac{1}{2}$, $\eKinf = 0$. This also implies that the lower bound in Theorem~\ref{lem:lower_bound} in unbounded for our regret-minimisation setting in presence of corruption. Thus, we also conclude that a logarithmic regret is not possible without a knowledge of a bound on $\varepsilon$ that is strictly smaller than $\frac{1}{2}$.}
\end{remark}

\section{Non-intersection of corruption neighbourhoods: Discussions and Proofs from Section~\ref{sec:GaussianLB}}\label{sec:non_inter}

Let $\mathcal{T}$ be the set of all functions $A$ of probability distributions, $A:\mathcal{P}(\R)\to \R$, such that 
$A$ is translation equivariant, meaning that for any $\tilde{\delta} > 0$, if $X\sim\kappa$, and $X+\tilde{\delta} \sim \eta$, then $A(\eta)=A(\kappa)+\tilde{\delta}$. The class $\cal T$ of translation equivariant functions is interesting because common functions like mean, median, and quantiles belong to this class. %Thus, to prove lower bounds for mean-estimation, 

It is well known that in the presence of corruption with  $\varepsilon$ probability, consistent estimation of any translation equivariant function, including mean, is not possible even with infinite samples~\citep{chen2018robust}. The presence of corruption introduces a bias that is unavoidable, that we refer to as \emph{bias due to corruption}. We recall this in what follows.

\begin{definition}[Corruption neighbourhood.]\label{def:corr_nbd}\emph{ For a fixed corruption proportion $\varepsilon$ and a fixed distribution $\kappa$, its corruption neighbourhood is defined as the collection of all distributions in the following set (denoted as $\kappa_\varepsilon$):}
\[ \kappa_\varepsilon := \lrset{ (1-\varepsilon)\kappa + \varepsilon H : ~ H\in\mathcal P(\R) }.  \] 
\end{definition}
Observe that the mean of distributions in the corruption neighbourhood of any distribution $\kappa$ can be arbitrarily large or small. \\

%\noindent{\bf Bias due to corruption.} Now, consider the problem of estimating the value of functions of a probability distribution, say $\kappa$, using independent samples that are generated from distributions in corruption neighbourhood $\kappa_\varepsilon$ defined above. It is well known that for certain functions (for example, mean), when estimating using corrupted samples, no sequence of estimators can converge to the corresponding true-value for original distribution $\kappa$, as the number of samples increase \citep{chen2018robust}. This is in contrast to the classical setup without corruption ($\varepsilon=0$) where one observes samples from the true distribution $\kappa$, instead from distributions in $\kappa_\varepsilon$. In the classical setup, empirical mean, for example, is a well-known consistent estimator that converges to the true mean as the number of samples increase.

The bias due to using the family of functions $\mathcal{T}$ defined above, for distribution $\kappa$ in presence of corruption with probability $\varepsilon$, is defined as
\begin{equation}
\label{eq:defn_bias}
b_\kappa(\varepsilon)\defn\inf_{A\in \mathcal{T}} \sup_{\kappa' \in \kappa_\varepsilon} |A(\kappa')-A(\kappa)|.
\end{equation}
The definition above quantifies the mini-max bias suffered by any translation-equivariant function of $\kappa$, in presence of $\varepsilon$ corruption. The lemma below, taken from \cite[Section 4.2]{robuststat}, shows that when $\kappa$ is symmetric and unimodal, then the function that suffers the minimum bias is median. We refer the reader to \cite{robuststat} for a proof of the lemma.

\begin{Lemma}[Optimality condition for median]\label{lem:bias}\label{lem:optimality_median}
Let $\kappa$ be a symmetric and unimodal distribution. The functional that achieves the infimum in $b_\kappa(\varepsilon)$ is median. Moreover, 
$$b_\kappa(\varepsilon) = F_\kappa^{-1}\left(\frac{1}{2(1-\varepsilon)} \right),$$
where $F_\kappa$ is the c.d.f. of $\kappa$.
\end{Lemma}

We now prove the equivalent conditions for $\ekl_\cG$ to be non-zero.
\subsection{Proof of Lemma \ref{lem:non-intersection}}\label{app:pf_non-intersection}

Without loss of generality, we assume that $m(\kappa)\le m(\eta)$. Define
\[ b_0(\varepsilon)\defn \Phi^{-1}\left( \frac{1}{2(1-\varepsilon)}\right).\]

We first prove that if for all $(H_1, H_2) \in \mathcal{P}(\R)^2$ such that $\kappa \corby H_1 \neq \eta \corby H_2$, then \[|m(\kappa)-m(\eta)|> 2 b_0(\epsilon).\]

For this, we show the contrapositive of the above. To this end, let us assume that $|m(\kappa)-m(\eta)|\le 2 b_0(\varepsilon)$. Then 
$$\exists~ \varepsilon' \le \varepsilon ~~ \text{ such that } |m(\kappa)-m(\eta)|=2 b_0(\varepsilon').$$ 

We construct a probability measure that belongs to the intersection of the corruption neighbourhoods of $\eta$ and $\kappa$. Define
$$
p'(x) \defn \begin{cases}
(1-\varepsilon')\varphi\left(x-m(\kappa)\right), & \text{for }(x-m(\kappa))\le b_0(\varepsilon')\\
(1-\varepsilon')\varphi\left(x-m(\kappa)-2 b_0(\varepsilon')\right), & \text{for }(x-m(\kappa))>b_0(\varepsilon') ,
\end{cases}.
$$
We first show that $p' \in \kappa_\varepsilon$, i.e., it belongs to the corruption neighbourhood of $\kappa$. To this end, consider $(p' -(1-\varepsilon')\kappa)(x)$, which equals
 \begin{align*}
    \begin{cases}
        0, & \text{for }(x-m(\kappa))\le  b_0(\varepsilon') \\
        (1-\varepsilon')\left(\varphi\left(x-m(\kappa)-2b_0(\varepsilon')\right)-\varphi\left( x-m(\kappa)\right)\right), & \text{for }(x-m(\kappa))>  b_0(\varepsilon').
\end{cases}
 \end{align*}

Now, if $x-m(\kappa) \in \left(b_0(\varepsilon'), 2b_0(\varepsilon')\right]$, then $$0\le 2b_0(\varepsilon')-(x-m(\kappa))\le b_0(\varepsilon') \le x-m(\kappa).$$  Hence $p'(x)\ge 0$. 

 Similarly, if $x-m(\kappa) \ge  2b_0(\varepsilon')$, then
$$0\le x-m(\kappa)-2b_0(\varepsilon')\le x-m(\kappa), $$ giving $p'(x)\ge 0.$

Additionally, 
\begin{align*}
&\int_\R (p' -(1-\varepsilon')\kappa)(x) \d x\\ 
&=(1-\varepsilon')\int \left(\varphi\left(x-m(\kappa)-2 b_0(\varepsilon')\right)-\varphi\left( x-m(\kappa)\right)\right) \1\left\{x-m(\kappa)>  b_0(\varepsilon')\right\}\\
&= (1-\varepsilon')\left(1-\Phi\left(-b_0(\varepsilon')\right)-\left(1-\Phi\left(b_0(\varepsilon')\right)\right)\right)\\
&=(1-\varepsilon')\left( \frac{1}{2(1-\varepsilon')}-1+\frac{1}{2(1-\varepsilon')}\right)=\varepsilon'.
\end{align*}
Hence, $p'-(1- \varepsilon')\kappa$ is also a non-negative measure that sums to $\varepsilon'$. This implies that $p' \in \kappa_{\varepsilon'}\subset \kappa_\varepsilon$.\\

We next show that $p'$ belongs to the corruption neighbourhood of $\eta$.\\

Since  
\[|m(\kappa)-m(\eta)|=2 b_0(\varepsilon')=m(\eta)-m(\kappa),\] 
we have
\begin{align*}
p' &= \begin{cases}
(1-\varepsilon')\varphi\left(x-m(\kappa)\right) & \text{for }x-m(\kappa)\le b_0(\varepsilon')\\
(1-\varepsilon')\varphi\left(x-m(\kappa)-2 b_0(\varepsilon')\right) & \text{for }x-m(\kappa)>b_0(\varepsilon')
\end{cases}\\
&=\begin{cases}
(1-\varepsilon')\varphi\left(x-m(\eta)+2b_0(\varepsilon')\right) & \text{for }x-m(\eta)\le- b_0(\varepsilon')\\
(1-\varepsilon')\varphi\left(x-m(\eta)\right) & \text{for }x-m(\eta)>-b_0(\varepsilon') .
\end{cases}
\end{align*}
Then, $(p' - (1-\varepsilon')\eta)(x)$ equals
$$
\begin{cases}
(1-\varepsilon')\left( \varphi\left(x-m(\eta)+2b_0(\varepsilon')\right)-\varphi\left(x-m(\eta)\right)\right) & \text{for }x-m(\eta)\le- b_0(\varepsilon')\\
0 & \text{for }x-m(\eta)>-b_0(\varepsilon') .
\end{cases}
$$

Again, if $x-m(\eta) \in \left(b_0(\varepsilon'), 2b_0(\varepsilon')\right]$, then $$0\le 2b_0(\varepsilon')-(x-m(\eta))\le b_0(\varepsilon') \le x-m(\eta),$$ implying that $p'(x)-(1-\varepsilon')\eta(x)\ge 0$.

On the other hand, if $x-m(\eta) \ge  2b_0(\varepsilon')$, then
$$0\le x-m(\eta)-2b_0(\varepsilon')\le x-m(\eta), $$ 
and then $p'(x)-(1-\varepsilon')\eta(x)\ge 0.$ 

Additionally,  $p'-(1-\varepsilon')\eta$ sums to $\varepsilon'$, as shown below:
\begin{align*}
&\int_\R (p'-(1-\varepsilon')\eta)(x)\\
&= \int (1-\varepsilon')\left( \varphi\left(x-m(\eta)+2b_0(\varepsilon')\right)-\varphi\left(x-m(\eta)\right)\right)\1\left\{x-m(\eta)\le- b_0(\varepsilon')\right\}\\
&=(1-\varepsilon')\left(\Phi(b_0(\varepsilon'))-\Phi(-b_0(\varepsilon')) \right)\\
&=(1-\varepsilon')\left(\frac{1}{2(1-\varepsilon')}-1+\frac{1}{2(1-\varepsilon')}\right)=\varepsilon'.
\end{align*}
Thus, $p'$ also belongs to the corruption neighbourhood of $\eta$, i.e., $p'\in \eta_{\varepsilon'}\subset \eta_\varepsilon$, proving one direction.\\

We now prove that if $|m(\kappa)-m(\eta)|>2 b_0(\varepsilon)$, then $\kappa_\varepsilon \cap \eta_\varepsilon = \emptyset$, again by proving the contrapositive. \\ 

Suppose $\exists \kappa' \in \kappa_\varepsilon \cap \eta_\varepsilon$. Since  median is a minimax-bias functional, and $\kappa' \in \kappa_\varepsilon$, 
$$|\Med(\kappa')-\Med(\kappa)|\le  b_0(\varepsilon). $$
Similarly, having $\kappa' \in \eta_\varepsilon$,
$$|\Med(\kappa')-\Med(\eta)|\le  b_0(\varepsilon). $$
Hence, we obtain that
$$|m(\kappa)-m(\eta)|=|\Med(\kappa)-\Med(\eta)|\le 2 b_0(\varepsilon), $$
proving the other direction.

\section{Properties of $\ekl_\cG$: a discussion and proofs}\label{app:prop_ekl}
$\ekl_\cG$ is crucial for our algorithm, both practically, and theoretically. We characterise its solutions in  Lemma~\ref{lem:pair_distrib}, and we prove various nice properties that are useful in algorithmic implementation, as well as for its analysis. In this appendix,  we discuss these  properties of $\ekl_\cG$, including those presented in the main text in Lemma~\ref{lem:basics_ekl}. 

Recall that for $x\in \R$, $y\in\R$, 
$$\ekl_\cG(x,y)\defn \inf\limits_{H,H'}\lrset{\KL(\G(x,1)\corby H, \G(y,1)\corby H'): ~ H\in\mathcal{P}(\R), ~ H' \in \mathcal P(\R)}.$$
Further, recall that Lemma~\ref{lem:pair_distrib} characterizes the optimal $H$ and $H'$ for this problem, and are defined by \eq~\eqref{eq:contamination1} and \eq~\eqref{eq:contamination2}. Later, in Lemma~\ref{lem:supports_cor}, we identify the support sets for the optimal corruption pair $(H_1, H_2)$ in the specific setting of $\eta = \G(x,1)$ and $\kappa = \G(y,1)$. We now prove Lemma~\ref{lem:basics_ekl}, before going on to developing additional properties, which will be handy in the analysis later. In particular, we show \eq~\eqref{eq:Kg} below which gives a closed-form expression for $\ekl_\cG(x,y)$, for $x < y$, once we know the optimal $c$ from Lemma~\ref{lem:basics_ekl}\ref{norm_const}, which we compute using a root-finding algorithm. For any $x<y$, let $\Delta = y-x$, $c$, $\Delta_+$ and $\Delta_-$ be as in Lemma~\ref{lem:basics_ekl}\ref{norm_const}. Then, 
\begin{align}
\hspace{-3mm}
\frac{ \ekl_\cG(x,y)}{1-\varepsilon}\!=\!
(1-c)\Phi\!\left(\!\frac{\Delta_-}{2}\!\right)\!\log\frac{1}{c}  +  \frac{\Delta^2}{2}\!\left(\!\Phi\!\left(\!\frac{\Delta_+}{2}\!\right)\!-\!\Phi\!\left(\!\frac{\Delta_-}{2}\!\right)\!\right) - \Delta\!\left(\!\varphi\!\left(\!\frac{\Delta_-}{2}\!\right)\!-\! \varphi\!\left(\!\frac{\Delta_+}{2}\!\right)\!\right)\!.\label{eq:Kg}
\end{align}

\subsection{Proof of Lemma~\ref{lem:basics_ekl}}\label{app:proof_basics_ekl}
Let $\eta$ represent the cdf of $\G(x,1)$ and $\kappa$ be that for $\G(y,1)$. Recall that $c_1$ and $c_2$ are the normalisation constants for the optimal corruption distributions in $\ekl_\cG(x,y)$ (Lemma~\ref{lem:pair_distrib}). 
%Define 
%$$A_{c_1} \defn \lrset{z\in\R~ :~  \frac{\d \eta}{\d \kappa}(z) \le c_1} \quad \text{ and } \quad D_{c_2} \defn \lrset{z\in \R ~ : ~ \frac{\d \eta}{\d \kappa}(z) \ge \frac{1}{c_2}}.$$ 

\noindent\textit{Proof of Lemma~\ref{lem:basics_ekl}\ref{norm_const}: } Since  $\d(\G(x,1) \corby H_1)$ is a probability distribution,  it sums to $1$. Let $W\sim \G(0,1)$ be a random variable distributed according to standard Gaussian. Using the explicit form of $A_{c_1}$ and $D_{c_1}$ from Lemma~\ref{lem:supports_cor} (presented later), we have 

\begin{align*}
1&=\int_\R \d (\eta \corby H_1)\\
&= (1-\varepsilon)\left(c_1\kappa(A_{c_1})+\eta(\R \setminus A_{c_1}) \right)\\
&= (1-\varepsilon)\left( c_1 \P\left(W+y \ge\frac{x+y}{2}+\frac{\log(\frac{1}{c_1})}{y-x}  \right)+\P\left(W+x <\frac{x+y}{2}+\frac{\log(\frac{1}{c_1})}{y-x}  \right) \right)\\
&= (1-\varepsilon)\left(c_1\left(1-\Phi\left(\frac{x+y}{2}+\frac{\log(\frac{1}{c_1})}{y-x} -y\right)\right)+\Phi\left(\frac{x+y}{2}+\frac{\log(\frac{1}{c_1})}{y-x}-x\right) \right)\\
&= (1-\varepsilon)\left(c_1\left(1-\Phi\left(-\frac{\Delta}{2}+\frac{\log(\frac{1}{c_1})}{\Delta}\right)\right)+\Phi\left(\frac{\Delta}{2}+\frac{\log(\frac{1}{c_1})}{\Delta}\right) \right).
\end{align*}
Similarly, $\d (\G(y,1) \corby H_2)$ is a probability distribution, it sums to $1$, giving 
\begin{align*}
1&=\int \d (\G(y,1) \corby H_2)\\
&= (1-\varepsilon)\left(c_2\eta(D_{c_2})+\kappa(\R \setminus D_{c_2}) \right)\\
&= (1-\varepsilon)\left( c_2 \P\left(W+x\le\frac{x+y}{2}-\frac{\log(\frac{1}{c_2})}{y-x}  \right)+\P\left(W+y >\frac{x+y}{2}-\frac{\log(\frac{1}{c_2})}{y-x}  \right) \right)\\
&= (1-\varepsilon)\left(c_2\Phi\left(\frac{x+y}{2}-\frac{\log(\frac{1}{c_2})}{y-x} -x\right)+1-\Phi\left(\frac{x+y}{2}-\frac{\log(\frac{1}{c_2})}{y-x}-y\right) \right)\\
&= (1-\varepsilon)\left(c_2\Phi\left(\frac{\Delta}{2}-\frac{\log(\frac{1}{c_2})}{\Delta}\right)+1-\Phi\left(-\frac{\Delta}{2}-\frac{\log(\frac{1}{c_2})}{\Delta}\right) \right) \\
&= (1-\varepsilon)\left(c_2\left(1-\Phi\left(-\frac{\Delta}{2}+\frac{\log(\frac{1}{c_2})}{\Delta}\right)\right)+\Phi\left(\frac{\Delta}{2}+\frac{\log(\frac{1}{c_2})}{\Delta}\right) \right).
\end{align*}
From the above, observe that $c_1$ and $c_2$ solve the same equation. Hence, they can be taken to be equal to a common value, say $c>0$.

We now prove the uniqueness of this common value $c$. From the discussion in the previous paragraph, $c$  solves the following equation:
\begin{equation}\label{eq:def_c2}
\frac{1}{1-\varepsilon}=c\Phi\left(\frac{\Delta_-}{2}\right)+\Phi\left(\frac{\Delta_+}{2}\right).
\end{equation}
Observe that $c$ is uniquely defined by Equation~\eqref{eq:def_c2}, indeed $c  \mapsto c\Phi\left(\frac{\Delta_-}{2}\right)+\Phi\left(\frac{\Delta_+}{2}\right) $ is increasing because its derivative is
$$\Phi\left(\frac{\Delta_-}{2}\right) + \frac{1}{\Delta}\varphi\left(\frac{\Delta_-}{2}\right)-\frac{1}{c\Delta}\varphi\left(\frac{\Delta_+}{2}\right)=\Phi\left(\frac{\Delta_-}{2}\right)>0. \quad \Box$$

\noindent\textit{Proof for Lemma~\ref{lem:basics_ekl}\ref{shift_inv}:} From Lemma~\ref{lem:pair_distrib} and the using part~\ref{norm_const} above in the  definition of $\ekl_\mathcal{G}$, we have for any $x<y$, 
\begin{align}\label{eq:trunc_kl1}
\frac{\ekl_\mathcal{G} (x,y)}{1-\varepsilon} =&   \int\limits_{A_c} \! c \varphi(t-y)\log \left(c\right)    + \int\limits_{D_c} \! \varphi(t-x)\log\frac{1}{c}  \! + \! \int\limits_{\R \setminus A_c\cup D_c} \varphi(t-x)\log \left( \frac{\varphi(t-x)}{\varphi(t-y)}\right) \d t .
\end{align}
On simplifying, it then equals $1-\varepsilon$ times 
\begin{align*}
    c\log \left(c\right)\Phi\left(\frac{\Delta_-}{2}\right)   + \log(1/c)\Phi\left(\frac{\Delta_-}{2}\right) + \int \limits_{\R \setminus A_c\cup D_c} \varphi(t-x)\log \left( \frac{\varphi(t-x)}{\varphi(t-y)}\right)\d t .
\end{align*}
We now compute the integral on $\R \setminus A_c\cup D_c$. For this, let $a<b$. Then clearly,
\begin{align*}
\int_a^b\! \varphi(t-x)\log \!\left(\! \frac{\varphi(t-x)}{\varphi(t-y)}\!\right)&\!=\!\frac{1}{\sqrt{2\pi}}\int_a^b e^{-\frac{(t-x)^2}{2}}\left(-\frac{(t-x)^2}{2}+\frac{(t-y)^2}{2} \right)\d t\\
&\!=\! (x-y)\left(\frac{1}{\sqrt{2\pi}}\int_a^b te^{-\frac{(t-x)^2}{2}}\d t-\frac{x+y}{2}(\Phi(b-x)-\Phi(a-x))   \right).
\end{align*}
Using the mean of a truncated-Gaussian random variable, we get 
\begin{align*}
&\int_a^b \! \varphi(t-x)\log \left( \frac{\varphi(t-x)}{\varphi(t-y)}\right)\\
&=\! (x-y)\left(\!x (\Phi(b-x)\!-\!\Phi(a-x))\!+\!\varphi(a-x)-\varphi(b-x)\!-\!\frac{x+y}{2}(\Phi(b-x)-\Phi(a-x))  \! \right)\\
&=\! \frac{(x-y)^2}{2}(\Phi(b-x)-\Phi(a-x))+(x-y)\left(\varphi(a-x)-\varphi(b-x) \right).\vspace{-0.5em}
\end{align*}
Now, substituting $\Delta=y-x$, $a=x+\frac{\Delta_-}{2}$ and $b=x+\frac{\Delta_+}{2}$, we have that the above integral equals \vspace{-0.5em}
$$\frac{\Delta^2}{2}\left(\Phi\left(\frac{\Delta_+}{2}\right)-\Phi\left(\frac{\Delta_-}{2}\right)\right)-\Delta\left(\varphi\left(\frac{\Delta_-}{2}\right)-\varphi\left(\frac{\Delta_+}{2}\right) \right).$$ 
Substituting this in \eq~\eqref{eq:trunc_kl1} we have that $\ekl_\mathcal{G} (x,y)$ equals $(1-\varepsilon)$ times
\begin{align}
c\log \left(c\right)\Phi\left(\frac{\Delta_-}{2}\right)  & + \log(1/c)\Phi\left(\frac{\Delta_-}{2}\right) \nonumber \\  
&+ \frac{\Delta^2}{2}\left(\Phi\left(\frac{\Delta_+}{2}\right)-\Phi\left(\frac{\Delta_-}{2}\right)\right) -\Delta\left(\varphi\left(\frac{\Delta_-}{2}\right)-\varphi\left(\frac{\Delta_+}{2}\right) \right). \label{eq:def_kleG}
\end{align}
Shift invariance now follows from the above expression for $\ekl_\cG$ only in terms of $\Delta$. $\quad\Box$\\

\noindent\textit{Proof for Lemma~\ref{lem:basics_ekl}\ref{differentiable}:}
Recall the defining equation for $c$ from \eq~\eqref{eq:def_c2}. Observe that $c$ is a function of $\Delta$. Then, by implicit function theorem, $c$ is differentiable. Let $c'$ denote the derivative of $c$ with respect to $\Delta$. Then, using the expressions for derivatives from Lemma~\ref{lem:helper_derivatives},
\begin{align*}
 \frac{\partial}{\partial \Delta}\varphi\lrp{\frac{\Delta_+}{2}}   = \varphi\lrp{\frac{\Delta_+}{2}}\lrp{ - \frac{\Delta_+\Delta_-}{4\Delta} - \frac{\Delta_+ \varphi(\Delta_-/2) }{2\Delta \Phi(\Delta_-/2)}  },
\end{align*}
and similarly,
\[ \frac{\partial}{\partial \Delta}\varphi\lrp{\frac{\Delta_-}{2}}  =  \varphi\lrp{\frac{\Delta_-}{2}}\lrp{ - \frac{\Delta_+\Delta_-}{4\Delta} + \frac{\Delta_- \varphi(\Delta_-/2) }{2\Delta \Phi(\Delta_-/2)} }.  \]
Since $\Delta \mapsto c$ is differentiable with continuous derivative on $(0,\infty)$, from Equation~\eqref{eq:def_kleG}, $\Delta\mapsto \ekl_\mathcal{G} (x,x+\Delta)$ is also differentiable with continuous derivative on  $(0,\infty)$.

Differentiating \eq~\eqref{eq:def_kleG} with respect to $\Delta$ (after setting $y=x+\Delta$), and substituting for $c'$  from Lemma~\ref{lem:helper_derivatives}, we have that 

\iffalse
\begin{align*}
\frac{1}{(1-\varepsilon)}&\frac{\partial \ekl_\mathcal{G} (x,x+\Delta)}{ \partial \Delta}\\
=& c' \log c \Phi\lrp{\frac{\Delta_-}{2}} + c' \Phi\lrp{\frac{\Delta_-}{2}} + \frac{c\log c}{2} \varphi\lrp{\frac{\Delta_-}{2}} \frac{\Delta_+}{\Delta} +\log c \varphi\lrp{\frac{\Delta_-}{2}} \frac{c'}{\Delta}\\
 &- \frac{c'}{c}\Phi\lrp{\frac{\Delta_-}{2}} - \frac{\log c}{2} \varphi\lrp{\frac{\Delta_-}{2}} \frac{\Delta_+}{\Delta} - \frac{c'}{\Delta c } \log c \varphi\lrp{\frac{\Delta_-}{2}}\\
&+ \Delta \Phi\lrp{\frac{\Delta_+}{2}} + \frac{\Delta^2}{4}\varphi\lrp{\frac{\Delta_+}{2}} \frac{\Delta_-}{\Delta} - \frac{c'\Delta}{2c} \varphi\lrp{\frac{\Delta_+}{2}}\\
&-\Delta \Phi\lrp{\frac{\Delta_-}{2}} - \frac{\Delta^2}{4}\varphi\lrp{\frac{\Delta_-}{2}}\frac{\Delta_+}{\Delta} - \frac{c'\Delta}{2c} \varphi\lrp{\frac{\Delta_-}{2}}\\
&-\varphi\lrp{\frac{\Delta_-}{2}} + \frac{\Delta_- \Delta_+}{4} \varphi\lrp{\frac{\Delta_-}{2}} -  \Delta_- \frac{\varphi^2(\Delta_-/2)}{2\Phi(\Delta_-/2)}\\
&+\varphi\lrp{\frac{\Delta_+}{2}} - \frac{\Delta_-\Delta_+}{4} \varphi\lrp{\frac{\Delta_+}{2}} - \Delta_+ \frac{\varphi(\Delta_+/2) \varphi(\Delta_-/2)}{2\Phi(\Delta_-/2)}.
\end{align*}
\fi
%On substituting for $c'$ and simplifying, the above equals
\begin{align*}
\frac{1}{(1-\varepsilon)}\frac{\partial \ekl_\mathcal{G} (x,x+\Delta)}{ \partial \Delta}&=-c\log c \phi\lrp{\frac{\Delta_-}{2}}  + \frac{c\log c}{2} \varphi\lrp{\frac{\Delta_-}{2}} \frac{\Delta_+}{\Delta} - \frac{c\log c}{\Delta} \frac{\varphi^2(\Delta_-/2)}{\Phi(\Delta_-/2)}\\
& \qquad -\log c \frac{\Delta_+}{2\Delta} \varphi\lrp{\frac{\Delta_-}{2}} + \frac{\log c}{\Delta} \frac{\varphi^2(\Delta_-/2)}{\Phi(\Delta_-/2)}- \frac{\Delta_+ \varphi(\Delta_+/2) \varphi(\Delta_-/2)}{2\Phi(\Delta_-/2)}\\
& \quad +\Delta\lrp{ \Phi\lrp{\frac{\Delta_+}{2}} \!-\! \Phi\lrp{\frac{\Delta_-}{2}} } \!+\! \frac{\Delta\Delta_-}{4}\varphi\lrp{\frac{\Delta_+}{2}} - \frac{\Delta \Delta_+}{4} \varphi\lrp{\frac{\Delta_-}{2}}\\
&\qquad  +\frac{\Delta}{2} \frac{\varphi(\Delta_-/2) \varphi(\Delta_+/2)}{\Phi(\Delta_-/2)} + \frac{\Delta}{2}\frac{\varphi^2(\Delta_-/2)}{\Phi(\Delta_-/2)} +\frac{\Delta_- \Delta_+}{4} \varphi\lrp{\frac{\Delta_-}{2}} \\
&\qquad - \frac{\Delta_+\Delta_-}{4} \varphi\lrp{\frac{\Delta_+}{2}} - \frac{\Delta_- \varphi^2(\Delta_-/2)}{2\Phi(\Delta_-/2)}.
\end{align*}
Next, using 
\[c\varphi(\Delta_-/2) = \varphi(\Delta_+/2),\] 
and collecting the coefficients of like-terms, the required derivative scaled by $1-\varepsilon$ equals
\begin{multline*}
\Delta\lrp{\Phi\lrp{\frac{\Delta_+}{2}} \!-\! \Phi\lrp{\frac{\Delta_-}{2}}} \!+\! \frac{\varphi^2(\Delta_+/2)}{\Phi(\Delta_-/2)}\lrp{\frac{1}{c\Delta} \log \frac{1}{c} \!-\! \frac{1}{c^2\Delta}\log\frac{1}{c} + \frac{\Delta}{2c} \!+\! \frac{\Delta}{2c^2} - \frac{\Delta_-}{2c^2} -\frac{\Delta_+}{2c} }\\
+\varphi\lrp{\frac{\Delta_+}{2}}\lrp{ \log\frac{1}{c} - \frac{\Delta_+}{2\Delta}\log\frac{1}{c} + \frac{\Delta_+}{2\Delta c} \log\frac{1}{c} + \frac{\Delta\Delta_-}{4} - \frac{\Delta\Delta_+}{4c} + \frac{\Delta_+\Delta_-}{4c} - \frac{\Delta_+\Delta_-}{4} }.
\end{multline*}
Substituting for $\Delta_+$ and $\Delta_-$ in the above expression, one can see that the coefficients of $\varphi(\Delta_+/2)$ and $\frac{\varphi^2(\Delta_+/2)}{\Phi(\Delta_-/2)}$ are $0$, giving

\begin{equation*}
\frac{1}{1-\varepsilon}\frac{\partial \ekl_\mathcal{G} (x,x+\Delta)}{ \partial \Delta} = \Delta\lrp{\Phi(\Delta_+/2) - \Phi(\Delta_- /2)}. 
\end{equation*}

For the inequality, observe that by definition of $c$, we have
$$\Phi \left(\frac{\Delta_-}{2}\right) \ge c\Phi \left(\frac{\Delta_-}{2}\right) = \frac{1}{1-\varepsilon}-\Phi\left( \frac{\Delta_+}{2}\right) .$$
Using this inequality in the derivative $\frac{\partial \ekl_\mathcal{G} (x,x+\Delta)}{ \partial \Delta}$ we  get the result. $\quad\Box$

\subsection{Additional properties of $\ekl_\cG$}\label{app:additional_prop}
In this section, we state various properties of $\ekl_\cG$ derived from the definitions of the optimal pair of corrupted distributions from Lemma~\ref{lem:pair_distrib}.

\begin{Lemma}\label{lem:supports_cor}
Let $y>x + \Delta_{\min}$. Let $H_1$ and $H_2$ be the pair of distributions from Lemma~\ref{lem:pair_distrib} for $\eta = \G(x,1)$ and $\kappa = \G(y,1)$. Then, $\Supp{H_1}=A_{c_1}$ and $\Supp{H_2}=D_{c_2}$, where 
$$A_{c_1}\! =\! \left\{t\in \R:\quad t\ge \frac{y+x}{2}+\frac{\log(1/c_1)}{y-x} \right\} \quad \text{and} \quad  D_{c_2}\!=\!\left\{t\in \R:\quad t\le \frac{x+y}{2}-\frac{\log(1/c_1)}{y-x} \right\}.$$

\end{Lemma}

\begin{proof}
First, by the definitions of $H_1$ and $H_2$, we have that $H_1$ is supported on
\begin{align*}
A_{c_1} = \left\{\frac{\d \G(x,1)}{\d \G(y,1)}(t)\le c_1 \right\}&=\left\{\log\left(\frac{\d \G(x,1)}{\d \G(y,1)}(t)\right)\le -\log\frac{1}{c_1} \right\}\\
&= \left\{ \frac{(t-y)^2}{2}-\frac{(t-x)^2}{2}\le -\log\frac{1}{c_1} \right\}\\
&=\left\{t (x-y)+\frac{y^2-x^2}{2}\le -\log\frac{1}{c_1} \right\}\\
&=\left\{t \ge \frac{x+y}{2}+\frac{\log\frac{1}{c_1}}{y-x} \right\}.
\end{align*}
Similarly, we have the rewriting
\begin{align*}
D_{c_2} = \left\{\frac{\d \G(x,1)}{\d \G(y,1)}(t)\ge \frac{1}{c_2} \right\}&=\left\{\log\left(\frac{\d \G(x,1)}{\d \G(y,1)}(t)\right)\ge \log\frac{1}{c_2} \right\}\\
&=\left\{t (x-y)+\frac{y^2-x^2}{2}\ge \log\frac{1}{c_2} \right\}\\
&=\left\{t \le \frac{x+y}{2}-\frac{\log\frac{1}{c_2}}{y-x} \right\}.
\end{align*}
\end{proof}
\begin{Lemma}\label{lem:GaussianCorruptionSupp}
For $y> x$, define $\Delta=y-x$, $$\Delta_+:=\Delta+2\log\left(\frac{1}{c}\right)\frac{1}{\Delta} \quad \text{and} \quad \Delta_- :=\Delta-2\log\left(\frac{1}{c}\right)\frac{1}{\Delta},$$ 
where $c$ is the normalization constant. $\Supp{H_1}=A_c$ and $\Supp{H_2}=D_c$, where
$$A_c = \left\{x \ge \frac{\Delta_+}{2}+m(\eta) \right\} \quad \text{and} \quad  D_c=\left\{x \le \frac{\Delta_-}{2}+m(\eta) \right\}.$$
\end{Lemma}
\begin{proof}
This follows from Lemma~\ref{lem:GaussianCorruptionSupp} with $c_1 = c_2 = c$.
\end{proof}

\begin{Lemma}\label{lem:helper_derivatives}
We have that $c$ is a continuous function of $\Delta$ with continuous derivative on $(0,\infty)$. Moreover, for any $\Delta>0$, 
\[ c' = \frac{-c \varphi(\Delta_-/2)}{\Phi(\Delta_-/2)}, \quad c\varphi\lrp{\frac{\Delta_-}{2}} = \varphi\lrp{\frac{\Delta_+}{2}}, \quad \frac{\partial \Delta_+}{\partial \Delta} = \frac{\Delta_-}{\Delta} - \frac{2c'}{\Delta c}, \quad \frac{\partial \Delta_-}{\partial \Delta} = \frac{\Delta_+}{\Delta} + \frac{2c'}{\Delta c}. \]
\end{Lemma}
\begin{proof}
$c$ is defined by the following equation:
\begin{equation*}
\frac{1}{1-\varepsilon}=c\Phi\left(\frac{\Delta_-}{2}\right)+\Phi\left(\frac{\Delta_+}{2}\right).
\end{equation*}
Because $\Delta \mapsto  \Delta_+$, $\Delta \mapsto \Delta_-$ and $\Phi$ are all differentiable with continuous derivative on $(0,\infty)$, we have by implicit function theorem, $c$ is a differentiable function of $\Delta$ with continuous derivative, let us denote $c'$ this derivative. We have on the one hand
$$\Delta_+':=\frac{\d }{\d \Delta}\Delta_+ =1-2\frac{c'(\Delta)}{\Delta c(\Delta)}-2\frac{\log(1/c(\Delta))}{\Delta^2} = \frac{\Delta_-}{\Delta}-2\frac{c'(\Delta)}{\Delta c(\Delta)},$$
and no the other hand
$$\Delta_-':=\frac{\d }{\d \Delta}\Delta_- =\frac{\Delta_+}{\Delta}+{2\frac{c'(\Delta)}{\Delta c(\Delta)}}.$$
Then, taking the derivative with respect to $\Delta$ in Equation~\eqref{eq:def_c2},
\begin{align}\label{eq:deriv_c}
0=& c'(\Delta)\Phi\left(\frac{\Delta_-}{2}\right)+c(\Delta)\lrp{\frac{\Delta_-'}{2}}\varphi\left(\frac{\Delta_-}{2}\right)+\lrp{\frac{\Delta_+'}{2}}\varphi\left(\frac{\Delta_+}{2}\right)\nonumber\\
=& c'(\Delta)\left(\Phi\left(\frac{\Delta_-}{2}\right)+\frac{1}{\Delta}\varphi\left(\frac{\Delta_-}{2}\right)- \frac{1}{c(\Delta)\Delta}\varphi\left(\frac{\Delta_+}{2}\right)\right)\nonumber\\
&+c(\Delta)\frac{\Delta_+}{2\Delta}\varphi\left(\frac{\Delta_-}{2}\right)+\frac{\Delta_-}{2\Delta}\varphi\left(\frac{\Delta_+}{2}\right).
\end{align}
Now, observe that $\Delta_+^2 = \Delta_-^2 + 8 \log(1/c(\Delta)),$ hence
\begin{equation}\label{eq:phi+-relation}
\varphi\left(\frac{\Delta_+}{2}\right)=\frac{1}{\sqrt{2\pi}}e^{-\frac{\Delta_+^2}{8}}=\frac{1}{\sqrt{2\pi}}e^{-\frac{\Delta_-^2}{8}+\log(c)}=c(\Delta)\varphi\left(\frac{\Delta_-}{2}\right). 
\end{equation}
Plugging this in Equation~\eqref{eq:deriv_c}, we have
\begin{equation*}
0 = c'(\Delta)\Phi\left(\frac{\Delta_-}{2}\right)+c(\Delta)\varphi\left(\frac{\Delta_-}{2}\right).
\end{equation*}
Hence, we deduce that
\begin{equation}\label{eq:deriv_c2}
c'(\Delta) = -\frac{c(\Delta)\varphi\left(\frac{\Delta_-}{2}\right)}{\Phi\left(\frac{\Delta_-}{2}\right)}.
\end{equation}
\end{proof}

As a direct consequence of the above properties of $\ekl_\cG$ and Taylor's inequality, we also have the following mean-value theorem for $\ekl_\cG$.

\begin{Lemma}[Mean-value theorem for $\ekl_\cG$]\label{lem:taylor_kl}
Suppose that $\mu_a \sim \G(m_a,1)$, $\mu_b \sim \G(m_b,1)$ and $m_* \in \R$ with both $ \Delta_a := m_*-m_a >\Delta_{\min}$ and $\Delta_b :=  m_*-m_b>\Delta_{\min}$. Then, 
$$\ekl_{\cG}(\mu_a,m_* )-\ekl_{\cG}(\mu_b,m_* )\le (1-\varepsilon)(m_b - m_a)_+\left(\Delta_a \vee \Delta_b\right).$$
\end{Lemma}

\begin{proof}
By Lemma~\ref{lem:basics_ekl}, we have $\eKinf(\nu_a, m_*) = \ekl_\cG(m_a, m_*)$ and similarly for $\eKinf(\nu_b,m_*)$. Using this and the shift invariance from Lemma~\ref{lem:basics_ekl}\ref{shift_inv},
\begin{align*}
\eKinf(\nu_a,m_* )-\eKinf(\nu_b,m_* )&= \ekl_\cG(m_a,m_* )-\ekl_\cG(m_b,m_* )\\
&= \ekl_\cG(m_*,2m_*-m_a )-\ekl_\cG(m_*,2m_*-m_b ),
\end{align*}
and then, denoting $\Delta_a = m_*-m_a=$ and $\Delta_b = m_*- m_b$, if $m_a < m_b$ then from Taylor's inequality and Lemma~\ref{lem:deriv_K},
\begin{align*}
&\eKinf(\nu_a,m_* )-\eKinf(\nu_b,m_* )\\
&\le  (m_b - m_a) \sup_{t \in (0,1)}\left|\frac{\partial \ekl_\mathcal{G} (x,x+\Delta)}{ \partial \Delta} \Big|_{\Delta=(1-t)(m_*-m_a)+t(m_*-m_b)} \right|\\
&\le(m_b - m_a)(1-\varepsilon)\sup_{\Delta=(1-t)(m_*-m_a)+t(m_*-m_b), t \in (0,1)} \Delta\left(2\Phi\left( \frac{\Delta_+}{2}\right) - \frac{1}{1-\varepsilon}\right)\\
&\le (1-\varepsilon)(m_b - m_a)\left(\Delta_a \vee \Delta_b\right).
\end{align*}
On the other hand, if $m_a \ge m_b$, then $\eKinf(\nu_a,m_* )-\eKinf(\nu_b,m_* ) \le 0$.
\end{proof}

Observe that Lemma~\ref{lem:taylor_kl} gives a bound very similar to that in the Gaussian setting without corruptions. Indeed, in the latter case, 
$$\KL(\mu_a, \G(m_*,1))-\KL(\mu_b, \G(m_*,1))= \frac{(\Delta_a^2-\Delta_b^2)}{2} \le (\Delta_a-\Delta_b)(\Delta_a \vee \Delta_b).$$ 
Lemma~\ref{lem:taylor_kl} is tight for $\Delta_a$ and $\Delta_b$ around $\Delta_{\min}$ but not when $\Delta_a$ and $\Delta_b$ are large, this is due to having bounded the derivative of  $\ekl_{\cG}(x,x+\Delta )$ by $\Delta$ in the proof, for simplicity because handling $\Phi(\Delta_+/2) - \Phi(\Delta_- /2)$ require knowledge on $c$ which is defined implicitely.
%\shubhada{Check the  discussion in this paragraph.}
%\deb{Timothee add a commentary here.}

\newpage
% -*- TeX-master: "../main_colt2023.tex" -*-

\section{Proofs of results from Section~\ref{sec:algo}}\label{app:regret_ubound}

\subsection{Proof of Theorem~\ref{th:upper_bound}: regret upper bound}
For $a\in[K]$ and $t\in\N$, let $\hat{\mu}_{a,t}$ denote the empirical distribution obtained using $t$ samples observed (corrupted samples) from arm $a$. To prove Theorem~\ref{th:upper_bound}, we use that 
\[N_a(T)=\sum_{n=1}^T \1\{A_n = a\},\]
and decompose $\{A_n = a\}$ using Lemma~\ref{lem:events} below. 

\begin{Lemma}[Decomposition of bad event]\label{lem:events}
For any $M>0$, 
\[\lrset{A_n = a} \subset E_n(a) \cup F_n(a) \cup G_n(a),\]
where $E_n(a)$, $F_n(a)$ and $G_n(a)$ are disjoint events defined by
\begin{align*} 
&E_n(a)= \lrset{A_n = a, N_a(n)\ekl_\cG\lrp{\Med(\hat{\mu}_a(n)) - \frac{\Delta_{\min}}{2}, m(\mu_1) - \delta } \le \log n},\\
&F_n(a) = \bigcup\limits_{t=N_{\min}}^n \Big\{A_n = a, ~ \Med(\hat{\mu}_{1,t}) \le m(\mu_1) - \delta-\frac{\Delta_{\min}}{2}  ,\\
&\hspace{6em}I_*(n) \le  t \ekl_\cG\lrp{ \Med(\hat{\mu}_{1,t})-\frac{\Delta_{\min}}{2} , m(\mu_1)-\delta} + \log t\le  t M + \log t\Big\} ,\\
&G_n(a) =\bigcup_{t=N_{\min}}^n \left\{A_n = a, \ekl_\cG\lrp{ \Med(\hat{\mu}_{1,t})-\frac{\Delta_{\min}}{2} , m(\mu_1)-\delta} \ge M, N_1(n)=t\right\}\,.
\end{align*}
\end{Lemma}

Using Lemma~\ref{lem:events}, observe that for $T\ge K N_{\min}$,
\begin{align*}\label{eq:decomp_N_efg}
 N_a(T) 
 & \le N_{\min} + \sum\limits_{n = K N_{\min}}^T \mathbbm{1}\lrp{E_n(a)} + \sum\limits_{n=K N_{\min}}^T \mathbbm{1}\lrp{F_n(a)} + \sum\limits_{n=KN_{\min}}^T \mathbbm{1}\lrp{G_n(a)} .  
\end{align*}
Thus, to bound the average number of pulls of suboptimal arm $a$, it suffices to bound the summation of the probabilities of the above indicator functions since 
\begin{equation}\label{eq:decomp_N_efg}
\E\lrp{N_a(T)} \le N_{\min}+\sum\limits_{n=KN_{\min}}^T \P\lrp{E_n(a)} + \sum\limits_{n=1}^T \P\lrp{F_n(a)} + \sum\limits_{n=1}^T \P\lrp{G_n(a)}. 
\end{equation}

In the above inequality, 
\begin{equation}\label{eq:cont3}
\sum\limits_n \P(E_n(a))\le \sum\limits_{n=1}^T  \mathbb{P}\lrp{A_n = a, N_a(n)\ekl_\cG(\Med(\hat{\mu}_a(n))-\frac{\Delta_{\min}}{2}, m(\mu_1) - \delta ) \le \log n},
\end{equation}
the second term is equal to 
\begin{align}\label{eq:cont4}
 \E\Big(\sum\limits_{n=KN_{\min}}^T \sum\limits_{t=N_{\min}}^n &\mathbbm{1}\Big(A_n = a, ~ \Med(\hat{\mu}_{1,t})  \le m(\mu_1) - \delta -\frac{\Delta_{\min}}{2} ,\nonumber\\
&~  I_*(n) \le  t \ekl_\cG\lrp{ \Med(\hat{\mu}_{1,t}) -\frac{\Delta_{\min}}{2}, m(\mu_1)-\delta} + \log t\le tM + \log(t)\Big)\Big),
\end{align}
and the third term satisfies 
\begin{equation}\label{eq:cont5}
\sum\limits_n \E(G_n(a)) \le \sum_{n=1}^T \bigcup_{t=1}^n \left\{\P \left( \ekl_\cG \lrp{\Med(\hat{\mu}_{1,t}) -\frac{\Delta_{\min}}{2}, m(\mu_1)-\delta}\ge M, N_1(n)=t\right\}\right).
\end{equation}
Here, \eq~\eqref{eq:cont3} corresponds to the deviation of suboptimal arm $a$, which will contribute to the main term in the total regret, while \eq~\eqref{eq:cont4} corresponds to the deviation of the optimal arm, whose total contribution to the regret will at most be a constant and \eq~\eqref{eq:cont5} corresponds to large deviations for $\ekl_\cG$ on the optimal arm. 

First, we bound the probability of the event $E_n(a)$ occurring with the following lemma. This gives us the main term in our regret upper bound. 

\begin{Lemma}\label{lem:event_e}
For $\delta > 0$ satisfying, 
\[\delta < \min\left(1,\Delta_a + \Delta_{\min},\frac{\ekl_\cG\lrp{m(\mu_{a}), m(\mu_1)}}{4(\Delta_a + \Delta_{\min})}\right), \] 
we have 
\begin{equation*}
\sum\limits_n \P(E_n(a)) \le \frac{\log(T)}{ \ekl_\cG\lrp{m(\mu_{a}), m(\mu_1)}-2\delta\left(\Delta_a + \delta + \Delta_{\min}\right)}
+ \frac{4}{1-\exp \left(-\delta^2/s_{\varepsilon}^2\right)},
\end{equation*}
\end{Lemma}

Next, we bound the probability of events $F_n(a)$ and $G_n(a)$ using the following two lemmas. These two events will have a negligible probability compared to the probability of event $E_n(a)$.

\begin{Lemma}\label{lem:event_f}
For $\delta<1$ and $M = \frac{\delta^2}{2s_{\varepsilon}^2}$, 
\begin{align*}
\sum_{n=1}^T\P(F_n(a))
\le \frac{e^{-\frac{\delta^2}{s_\varepsilon^2}}}{\left(1-\exp\left(-\frac{\delta^2}{s_\varepsilon^2}\right) \right)^2} + \frac{2}{\left(1-\exp\left(-\frac{\delta^2}{2s^2_{\varepsilon}}\right)\right)^2}
\le \frac{4}{\left(1-\exp\left(-\frac{\delta^2}{2s_\varepsilon^2}\right) \right)^2}.
\end{align*}

\end{Lemma}

\begin{Lemma}\label{lem:event_g}
Let $N_{\min}$ be given by
$$ N_{\min} = \left\lceil \frac{2\log(T)s_{\varepsilon}^2}{\log(1+\log(T)^{0.99})\delta^2}\right\rceil, $$
Then, for any value of $M>0$, we have
$$
\sum_{n=1}^T \P(G_n(a))\le 1+\log(T)^{0.99}.$$
\end{Lemma}
Substituting the bounds from  Lemmas~\ref{lem:event_e},~\ref{lem:event_f},~\ref{lem:event_g} in  \eq~\eqref{eq:decomp_N_efg}, we get
\begin{align}
\E[N_a(T)] &\le ~~  \frac{\log(T)}{ \ekl_\cG\lrp{m(\mu_{a}), m(\mu_1)}-2\delta\left(\Delta_a + \delta + \Delta_{\min}\right)} \nonumber  \\
&+\left\lceil \frac{2\log(T)s_{\varepsilon}^2}{\log(1+\log(T)^{0.99})\delta^2}\right\rceil +(\log T)^{0.99}+\frac{4}{\left(1-\exp\left(-\frac{\delta^2}{2s_\varepsilon^2}\right) \right)^2} +\frac{4}{1-\exp \left(-\delta^2/s_{\varepsilon}^2\right)}.\label{eq:finite_ub}
\end{align}

%\shubhada{We will shift the above finite $T$ bound in the main text. But it depends on $\delta$ that we choose later so that it is small for large $T$.}

%\shubhada{Are the $\sqrt{\log T}$ all $(\log T)^{0.99}$ in the expressions above and below?}

Next, choose 
\[\delta^2 =\frac{1}{\log(1+\log(1+\log(T)))},\] 
which also satisfies the constraints for Lemma~\ref{lem:event_f} for $T$ sufficiently large, and is such that $\delta \xrightarrow[T\to \infty]{}0$. It satisfies, 
$$\left\lceil \frac{2\log(T)s_{\varepsilon}^2}{\log(1+(\log T)^{0.99})\delta^2}\right\rceil +(\log T)^{0.99}+\frac{4}{\left(1-\exp\left(-\frac{\delta^2}{2s_\varepsilon^2}\right) \right)^2} +\frac{4}{1-\exp \left(-\delta^2/s_{\varepsilon}^2\right)} = o(\log(T)).$$
Hence, we have shown that
$$\lim_{T \to \infty} \frac{\E[N_a(T)]}{\log(T)} \le  \frac{1}{\ekl_\cG\lrp{m(\mu_{a}), m(\mu_1)}},$$
which concludes the proof of Theorem~\ref{th:upper_bound}.

\subsection{Proof of Lemma~\ref{lem:event_e}: controlling deviations of suboptimal arm (event $E_n(a)$)}
Let us first handle the summation from \eq~\eqref{eq:cont3}, this term will give us the main term in regret. Consider the following inequalities:
\begin{align*}
  \sum\limits_{n=1}^T\mathbbm{1}\lrp{E_n(a)} 
  & = \sum\limits_{n=1}^T \mathbbm{1}\lrp{A_n = a, N_a(n)\ekl_\cG\lrp{\Med(\hat{\mu}_a(n)) - \frac{\Delta_{\min}}{2}, m(\mu_1) - \delta } \le \log n}\\
  & \le \sum\limits_{n=1}^T \sum\limits_{t=1}^n \mathbbm{1}\lrp{A_n = a, t\ekl_\cG\lrp{\Med(\hat{\mu}_{a,t}) - \frac{\Delta_{\min}}{2}, m(\mu_1) - \delta } \le \log T, N_a(n) = t} \\
  & \le \sum\limits_{t=1}^T \mathbbm{1}\lrp{t\ekl_\cG\lrp{\Med(\hat{\mu}_{a,t}) - \frac{\Delta_{\min}}{2}, m(\mu_1) - \delta } \le \log T}. 
\end{align*}
The last line follows from the fact that for a given $t$, there exists only one $n$ such that the two events $A_n=a$ and $N_a(n)=t$ are true.
Thus, to bound $\sum_n \P(E_n(a))$, it suffices to bound 
\begin{equation}\label{eq:sum_p_f} \sum\limits_{t=1}^\infty \P\lrp{t\ekl_\cG\lrp{\Med(\hat{\mu}_{a,t}) - \frac{\Delta_{\min}}{2}, m(\mu_1) - \delta } \le \log T}.
\end{equation}
Each summand in the above expression is bounded by
\begin{align*}
\P\lrp{  t \ekl_\cG\lrp{\Med(\hat{\mu}_{a,t}) - \frac{\Delta_{\min}}{2}, m(\mu_1) } - t \int\limits_{m(\mu_1)-\delta}^{m(\mu_1)} \frac{d~ \ekl_\cG\lrp{ \Med(\hat{\mu}_{a,t}) - \frac{\Delta_{\min}}{2}, z } }{d z} dz  \le \log T}.
\end{align*}
Using that for $x\in\R$, $\frac{d~ \ekl_\cG\lrp{x, x+\Delta } }{d \Delta} \le \Delta$ (see Lemma~\ref{lem:deriv_K}), and injecting the probability back into \eq~\eqref{eq:sum_p_f}, we get
\begin{align*}
&\sum\limits_{t=1}^\infty \P\lrp{t\ekl_\cG\lrp{\Med(\hat{\mu}_{a,t}) - \frac{\Delta_{\min}}{2}, m(\mu_1) - \delta } \le \log T}\\
&\le \sum\limits_{t=1}^\infty \P\lrp{  t \ekl_\cG\lrp{\Med(\hat{\mu}_{a,t}) \! -\! \frac{\Delta_{\min}}{2}, m(\mu_1) } - t\! \int\limits_{m(\mu_1)-\delta}^{m(\mu_1)} \! \left(z- \Med(\hat{\mu}_{a,t}) \!+\! \frac{\Delta_{\min}}{2}\!\right) dz  \le \log T\!}\\
&\le \sum\limits_{t=1}^\infty \P\lrp{  t \ekl_\cG\lrp{\Med(\hat{\mu}_{a,t}) - \frac{\Delta_{\min}}{2}, m(\mu_1) } - t \delta\left(m(\mu_1)- \Med(\hat{\mu}_{a,t}) + \frac{\Delta_{\min}}{2}\right)  \le \log T}.
\end{align*}
Using Theorem \ref{th:concentration_median} with $y=\delta\le 1$ to bound the probability that the median have deviations larger than $\delta$, we get
\begin{align*}
&\sum\limits_{t=1}^\infty \P\lrp{t\ekl_\cG\lrp{\Med(\hat{\mu}_{a,t}) - \frac{\Delta_{\min}}{2}, m(\mu_1) - \delta } \le \log T}\\
\le& \sum\limits_{t=1}^\infty \P\lrp{  t \ekl_\cG\lrp{\Med(\hat{\mu}_{a,t}) - \frac{\Delta_{\min}}{2}, m(\mu_1) } - t \delta\left(\Delta_a + \delta+\Delta_{\min}\right)  \le \log T}\\
&+\sum_{t=1}^\infty 2\exp \left(- \frac{t\delta^2}{s_{\varepsilon}^2}\right).
\end{align*}
At this point, let us introduce
$$t_0 = \left\lceil \frac{\log(T)}{ \ekl_\cG\lrp{m(\mu_{a}), m(\mu_1)}-2\delta\left(\Delta_a + \delta + \Delta_{\min}\right)} \right\rceil. $$
where, because of the inequality
$\delta \le \min(\Delta_a + \Delta_{\min},\frac{1}{4(\Delta_a + \Delta_{\min})}\ekl_\cG\lrp{m(\mu_{a}), m(\mu_1)})$,
we can conclude that 
\begin{align*}
2\delta\left(\Delta_a + \delta + \Delta_{\min}\right) &\le 4\delta\left(\Delta_a + \Delta_{\min}\right)\\
&\le \ekl_\cG\lrp{m(\mu_{a}), m(\mu_1)}.
\end{align*}
Hence the denominator in $t_0$ is positive.

The required sum-of-probabilities \eqref{eq:sum_p_f} can further be  bounded by:
\begin{multline*}
\sum\limits_{t=t_0}^\infty \P\left(  t_0\left( \ekl_\cG\lrp{\Med(\hat{\mu}_{a,t}) - \frac{\Delta_{\min}}{2}, m(\mu_1) } -  \delta\left(\Delta_a + \delta+\Delta_{\min}\right)\right)  \le \log T\right) \\
+ 2\sum\limits_{t=t_0}^\infty\exp \left(- \frac{t\delta^2}{s_{\varepsilon}^2}\right) + t_0-1,
\end{multline*}
which is further less than
\begin{multline*}
\sum\limits_{t=t_0}^\infty \P\lrp{  \ekl_\cG\lrp{\Med(\hat{\mu}_{a,t}) - \frac{\Delta_{\min}}{2}, m(\mu_1) } \le \ekl_\cG\lrp{m(\mu_{a}), m(\mu_1)}-\delta\left(\Delta_a+\delta+\Delta_{\min} \right)}\\
+ 2\sum\limits_{t=t_0}^\infty\exp \left(- \frac{t\delta^2}{s_{\varepsilon}^2}\right)+ t_0 -1.
\end{multline*}

Now, using Lemma~\ref{lem:concentration_ekl} under the condition $\delta\le 1$,
we bound the probability in the summation above as below:
\begin{align*}
\sum\limits_{t=1}^\infty& \P\lrp{t\ekl_\cG\lrp{\Med(\hat{\mu}_{a,t}) - \frac{\Delta_{\min}}{2}, m(\mu_1) - \delta } \le \log T}\le 4\sum\limits_{t=t_0}^\infty\exp \left(- \frac{t\delta^2}{2s_{\varepsilon}^2}\right)+t_0-1,
\end{align*}
which is at most
\begin{align*}
    \frac{4}{1-\exp \left(-\delta^2/s_{\varepsilon}^2\right)} + t_0 - 1.
\end{align*}
Hence for $\delta\le 1$, we have
\begin{equation}\label{eq:bound_pen}
\sum\limits_n \P(E_n(a)) \le t_0 - 1
+ \frac{4}{1-\exp \left(-\delta^2/s_{\varepsilon}^2\right)},
\end{equation}
where it can be checked, by definition, that
$$ t_0 - 1 \le \frac{\log(T)}{ \ekl_\cG\lrp{m(\mu_{a}), m(\mu_1)}-2\delta\left(\Delta_a + \delta + \Delta_{\min}\right)} .  $$

% Observe that \shubhada{By change of variable?}
% \begin{align*}
%   & \int\limits_{m(\mu_1)-\delta}^{m(\mu_1)} \frac{d~ \ekl_\cG\lrp{ \Med(\hat{\mu}_{a,t}) - \frac{\Delta_{\min}}{2}, z } }{d z} dz\\ 
%   & \int\limits_{m(\mu_1)-\delta - \Med(\hat{\mu}_{a,t}) + \frac{\Delta_{\min}}{2} }^{m(\mu_1) - \Med(\hat{\mu}_{a,t}) + \frac{\Delta_{\min}}{2} } \frac{d~ \ekl_\cG\lrp{ \Med(\hat{\mu}_{a,t}) - \frac{\Delta_{\min}}{2}, \Med(\hat{\mu}_{a,t}) - \frac{\Delta_{\min}}{2} + \Delta } }{d \Delta} d \Delta \\
%   & = \int\limits_{m(\mu_1)-\delta - \Med(\hat{\mu}_{a,t}) + \frac{\Delta_{\min}}{2} }^{m(\mu_1) - \Med(\hat{\mu}_{a,t}) + \frac{\Delta_{\min}}{2} } (1-\varepsilon) \Delta \lrp{ \Phi\lrp{\frac{\Delta_+}{2}} - \Phi\lrp{ \frac{\Delta_-}{2}} } d\Delta\\
%   & = \int\limits_{m(\mu_1)-\delta - \Med(\hat{\mu}_{a,t}) + \frac{\Delta_{\min}}{2} }^{m(\mu_1) - \Med(\hat{\mu}_{a,t}) + \frac{\Delta_{\min}}{2} } (1-\varepsilon) \Delta \lrp{ \Phi^c\lrp{ \frac{\Delta_-}{2}} - \Phi^c\lrp{\frac{\Delta_+}{2}} } d\Delta\\
%   & \le \int\limits_{m(\mu_1)-\delta - \Med(\hat{\mu}_{a,t}) + \frac{\Delta_{\min}}{2} }^{m(\mu_1) - \Med(\hat{\mu}_{a,t}) + \frac{\Delta_{\min}}{2} } (1-\varepsilon) \Delta \lrp{2 - \frac{1}{1-\varepsilon}} d\Delta
% \end{align*}

% \todoS{Need a tighter bound on derivative! Can we use the bounds on $\Phi$ instead?}

\subsection{Proof of Lemma~\ref{lem:event_f}: controlling deviation of the optimal arm (event $F_n(a)$)}

Since each arm is pulled at least $N_{\min}$ times till time $n\ge KN_{\min}$, we have
\begin{align*}
\sum\limits_{n=KN_{\min}}^T\P(F_n(a))&~=~ \E\Big(\sum\limits_{n=KN_{\min}}^T \sum\limits_{t=N_{\min}}^n \mathbbm{1}\Big(A_n = a, ~ \Med(\hat{\mu}_{1,t})  \le m(\mu_1) - \delta -\frac{\Delta_{\min}}{2} ,\\
&~  I_*(n) \le  t \ekl_\cG\lrp{ \Med(\hat{\mu}_{1,t}) -\frac{\Delta_{\min}}{2}, m(\mu_1)-\delta} + \log t\le tM + \log(t)\Big)\Big). 
\end{align*}
By changing the order of summation in the above expression, it can be shown to equal 
\begin{multline*} \sum\limits_{t=1}^T  \E\Big(\sum\limits_{n=t}^T  \mathbbm{1}\Big(A_n = a, ~ \Med(\hat{\mu}_{1,t})  \le m(\mu_1) - \delta -\frac{\Delta_{\min}}{2} ,~ \\ I_*(n) \le  t \ekl_\cG( \Med(\hat{\mu}_{1,t})-\frac{\Delta_{\min}}{2} , m(\mu_1)-\delta) + \log t\le tM+\log(t)\Big)\Big),
\end{multline*}
which is at most 
\begin{align*}
 \sum\limits_{t=1}^T  &\E\left(\mathbbm{1}\lrp{\Med(\hat{\mu}_{1,t})  \le m(\mu_1) - \delta-\frac{\Delta_{\min}}{2} , \ekl_\cG\lrp{ \Med(\hat{\mu}_{1,t})-\frac{\Delta_{\min}}{2} , m(\mu_1)-\delta}\le M }\phantom{ \frac{1}{2}}\right.\\
&\left. \hspace{4em}\times \sum\limits_{n=t}^T  \mathbbm{1}\lrp{A_n = a,~  I_*(n) \le  t \ekl_\cG\lrp{ \Med(\hat{\mu}_{1,t})-\frac{\Delta_{\min}}{2} , m(\mu_1)-\delta} + \log t}\right).
\end{align*}

Recall that for time $n$ such that $A_n = a$, $I_*(n) = N_a(n)\ekl_\cG(\Med(\hat{\mu}_a(n))-\Delta_{\min}, \Med_*(n)) + \log N_a(n)$, which is at least $\log N_a(n)$. Using this, the above summation is bounded by  
\begin{align*}
 \sum\limits_{t=1}^T  &\E\left(\mathbbm{1}\lrp{\Med(\hat{\mu}_{1,t})  \le m(\mu_1) - \delta -\frac{\Delta_{\min}}{2}, \ekl_\cG \lrp{\Med(\hat{\mu}_{1,t}) -\frac{\Delta_{\min}}{2}, m(\mu_1)-\delta}\le M }\phantom{ \frac{1}{2}}\right.\\
&\left. \hspace{2em}\times \sum\limits_{n=t}^T  \mathbbm{1}\lrp{A_n = a,~  \log N_a(n) \le  t \ekl_\cG \lrp{ \Med(\hat{\mu}_{1,t}) -\frac{\Delta_{\min}}{2}, m(\mu_1)-\delta} + \log t}\right),
\end{align*}
which is at most (also see \citet[Lemma 13]{JMLR:v16:honda15a})
\begin{multline}\label{eq:bound_1_Fna}
\sum\limits_{t=1}^T  \E\Big(\mathbbm{1}\lrp{\Med(\hat{\mu}_{1,t})  \le m(\mu_1) - \delta -\frac{\Delta_{\min}}{2},  \ekl_\cG \lrp{ \Med(\hat{\mu}_{1,t})-\frac{\Delta_{\min}}{2} , m(\mu_1)-\delta} \le M  }\\
\times e^{  t \ekl_\cG\lrp{ \Med(\hat{\mu}_{1,t})-\frac{\Delta_{\min}}{2} , m(\mu_1)-\delta} + \log t}\Big).
\end{multline}

%\begin{Lemma}\label{lem:dev2} For $\delta < \min\limits_{b\ne 1} %\lrset{m(\mu_1) - m(\mu_b) - \frac{\Delta_{\min}}{2}}$ and $\delta \le %\frac{1}{8\Delta_{\min}} $,
%    \[\sum\limits_{n=1}^T\P(F_n(a)) = o(\log T). \]
%\end{Lemma}

Using the bound on $\sum_n \P(F_n(a)) $ from \eq~\eqref{eq:bound_1_Fna} and observing that the expectation in the bound is for a non-negative random variable, we get the following bound on $\sum_{n=1}^T \P(F_n(a))$:
 \begin{multline}\label{eq:decomp_fn}
 \sum_{t=1}^T t \int\limits_{0}^\infty \mathbb{P}\Big( \mathbbm{1}\lrp{\Med(\hat{\mu}_{1,t}) \le m_1(\mu) -\delta-\frac{\Delta_{\min}}{2},  \ekl_\cG \lrp{ \Med(\hat{\mu}_{1,t}) -\frac{\Delta_{\min}}{2}, m(\mu_1)-\delta} \le M }\\
    \times    e^{t \ekl_\cG\lrp{ \Med(\hat{\mu}_{1,t})-\frac{\Delta_{\min}}{2} , m(\mu_1)-\delta} } \ge x \Big) \d x.
\end{multline}
Let us control the integral above separately on  $[0,1]$ and $[1,\infty)$. 

\paragraph{Integral on $[0,1]$}
On $[0,1]$ we only control the deviations of the empirical median and we do not care about the deviations of $\ekl_{\cG}$:
\begin{align*}
& \int\limits_{0}^1 \mathbb{P}\Big( \mathbbm{1}\lrp{\Med(\hat{\mu}_{1,t}) \le m_1(\mu) -\delta-\frac{\Delta_{\min}}{2},  \ekl_\cG\lrp{ \Med(\hat{\mu}_{1,t}) -\frac{\Delta_{\min}}{2}, m(\mu_1)-\delta} \le M }\\ 
& \phantom{\int\limits_{0}^1 \mathbb{P} \Big(}\qquad  \times e^{t \ekl_\cG\lrp{ \Med(\hat{\mu}_{1,t}) -\frac{\Delta_{\min}}{2}, m(\mu_1)-\delta} } \ge x \Big)\d x\\
&\le \! \int\limits_{0}^1 \! \mathbb{P}\lrp{\! \Med(\hat{\mu}_{1,t}) \le m_1(\mu) -\delta-\frac{\Delta_{\min}}{2}\! }\! \d x 
%\\
%&
\!=\!   \mathbb{P}\!\lrp{ \!\Med(\hat{\mu}_{1,t}) \le m_1(\mu)\! -\!\delta-\frac{\Delta_{\min}}{2} \! } \le 2e^{-t\frac{\delta^2}{s_\varepsilon^2}}.
\end{align*}

Using Theorem \ref{th:concentration_median} for the last line, for $\delta < 1$. Then, we get 
\begin{align}\label{eq:between_0_and_1}
\sum_{t=N_{\min}}^T&\int\limits_{0}^1 t\mathbb{P}\Big( \mathbbm{1}\lrp{\Med(\hat{\mu}_{1,t}) \le m_1(\mu) -\delta-\frac{\Delta_{\min}}{2},  \ekl_\cG\lrp{ \Med(\hat{\mu}_{1,t}) -\frac{\Delta_{\min}}{2}, m(\mu_1)-\delta}\le M }\nonumber\\
&\le  2\sum_{t=1}^\infty t e^{-t\frac{\delta^2}{s_\varepsilon^2}} = 2\frac{e^{-\frac{\delta^2}{s_\varepsilon^2}}}{\left(1-e^{-\frac{\delta^2}{s_\varepsilon^2}} \right)^2}.
\end{align}
Next, we bound the integral on $[1,\infty)$.
\paragraph{Integral on $[1,\infty)$}
We use that the deviations of $\ekl_{\cG}$ are bounded by $M$ in the indicator function to bound simplify the probability as follows.
\begin{align*}
    &\sum_{t=1}^Tt \int\limits_{1}^\infty \mathbb{P}\Big( \mathbbm{1}\lrp{\Med(\hat{\mu}_{1,t}) \le m_1(\mu) -\delta-\frac{\Delta_{\min}}{2},  \ekl_\cG\lrp{ \Med(\hat{\mu}_{1,t})-\frac{\Delta_{\min}}{2} , m(\mu_1)-\delta} \le M }\\
        & \phantom{\sum_{t=1}^Tt \int\limits_{1}^\infty \mathbb{P}\Big( \qquad } \times e^{t \ekl_\cG\lrp{ \Med(\hat{\mu}_{1,t})-\frac{\Delta_{\min}}{2} , m(\mu_1)-\delta}} \ge x \Big) \d x\\
    & \le  \sum_{t=1}^Tt \int\limits_{1}^{\infty}\mathbb{P}\lrp{e^{tM}  \ge e^{t \ekl_\cG\lrp{ \Med(\hat{\mu}_{1,t})-\frac{\Delta_{\min}}{2} , m(\mu_1)-\delta}} \ge x } \d x \\
    & =    \sum_{t=1}^Tt \int\limits_{1}^{\exp(tM)}\mathbb{P}\lrp{    t \ekl_\cG\lrp{ \Med(\hat{\mu}_{1,t})-\frac{\Delta_{\min}}{2} , m(\mu_1)-\delta} \ge \log x } \d x
    \end{align*}
Then, we use a change of variable $x \gets e^y$ to show that the above is smaller than
\begin{align*}
\sum_{t=1}^Tt \int\limits_{0}^{tM}\mathbb{P}\lrp{    t \ekl_\cG( \Med(\hat{\mu}_{1,t})-\frac{\Delta_{\min}}{2} , m(\mu_1)-\delta) \ge y } e^y \d y.
\end{align*}
Next, we use the first case of
Lemma~\ref{lem:concentration_ekl} with $y=\delta$ and bound the probability that $\ekl_\cG( \Med(\hat{\mu}_{1,t})-\frac{\Delta_{\min}}{2} , m(\mu_1)-\delta)$ is strictly positive. We  have, 
\begin{align*}
&\mathbb{P}\lrp{  \ekl_\cG\lrp{ \Med(\hat{\mu}_{1,t})-\frac{\Delta_{\min}}{2} , m(\mu_1)-\delta} > 0 }
\le 2\exp\left(-\frac{t\delta^2}{s^2_{\varepsilon}}\right)
\end{align*}
Using this bound, we get the following control
\begin{align*}
     &\sum_{t=1}^Tt \int\limits_{1}^\infty \mathbb{P}\Big( \mathbbm{1}\lrp{\Med(\hat{\mu}_{1,t}) \le m_1(\mu) -\delta-\frac{\Delta_{\min}}{2},  \ekl_\cG\lrp{ \Med(\hat{\mu}_{1,t})-\frac{\Delta_{\min}}{2} , m(\mu_1)-\delta} \le M }\\
        & \phantom{\sum_{t=1}^Tt \int\limits_{1}^\infty \mathbb{P}\Big( \qquad } \times e^{t \ekl_\cG\lrp{ \Med(\hat{\mu}_{1,t})-\frac{\Delta_{\min}}{2} , m(\mu_1)-\delta}} \ge x \Big) \d x\\
    &\le 2\sum_{t=N_{\min}}^T \int\limits_{0}^{tM} t \exp\left(-\frac{t\delta^2}{s^2_{\varepsilon}}\right) e^y\d y\le 2\sum_{t=N_{\min}}^T t \exp\left(-\frac{t\delta^2}{s^2_{\varepsilon}}\right) e^{Mt}.
\end{align*}
Now, take $M = \frac{\delta^2}{2s_\varepsilon^2}$, to keep the exponent of the exponential negative, we get
\begin{align*}
     &\sum_{t=1}^Tt \int\limits_{1}^\infty \mathbb{P}\Big( \mathbbm{1}\lrp{\Med(\hat{\mu}_{1,t}) \le m_1(\mu) -\delta-\frac{\Delta_{\min}}{2},  \ekl_\cG\lrp{ \Med(\hat{\mu}_{1,t})-\frac{\Delta_{\min}}{2} , m(\mu_1)-\delta} \le M }\\
        & \phantom{\sum_{t=1}^Tt \int\limits_{1}^\infty \mathbb{P}\Big( \qquad } \times e^{t \ekl_\cG\lrp{ \Med(\hat{\mu}_{1,t})-\frac{\Delta_{\min}}{2} , m(\mu_1)-\delta}} \ge x \Big) \d x\\
   &\le 2\sum_{t=1}^T t \exp\left(-\frac{t\delta^2}{2s^2_{\varepsilon}}\right) \le \frac{2\exp\left(-\frac{\delta^2}{2s^2_{\varepsilon}}\right) }{\left(1-\exp\left(-\frac{\delta^2}{2s^2_{\varepsilon}}\right)\right)^2}\le \frac{2}{\left(1-\exp\left(-\frac{\delta^2}{2s^2_{\varepsilon}}\right)\right)^2}.
\end{align*}

\paragraph{Wrap-up: bounding $\sum_{n}\P(F_n(a))$:}
Combining Equations~\eqref{eq:decomp_fn},~\eqref{eq:between_0_and_1}, and~\eqref{eq:decomp_fn}, and choosing $$M = \frac{\delta^2}{2s_{\varepsilon}^2},$$ we finally obtain
\begin{align}
\sum_{n=1}^T\P(F_n(a))
\le \frac{e^{-\frac{\delta^2}{s_\varepsilon^2}}}{\left(1-\exp\left(-\frac{\delta^2}{s_\varepsilon^2}\right) \right)^2} + \frac{2}{\left(1-\exp\left(-\frac{\delta^2}{2s^2_{\varepsilon}}\right)\right)^2}
\le \frac{4}{\left(1-\exp\left(-\frac{\delta^2}{2s_\varepsilon^2}\right) \right)^2}.
\end{align}

\subsection{Proof of Lemma~\ref{lem:event_g}: controlling large deviations of the kl (event $G_n(a)$)}
Let us now control $\P(G_n(a))$. We have for any $M>0$,
\begin{align*}
&\sum_{n=KN_{\min}}^T \P(G_n(a)) \\
&= \sum_{n=N_{\min}}^T \P\left( \bigcup\limits_{t=N_{\min}}^n \left\{A_n = a,\ekl_\cG\lrp{ \Med(\hat{\mu}_{1,t})-\frac{\Delta_{\min}}{2} , m(\mu_1)-\delta} \ge M, N_1(n)=t\right\}\right)\\
&\le \sum_{n=N_{\min}}^T \sum_{t=N_{\min}}^n \P\left( \ekl_\cG\lrp{ \Med(\hat{\mu}_{1,t})-\frac{\Delta_{\min}}{2} , m(\mu_1)-\delta} \ge \frac{\delta^2}{2s^2_{\varepsilon}}\right)\\
&\le \sum_{n=N_{\min}}^T \sum_{t=N_{\min}}^n  \P\left( \ekl_\cG\lrp{ \Med(\hat{\mu}_{1,N_{\min}})-\frac{\Delta_{\min}}{2} , m(\mu_1)-\delta} \ge 0\right)\\
&\le T^2e^{-N_{\min}\frac{\delta^2}{s_{\varepsilon}^2}}.
\end{align*}
%\shubhada{How do we get the $1^{st}$ inequality above? Do we know that the summand is monotonic in $t$? We can use the bound first, then replace the summation by $T^2$ times the worst-case value to get the same bound - since the bound is monotonic.}

This leads us to choose
$$ N_{\min} = \left\lceil \frac{2\log(T)s_{\varepsilon}^2}{\log(1+\log(T)^{0.99})\delta^2}\right\rceil, $$
which ensures that
\begin{equation*}
\sum_{n=1}^T \P(G_n(a))\le 1+\log(T)^{0.99}.
\end{equation*}

%%%%%%%%%%%%%%%%%%%%%%%%%%%%%%%%%%%%%%%%%%%%%
\subsection{Proof of Theorem~\ref{th:concentration_median}: concentration of empirical median}\label{sec:proof_med_concentration}
Without loss of generality, by doing the change of variable $X \gets X-m$, we assume in the proof that $m=0$.
For any $\lambda>0$, we have
\begin{align*}
\P \left(  \Med(X_1^{n}) > \lambda \right) &\le  \P \left( \#\{i : X_i \ge \lambda\} \ge \frac{n}{2} \right)= \P \left( \frac{1}{n} \sum_{i=1}^{n} \1\{ X_i \ge \lambda\} \ge \frac{1}{2} \right).
\end{align*}
Let $W_1,\dots,W_n$ i.i.d $\Ber(\varepsilon)$, $Y_1,\dots,Y_n$ i.i.d $\sim \G(m,1)$ and $O_1,\dots,O_n$ be i.i.d from $H$, with the $W$'s, the $Y$'s and the $O$'s all independents. By characterization of a mixture of distributions, we have that $X_i$ is equal in distribution to $(1-W_i) Y_i+W_i O_i$. Hence, 
\begin{align}\label{eq:concentration_med_1}
\P \left( \Med(X_1^{n}) \ge \lambda \right) &= \P \left( \frac{1}{n} \sum_{i=1}^{n} \left(\1\{ (1-W_i) Y_i+W_iO_i \ge \lambda\} \right) \ge \frac{1}{2} \right)\nonumber\\
&= \P \left(  \frac{1}{n} \sum_{i=1}^{n}\left((1-W_i) \1\{ Y_i \ge\lambda\}+ W_i\1\{ O_i\ge \lambda\} \right) \ge \frac{1}{2}  \right)\nonumber\\
&\le  \P \left( \frac{1}{n} \sum_{i=1}^{n} (1-W_i)\1\{ Y_i\ge \lambda\} +\frac{1}{n} \sum_{i=1}^{n} W_i\ge \frac{1}{2} \right).
\end{align}
The quantities appearing in the right-hand-side of \eq~\eqref{eq:concentration_med_1} are all with values in $\{0,1\}$.
\paragraph{Concentration of Bernoulli random variables}\mbox{}\\
$W_1,\dots,W_n$ are i.i.d Bernoulli random variables with mean $\varepsilon$. From \citet[Lemma 6]{bourel2020tightening}, for any $\gamma \in (0,1)$,
\begin{align}\label{eq:concentration_W}
\P\left( \frac{1}{n} \sum_{i=1}^{n} W_i \ge \varepsilon  + \lrp{\frac{(1-2\varepsilon)\log\frac{1}{\gamma}}{4n\log\frac{1-\varepsilon}{\varepsilon}}}^\frac{1}{2}\right)&\le \gamma.
\end{align}

Similarly, for $1\le i\le n$ , $(1-W_i)\1\{ Y_i> \lambda\}$ are also Bernoulli random variables with mean 
\[\E[(1-W_i)\1\{ Y_i> \lambda\}] = (1-\varepsilon)(1-\Phi(\lambda)) \le (1-\Phi(\lambda)) \le 1/2.\] 
Again using the sub-Gaussian concentration from \citet[Lemma 6]{bourel2020tightening}, we have with probability larger than $1-\gamma$,  
\begin{align}\label{eq:concentration_WY}
 \frac{1}{n} \sum_{i=1}^{n} (1-W_i)\1\{ Y_i\ge \lambda\}&\le  (1-\varepsilon)(1-\Phi(\lambda))  + \lrp{\frac{(1-2(1-\varepsilon)(1-\Phi(\lambda)))\log\frac{1}{\gamma}}{4n\log\left(\frac{1-(1-\varepsilon)(1-\Phi(\lambda))}{(1-\varepsilon)(1-\Phi(\lambda))}\right)}}^\frac{1}{2} \nonumber\\
&\le (1-\varepsilon)(1-\Phi(\lambda))  + \lrp{\frac{(1-2(1-\varepsilon)(1-\Phi(\lambda)))\log(1/\gamma)}{4n\log\left(\frac{\Phi(\lambda)}{1-\Phi(\lambda)}\right)}}^\frac{1}{2},
\end{align}
where in the last line, we used that  $p \mapsto (1-p)/p$ is decreasing on $(0,1)$.
Then, from Equations~\eqref{eq:concentration_WY} and~\eqref{eq:concentration_W}, we get with probability larger than $1-2\gamma$,
\begin{multline*}
 \frac{1}{n} \sum_{i=1}^{n} (1-W_i)\1\{ Y_i\ge \lambda\}+\frac{1}{n} \sum_{i=1}^{n} W_i \\
\le (1-\varepsilon)(1-\Phi(\lambda))  + \varepsilon
+ \lrp{\frac{(1-2(1-\varepsilon)(1-\Phi(\lambda)))\log\frac{1}{\gamma}}{4n\log\left(\frac{\Phi(\lambda)}{1-\Phi(\lambda)}\right)}}^\frac{1}{2}   + \lrp{\frac{(1-2\varepsilon)\log\frac{1}{\gamma}}{4n\log\frac{1-\varepsilon}{\varepsilon}}}^\frac{1}{2}.
\end{multline*}
In this equation, there are two free parameters: $\lambda$ and $\gamma$. Next, we choose $\lambda$ so that \[\frac{1}{n} \sum_{i=1}^{n} (1-W_i)\1\{ Y_i\ge \lambda\}+\frac{1}{n} \sum_{i=1}^{n} W_i \le \frac{1}{2}\] with high probability. This choice of $\lambda$ will then allow us to control the probability in \eq~\eqref{eq:concentration_med_1}.
\paragraph{Choice of $\lambda$}\mbox{}\\
First, we state some basic inequalities for $\Phi(\lambda)$. We have for $\lambda = \frac{\Delta_{\min}}{2} + L$, using Taylor's inequality, 
$$\Phi(\lambda)-\frac{1}{2(1-\varepsilon)}\ge L \varphi\left(\frac{\Delta_{\min}}{2} + L\right)$$
and from monotonicity of $x \mapsto x/(1-x)$ on $[0,1)$,
$$\frac{\Phi(\lambda)}{1-\Phi(\lambda)} \ge \frac{\Phi(\frac{\Delta_{\min}}{2})}{1-\Phi(\frac{\Delta_{\min}}{2})} =  \frac{\frac{1}{2(1-\varepsilon)}}{ \frac{1-2\varepsilon}{2(1-\varepsilon)}}= \frac{1}{1-2\varepsilon}.$$
Then, we have
 \begin{align*}
&(1-\varepsilon)(1-\Phi(\lambda)) + \varepsilon - \frac{1}{2}    + \lrp{\frac{(1-2(1-\varepsilon)(1-\Phi(\lambda)))\log(1/\gamma)}{4n\log\left(\frac{\Phi(\lambda)}{1-\Phi(\lambda)}\right)}}^\frac{1}{2} + \lrp{\frac{(1-2\varepsilon)\log(1/\gamma)}{4n\log((1-\varepsilon)/\varepsilon)}}^\frac{1}{2} \\
&\le (1-\varepsilon) \left(\frac{1-2\varepsilon}{2(1-\varepsilon)} -L\varphi\left(\frac{\Delta_{\min}}{2}+L\right)\right) + \varepsilon - \frac{1}{2}    \\
&\qquad\qquad + \lrp{\frac{\left(1-2(1-\varepsilon)\left(\frac{1-2\varepsilon}{2(1-\varepsilon)} +L\varphi(\frac{\Delta_{\min}}{2}+L) \right)\right)\log\frac{1}{\gamma}}{4n\log\left(\frac{1}{1-2\varepsilon} \right)}}^\frac{1}{2} + \lrp{\frac{(1-2\varepsilon)\log\frac{1}{\gamma}}{4n\log\frac{1-\varepsilon}{\varepsilon}}}^\frac{1}{2} \\
&\le - (1-\varepsilon)L\varphi\left(\frac{\Delta_{\min}}{2}+L\right)  + \lrp{\frac{\varepsilon \log\frac{1}{\gamma}}{2n\log\left(\frac{1}{1-2\varepsilon} \right)}}^\frac{1}{2} + \lrp{\frac{(1-2\varepsilon)\log\frac{1}{\gamma}}{4n\log\frac{1-\varepsilon}{\varepsilon}}}^\frac{1}{2}\\
&\le  - (1-\varepsilon)L\varphi\left(\frac{\Delta_{\min}}{2}+L\right)  +  \lrp{\frac{\varepsilon \log\frac{1}{\gamma}}{2n\log\left(\frac{1}{1-2\varepsilon} \right)}}^\frac{1}{2}  + \lrp{\frac{(1-2\varepsilon)\log\frac{1}{\gamma}}{4n\log\frac{1-\varepsilon}{\varepsilon}}}^\frac{1}{2}.
\end{align*}
Now, suppose $L \le 1$ and choose
\begin{align}\label{eq:choice_L}
L &=  \frac{1}{(1-\varepsilon)\varphi\left(\frac{\Delta_{\min}}{2}+1\right)}\left(\sqrt{\frac{\varepsilon}{2\log\left(\frac{1}{1-2\varepsilon} \right)}}  + \sqrt{\frac{(1-2\varepsilon)}{4\log((1-\varepsilon)/\varepsilon)}}\right)\sqrt{\frac{\log(1/\gamma)}{n}} \nonumber\\
&= s_\varepsilon\sqrt{\frac{\log(1/\gamma)}{n}},
\end{align}
with $s_\varepsilon\sqrt{\frac{\log(1/\gamma)}{n}}  \le 1.$
After this choice of $\lambda$, there is only one free parameter remaining: $\gamma$.
\paragraph{Injection of chosen $\lambda$ in \eq~\eqref{eq:concentration_med_1}}\mbox{}\\
From the choice of $\lambda = \frac{\Delta_{\min}}{2}+L$ from \eq~\eqref{eq:choice_L}, we have 
$$ \P \left(  \Med(X_1^{n}) \ge \frac{\Delta_{\min}}{2} + s_\varepsilon\sqrt{\frac{\log(1/\gamma)}{n}} \right)\le 2\gamma.$$
Under the condition that $\gamma \ge \exp\left(-n/s_\varepsilon^2\right).$
Let us now reformulate this result by solving the following equation for $\gamma$:
$$ y = s_\varepsilon\sqrt{\frac{\log(1/\gamma)}{n}},$$
we get for any $0\le y \le 1$
$$ \P \left(  \Med(X_1^{n}) \ge \frac{\Delta_{\min}}{2} +y\right) \le  2\exp \left(-ny^2/s_\varepsilon^2\right).$$
To get the other direction, remark that $X$ is equal in distribution to $-X$ and inject in the above concentration.

\subsection{Proof of Lemma~\ref{lem:concentration_ekl}: concentration of $\ekl_{\cG}$}\label{app:lem:concentration_ekl}
The two proofs are very similar, except that we don't concentrate around the same quantity.
\paragraph{Case $m_b = m_a - \delta$ } \mbox{}\\
We write that from Lemma~\ref{lem:taylor_kl},
\begin{align*}
\ekl_\cG&\left(\Med(X_1^n)-\frac{\Delta_{\min}}{2}, m_a - \delta \right)\\
&=\ekl_\cG\left(\Med(X_1^n)-\frac{\Delta_{\min}}{2}, m_a - \delta \right)-\ekl_\cG\left(m_a-\Delta_{\min}-\delta, m_a-\delta \right)\\
&\le\left(m_a- \Med(X_1^n)-\frac{\Delta_{\min}}{2}-\delta \right)_+  \max\left(m_a-\Med(X_1^n)-\delta+\frac{\Delta_{\min}}{2}, \Delta_{\min}\right). 
\end{align*}

Then, from Theorem~\ref{th:concentration_median}, with probability larger than $1-2\exp \left(\frac{-ny^2}{s_\varepsilon^2} \right)$, we have for any $y \le 1$,
\begin{align}\label{eq:ekl_con_0_1}
\ekl_\cG\left(\Med(X_1^n)-\frac{\Delta_{\min}}{2}, m_a - \delta \right)&\le\left(y-\delta \right)_+ \max\left( y-\delta+\Delta_{\min}, \Delta_{\min}\right)\nonumber \\
 &=\left(y-\delta \right)_+\left(|y-\delta|+\frac{\Delta_{\min}}{2}\right),
\end{align}
where the last line comes from the fact that when $y \le \delta$, the bound is $0$ anyway.

\paragraph{Case $m_b > m_a + \Delta_{\min}$}\mbox{}\\
From Lemma~\ref{lem:taylor_kl},
\begin{align*}
\ekl_\cG&\left(m_a, m_b \right)-\ekl_\cG\left(\Med(X_1^n)-\frac{\Delta_{\min}}{2}, m_b \right)\\
&\le\left( \Med(X_1^n)-m_a-\frac{\Delta_{\min}}{2} \right)_+  \max\left(m_b-\Med(X_1^n)+\frac{\Delta_{\min}}{2}, m_b-m_a\right).
\end{align*}
Then, from Theorem~\ref{th:concentration_median},  with probability larger than $1-2\exp \left(\frac{-ny^2}{s_\varepsilon^2} \right)$, we have for any $y \le 1$,
\begin{align*}
\ekl_\cG&\left(m_a, m_b \right)-\ekl_\cG\left(\Med(X_1^n)-\frac{\Delta_{\min}}{2}, m_b \right)\\
&\le y\max\left(m_b-m_a+y+\Delta_{\min}, m_b-m_a\right)\\
&=   y (m_b-m_a+y+\Delta_{\min}).
\end{align*}

\subsection{Proof of Lemma~\ref{lem:events}}

The event $\lrset{A_n = a}$ can be written as a disjoint union of
\begin{equation}\label{eq:event1}
    \lrset{A_n = a, ~ \Med_*(n) > m(\mu_1) - \delta - \frac{\Delta_{\min}}{2} }
\end{equation}
and
\begin{equation} \label{eq:event2}
    \lrset{A_n = a, ~ \Med_*(n) \le m(\mu_1) - \delta - \frac{\Delta_{\min}}{2}}.
\end{equation}

Of these, intuitively, the second event in \eq~\eqref{eq:event1} should not be rare. However, once sufficient samples have been allocated to arm $a$, the event $\lrset{A_n = a}$ becomes rare when $\Med_*(n)$ is close to $m_1(\mu)$. This is because after sufficient samples, $\hat{\mu}_a(n) \approx \mu_a$, which implies that $\ekl_\cG(\Med(\hat{\mu}_a(n)) - \Delta_{\min}, \Med_*(n))$ should be large. For the event in \eq~\eqref{eq:event2}, for large $n$, the event $\lrset{ \Med_*(n) \le m_1(\mu)-\delta - \Delta_{\min}/2}$ should be rare. We will show that the probability of \eq~\eqref{eq:event1} occurring, summed across time, contributes to the main term in regret.\\

Define $I_*(n) := \min_a I_a(n)$ to be the minimum index. Recall that $a^*(n)$ denotes the arm with the maximum estimated mean, i.e., 
$$a^*(n) \in \arg\max\limits_{b\in [K]} \Med(\hat{\mu}_b(n)).$$  
Since $A_n = a$ implies that $I_a(n) = I_*(n)$. Then,  
\begin{align*} 
    I_a(n) &= I_*(n) \\
           &\le I_{a^*(n)}(n)\\
           &= \log N_{a^*(n)}(n)\\
           &\le \log n.
\end{align*}
Thus, $\lrset{A_n = a}$ implies that $I_a(n)\le \log n$ and \eq~\eqref{eq:event1} is contained in 
\[ \lrset{A_n = a, ~ N_a(n) \ekl_\cG\lrp{\Med(\hat{\mu}_a(n)) - \Delta_{\min}, \Med_*(n) } \le \log n , ~ \Med_*(n)   > m(\mu_1) - \delta  - \frac{\Delta_{\min}}{2}}.\]
Next, using the monotonicity of $\ekl_\cG$ in the second argument and its translation invariance (Lemma \ref{lem:basics_ekl}) in the above containment, we have that $\lrset{A_n = a, ~ \Med_*(n)  > m_1(\mu) - \delta -  \frac{\Delta_{\min}}{2}}$ is contained in 
\begin{equation} \label{eq:cont1}
    \lrset{A_n = a, N_a(n)\ekl_\cG\lrp{\Med(\hat{\mu}_a(n)) - \frac{\Delta_{\min}}{2}, m(\mu_1) - \delta } \le \log n}.
\end{equation}

Next, observe that the event in \eq~\eqref{eq:event2} satisfies
\begin{multline*}
\lrset{A_n = a, ~ \Med_*(n)   \le m(\mu_1) - \delta - \frac{\Delta_{\min}}{2}} \\
\subset \lrset{A_n = a, ~ \Med_*(n)   \le m(\mu_1) - \delta - \frac{\Delta_{\min}}{2},\ekl_\cG\lrp{\Med(\hat{\mu}_1(n)) - \frac{\Delta_{\min}}{2} , m(\mu_1)-\delta}\le M }\\
\cup \lrset{ \ekl_\cG\lrp{\Med(\hat{\mu}_1(n)) - \frac{\Delta_{\min}}{2} , m(\mu_1)-\delta} \ge M} 
\end{multline*}
which is included in 
\begin{multline*}
 \bigcup\limits_{t=1}^n \left(A_n = a, ~ \Med_*(n)   \le m(\mu_1) - \delta - \frac{\Delta_{\min}}{2},\right.\\
 \left.\ekl_\cG\lrp{\Med(\hat{\mu}_1(n)) - \frac{\Delta_{\min}}{2} , m(\mu_1)-\delta}\le M,  N_1(n) = t \right)\\
\cup \lrset{ \ekl_\cG\lrp{\Med(\hat{\mu}_1(n)) - \frac{\Delta_{\min}}{2} , m(\mu_1)-\delta} \ge M,  N_1(n) = t} 
\end{multline*}
Let $\hat{\mu}_{1,t}$ denote the empirical distribution for arm $1$ with $t$ samples. Now, since $A_n = a$ implies that
\[   I_a(n) = I_*(n) \le I_1(n) =  N_1(n) \ekl_\cG\lrp{ \Med(\hat{\mu}_1(n)) - \Delta_{\min} , \Med_*(n)} + \log N_1(n), \]
the above union-of-events is further contained in
\begin{multline*}
    \bigcup\limits_{t=1}^n \Big\{A_n = a, ~ \Med(\hat{\mu}_{1,t})  \le \Med_*(n) \le m(\mu_1) - \delta-\frac{\Delta_{\min}}{2}  ,\\
    ~  I_*(n) \le  t \ekl_\cG\lrp{\Med(\hat{\mu}_{1,t}) - \frac{\Delta_{\min}}{2} , m(\mu_1)-\delta} + \log t\le tM+\log(t) \Big\}\\
\cup \left\{ \ekl_\cG\lrp{\Med(\hat{\mu}_{1,t}) - \frac{\Delta_{\min}}{2} , m(\mu_1)-\delta} \ge M,  N_1(n) = t\right\},
\end{multline*}
which is the union of $F_n(a)$ and $G_n(a)$.

\newpage
\section{\rimed{} for misspecified Gaussian model}\label{app:misspecified}
In the main text, we assumed that the arm distributions followed a Gaussian distribution with unit variance and observations were corrupted. In this section, we explore a modification of \rimed{} that demonstrates a strong numerical performance even in the presence of model misspecification (with corruption). 
Impacts of model misspecification in bandits has attracted increasing attention, where the existing studies focus mostly on the linear bandits and linear contextual bandits \citep{ghosh2017misspecified,foster2020adapting}.
In contrast, we study multi-armed stochastic bandits with misspecification about reward distributions being Gaussian.

We commence by detailing the misspecified model and the adaptation of \rimed, providing a brief outline of its theoretical rationale. The section concludes with  numerical analyses of the proposed algorithm.

\subsection{Misspecified Gaussian model: Regret lower bound and algorithm design}
Fix $\varepsilon > 0$, the fraction of observations that are corrupted. Recall that we place absolutely no assumptions on corruption distributions. In the current section, we consider misspecification of the Gaussian models. Fix $\varepsilon^{(m)} >0$, \textemdash this denotes the fraction of samples that are misspecified. Fix $\delta > 0,$ \textemdash this represents a bound on the difference of the mean of the Gaussian distribution and the misspecification distribution. Unlike for the corruption distributions, which are allowed to perturb the mean arbitrarily, we need this bound for the misspecification model. To be more specific, we consider the following perturbation of the Gaussian models for the arm distributions, with fixed and known $\varepsilon^{(m)}$, and $\delta$. 
\begin{equation*}
\mathcal{I}^\delta_{\varepsilon^{(m)}} := \lrset{\! (1-\varepsilon^{(m)}) \mathcal{N}(x,1)+ \varepsilon^{(m)} \eta^{(m)}: ~ x \in \R,~|x-m(\eta^{(m)})|\le \delta,~ \eta^{(m)}\in\mathcal P(\R)},
\end{equation*}
where, recall that $\mathcal P(\R)$ denotes the collection of all probability measures on $\R$. Here, $\eta^{(m)}$ is the misspecification distribution, which is restricted to have a mean close to that of the true Gaussian distribution, but otherwise can be arbitrary. \\

\noindent{\bf Bandit model.} Each arm $a$ in the bandit instance is associated with a distribution $\mu_a \in \mathcal I^{\delta}_{\varepsilon^{(m)}}$, which is a $(1-\varepsilon^{(m)}, \varepsilon^{(m)})$ mixture of $\G(m_a,1)$ and $\eta^{(m)}_a$, i.e., 
\[ \mu_a = (1-\varepsilon^{(m)}) \G(m_a, 1) + \varepsilon^{(m)}\eta^{(m)}_a. \]
On pulling an arm $a$, the algorithm receives a reward which is an independent sample drawn from $\mu_a$. However, as in the main text,  with probability $1-\varepsilon$, it observes this independent sample, but with the remaining $\varepsilon$ probability, it observes a sample drawn from an arbitrary corruption distribution. However, unlike in the main text, here the uncorrupted sample is not from a Gaussian distribution, but a mixture of a Gaussian and a misspecification distribution.

Next, consider the $\varepsilon$ corruption neighbourhood of distributions in $\mathcal I^{\delta}_{\varepsilon^{(m)}}$, given by $\mathcal C^{\delta}_{\varepsilon, \varepsilon^{(m)}}$ below.
\begin{equation*}
\mathcal C^{\delta}_{ \varepsilon, \varepsilon^{(m)}} = \lrset{ \kappa \corby H: ~ \kappa\in\mathcal I^{\delta}_{\varepsilon^{(m)}}, ~ H^{(m)} \in \mathcal{P}(\R) }.\end{equation*}

For an arm $a$ with distribution $\mu_a\in \mathcal I^{\delta}_{\varepsilon^{(m)}}$ and corruption distribution $H$, the observations from these arms are distributed as $\mu_a \corby H \in \mathcal C^{\delta}_{\varepsilon, \varepsilon^{(m)}}$. This can be re-expressed as
\begin{align*}
 (1-\varepsilon^{(m)})(1-\varepsilon)\mathcal{N}(m_a,1) + \varepsilon^{(m)} (1-\varepsilon) \eta^{(m)}_a + \varepsilon H.
\end{align*}

\noindent{\bf Regret. } For this setting, $\bar{\Delta}_a := m^*(\mu) - m(\mu_a)$, where recall that $m^*(\mu)$ denotes the maximum mean of distributions in $\mu$, and define $\Delta_a := \max_b(m_b - m_a)$, where $m_a$ is the mean of the Gaussian distribution associated with arm $a$. Then, $\bar{\Delta}_a$ equals
\[ \max\limits_{b} \lrset{ (1-\varepsilon^{(m)})m_b + \varepsilon^{(m)} m(\eta^{(m)}_b) -  (1-\varepsilon^{(m)})m_a - \varepsilon^{(m)} m(\eta^{(m)}_a)}.\]
Since $\eta_a$ is assumed to satisfy  $\abs{m(\eta^{(m)}_a)-m_a} \le \delta$, we have 
\[\bar{\Delta}_a \le \max\limits_b \lrset{  m_b - m_a + 2\varepsilon^{(m)} \delta  } =  \Delta_a + 2\varepsilon^{(m)}\delta. \]

The expected regret incurred by the algorithm in $T$ trials can then be shown to satisfy
\begin{align*}
 \E[R_T] &= \sum_{a=1}^K\E[N_a(T)] \bar{\Delta}_a \le \sum_{a=1}^K\E[N_a(T)]\left(\Delta_a + 2\varepsilon^{(m)} \delta\right).
\end{align*}
Here, the expectation is with respect to the randomness in the algorithm, arm distributions, as well as the corruption distributions. As in Theorem~\ref{lem:lower_bound}, one can then obtain the following lower bound on regret for an appropriate definition of uniformly-good algorithms (that know both $\varepsilon$ and $\varepsilon^{(m)}$): 
\[ \liminf\limits_{T\rightarrow \infty} \frac{1}{\log T}\lrp{\sup\limits_{{\bf H} \in\mathcal{P}(\R)^{K}}\Exp{\mu \corby {\bf H}}{N_a(T)}} \ge~\frac{1}{\eKinf\lrp{\mu_a,m^*(\mu); \mathcal I^{\delta}_{\varepsilon^{(m)}}}}, \]
where $\eKinf$ is defined as earlier, and is given below for completeness. For $\eta\in\mathcal P(\R)$, $x\in\R$, 
\begin{equation}\label{eq:missp_klinf}
    \eKinf\lrp{\!\mu_a,m^*(\mu); \mathcal I^{\delta}_{\varepsilon^{(m)}}\!}\! =\! \min\!\lrset{\! \KL(\mu_a\corby H, \kappa\corby H'): \kappa\in \mathcal{I}^{\delta}_{\varepsilon^{(m)}}, m(\kappa) \ge x, H,\!H'\! \in \!\mathcal P(\R) \!}.
\end{equation}

For $x\in\R$ and $y\ge x$, recall the definition of Gaussian $\eKinf$,  $\ekl_{\cG}(x,y)$, from \eq~\eqref{eq:klg_cor}. We now show that the $\eKinf$ with respect to the misspecified model $\mathcal{I}^{\delta}_{\varepsilon^{(m)}}$ defined above, is lower bounded by that for a Gaussian class with a unit variance, with a blown-up corruption proportion. 

\begin{Lemma}\label{lem:kl_misspecified}
Let $\mu_a = (1-\varepsilon^{(m)})\mathcal{N}(m_a,1)+\varepsilon^{(m)} \eta^{(m)}_a \in \mathcal{I}^{\delta}_{\varepsilon^{(m)}}$, and $\tilde{\varepsilon} := \varepsilon + \varepsilon^{(m)} - \varepsilon\varepsilon^{(m)}$. Then 
$$\eKinf\left(\mu_a, x;\mathcal{I}^{\delta}_{\varepsilon^{(m)}} \right) \ge \mathrm{kl}^{\tilde{\varepsilon}}_{\mathcal{G}}\left(m_a, x-\varepsilon^{(m)}\delta\right).$$
\end{Lemma}
\begin{proof}
Observe from \eq~\eqref{eq:missp_klinf} that $\eKinf(\mu_a, x;\mathcal{I}^{\delta}_{\varepsilon^{(m)}})$ equals
\begin{align*}
\min&\left\{ \KL(( (1-\varepsilon^{(m)}) \G(m_a, 1) + \varepsilon^{(m)}\eta^{(m)}_a)\corby H, ((1-\varepsilon^{(m)}) \G(y, 1) + \varepsilon^{(m)}\kappa^{(m)})\corby H')\!:\right.\\
& \qquad\qquad\qquad\left.(1-\varepsilon^{(m)})y + \varepsilon^{(m)}m(\kappa^{(m)})\ge x, H,H',\kappa^{(m)} \in \mathcal P(\R), \abs{y-m(\kappa^{(m)})}\le \delta \right\},
\end{align*}
where the minimisation is over $y$, $\kappa^{(m)}$, $H $, and $H'$. The inequalities on $m(\kappa^{(m)})$ in the constraints above imply 
\[ y \ge x - \varepsilon^{(m)} \delta. \]
Using this in the constraints instead, and further optimising over $\eta^{(m)}_a$, $\eKinf$ is lower bounded as below:
\begin{align*}
    \min&\left\{ \KL(( (1-\varepsilon^{(m)}) \G(m_a, 1) + \varepsilon^{(m)}\eta^{(m)}_a)\corby H, ((1-\varepsilon^{(m)}) \G(y, 1) + \varepsilon^{(m)}\kappa^{(m)})\corby H')\!:\right.\\
    &\qquad\qquad\qquad\qquad \left.y\ge x - \varepsilon^{(m)}\delta, H,H',\kappa^{(m)},\eta^{(m)}_a \in \mathcal P(\R) \right\},
\end{align*}
where the minimisation is over $y$, $\kappa^{(m)}$, $\eta^{(m)}_a$, $H $, and $H'$. Let $\tilde{\varepsilon} := \varepsilon + \varepsilon^{(m)} - \varepsilon \varepsilon^{(m)}$. Then, the lower bound obtained above equals, %\shubhada{This is different from earlier. Discuss once.}
\[ \operatorname{KL}^{\tilde{\varepsilon}}_{\operatorname{inf}}(\G(m_a,1), x - \varepsilon^{(m)} \delta; \cG), \]
which equals
\[ \mathrm{kl}^{\tilde{\varepsilon}}_{\mathcal{G}}(m_a, x-\varepsilon^{(m)}\delta),
 \] 
 proving the desired bound.
\end{proof}

We now present the modification of \rimed{} for this setting. We do not use the knowledge of $\delta$ in algorithm design. 

\begin{paragraph}{\bf Algorithm.}
    We increase the value of the parameter $\varepsilon$ in \rimed{} to \[\tilde{\varepsilon} = \varepsilon + \tilde{\varepsilon}^{(m)} - \varepsilon\tilde{\varepsilon}^{(m)}\] 
    to encompass both corruption and the misspecification distributions as outliers (corruptions), i.e., we modify \rimed{} to use $\tilde{\varepsilon}$ in place of $\varepsilon$ everywhere (index as well as $N_{\min}$). We call $\rimed{}^{(m)}$ the resulting algorithm (similarly $\rimedstar{}^{(m)}$).\\
\end{paragraph}

\noindent{\bf Regret bound.} Following the proof of Theorem~\ref{th:upper_bound}, we get the following upper bound on regret of the modified algorithm. For $\mu\in \mathcal I^{\delta}_{\varepsilon^{(m)}}$ such that for each sub-optimal arm $a$, $\Delta_a - 2\varepsilon^{(m)}\delta \ge \Delta_{\min},$ where $\Delta_{\min}$ is the minimum gap (Definition~\ref{def:deltamin}) corresponding to $\tilde{\varepsilon}$, $\rimed{}^{(m)}$ satisfies
\begin{align*}
\lim_{T \to \infty} \frac{\E[N_a(T)]}{\log(T)}& \le  \frac{1}{\mathrm{kl}^{\tilde{\varepsilon}}_{\mathcal{G}}\lrp{m_a, \max_b \lrset{m_b} }}, 
\end{align*}
where the arguments of $\mathrm{kl}^{\tilde{\varepsilon}}_{\cG}$ are means of the Gaussian parts of the arm distributions. Further, recall that the condition on the misspecification distributions,  gives the following: 
\[\abs{m_a - m(\eta^{(m)}_a)} \le \delta \implies m(\mu_a) \ge m_a - \varepsilon^{(m)}\delta.\]

Since $\mathrm{kl}^{\tilde{\varepsilon}^{(m)}}$ is non-decreasing in its second argument (Lemma~\ref{lem:basics_ekl}), we get the following upper bound on regret: 
\begin{align*}
\lim_{T \to \infty} \frac{\E[N_a(T)]}{\log(T)}& \le  \frac{1}{\mathrm{kl}^{\tilde{\varepsilon}}_{\mathcal{G}}\lrp{m_a, \max_b m_b }} \le \frac{1}{\mathrm{kl}^{\tilde{\varepsilon}}_{\mathcal{G}}\lrp{m_a, m^*(\mu) - \varepsilon^{(m)}\delta}}, 
\end{align*}
where $m^*(\mu)$ denotes the maximum mean of the arms in $\mu$. This establishes a logarithmic regret for the misspecified setting. In the next section, we present some numerical results to justify the logarithmic bound. 

\subsection{Experimental illustration}

In this section, we present experiments for the misspecified setting. We consider two settings of bandit with $3$ arms: Setting 4 and Setting 5. In Setting 4, there are no outliers, the law is misspecified Gaussian with means $m_a$ having values $[0.6,0.8,1]$, and standard deviation $0.5$; the misspecification distribution for each arm is Gaussian with means $[3,3,3]$ and standard deviation $0.5$; the misspecification weight $\varepsilon^{(m)}=0.1$. Plots of the three distributions can be found on Figure~\ref{fig:misspecified_distrib}. 

In Setting 5, in addition to model misspecification, we also have corruption. The arm distributions are the same as in Setting 4. In addition, the corruption proportion is set to $\varepsilon = 0.01$, with corruption distributions for each arm being Gaussian with means $[10,10,-20]$, and standard deviation $1$. The results are plotted in Figure~\ref{fig:misspecification}.

In Figure~\ref{fig:misspecification} we see that $\rimedstar{}^{(m)}$ performs well in a misspecified setting. In particular, it is better than \texttt{IMED} which (mistakenly) considers a Gaussian model. This shows that using corruption to tackle model misspecification is worthwhile. As in experiments from Section~\ref{sec:xp}, $\rimedstar{}^{(m)}$ is also better than \texttt{RobustUCB}, mainly due to the non-optimality of \texttt{RobustUCB}.
\begin{figure}[t!]
    \centering
    \includegraphics[width=0.6\textwidth]{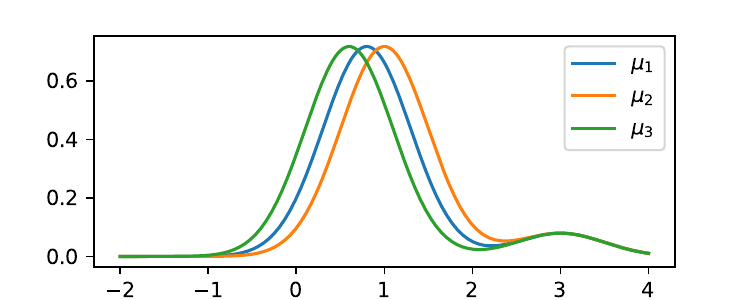}
    \caption{Reward distributions for arms in Settings 4 and 5.}\label{fig:misspecified_distrib}
\end{figure}

\begin{figure}[t!]
\centering
\includegraphics[width=0.49\textwidth, trim=11 11 11 11,clip]{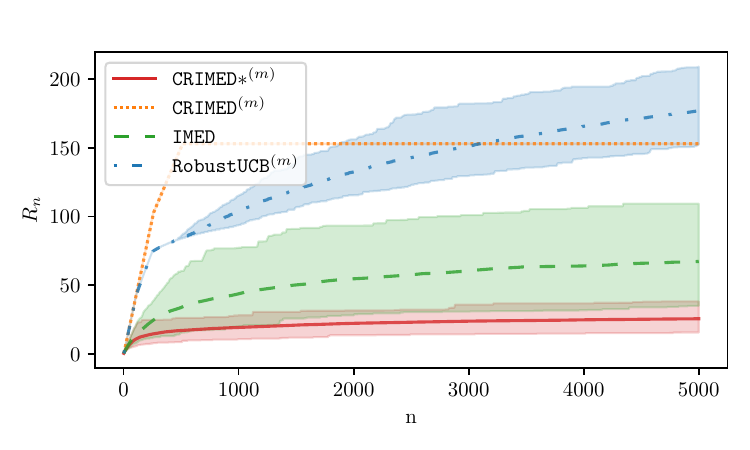}
\hfill
\includegraphics[width=0.49
\textwidth, trim=11 11 11 11,clip]{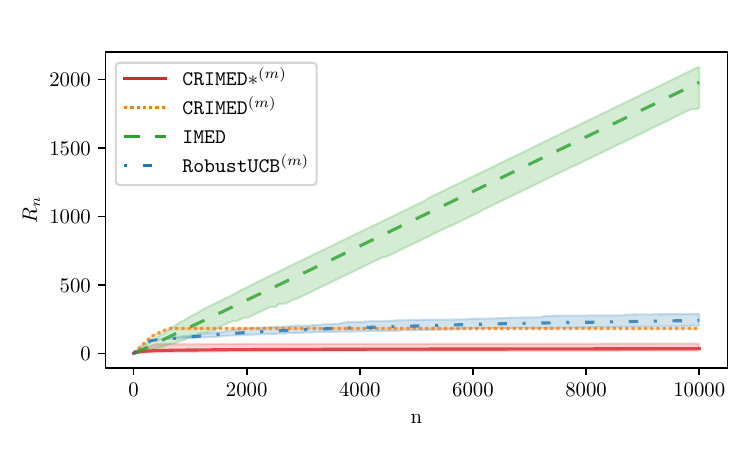}
\caption{Cumulative regret for $100$ repetitions on Settings 4 (left) and 5 (right). Solid lines represent the means and shaded area are $90\%$ percentile intervals.}\label{fig:misspecification}
\end{figure}

% \crefalias{section}{appendix} % uncomment if you are using cleveref

\end{document}